\documentclass[11pt]{article}
\usepackage{etoolbox}
\newbool{conf}
\setbool{conf}{false}

\newbool{report}
\setbool{report}{true}

\ifbool{report}{
	\usepackage{hslTR}
}{
	\def\volumeyear{2024}
	\setcounter{secnumdepth}{3}
}

\usepackage{graphics}
\usepackage{graphicx}
\usepackage{amssymb}
\usepackage{verbatim}     
\usepackage{amsxtra}
\usepackage{epsfig}
\usepackage{subfigure}
\usepackage{amsmath}
\usepackage{latexsym}
\usepackage{ifthen}
\usepackage{psfrag}
\usepackage{subfigure}
\usepackage{color}
\usepackage{float}
\usepackage{tikz}
\usepackage{mathtools}
\usepackage{etoolbox}
\usepackage{algorithm}
\usepackage{algpseudocode}
\let\labelindent\relax
\usepackage{enumitem}
\usepackage{import}
\usepackage{pdfpages}
\usepackage{amsthm}

\usepackage{moreverb,url}
\usepackage{natbib}
\usepackage[colorlinks,bookmarksopen,bookmarksnumbered,citecolor=red,urlcolor=red]{hyperref}

\usepackage{./macros/rgsEnvironments}
\usepackage{./macros/rgsMacros}  

\usepackage{endnotes}
\let\footnote=\endnote

\DeclareMathOperator{\interior}{int}
\DeclareMathOperator*{\argmin}{arg\,min}
\newcommand{\nnumbers}[1][]{\mathbb{N}}
\newcommand{\myreals}[1][]{\mathbb{R}^{#1}}
\newcommand{\mypreals}[1][]{\mathbb{R}_{\geq 0}^{#1}}
\newcommand{\myinputt}{\upsilon}
\newcommand{\mystatet}{\phi}
\newcommand{\hs}{\mathcal{H}}
\newcommand{\il}{\mathcal{U}}

\newcommand{\pn}[1]{{#1}}
\newcommand{\pnc}[1]{#1}
\newcommand{\sj}[1]{{#1}}
\newcommand{\stree}{\mathcal{T}}
\newcommand{\fw}[1]{{#1}^{\mathrm{fw}}}
\newcommand{\bw}[1]{{#1}^{\mathrm{bw}}}
\newcommand{\bwr}[1]{{#1}^{\mathrm{bw'}}}
\newcommand{\recon}[1]{{#1}^{\mathrm{r}}}
\newcommand{\simu}[1]{{#1}^{\mathrm{r}}}

\newtheorem{problem}{Problem}
\newcommand\BibTeX{{\rmfamily B\kern-.05em \textsc{i\kern-.025em b}\kern-.08em
		T\kern-.1667em\lower.7ex\hbox{E}\kern-.125emX}}

\newlist{steps}{enumerate}{1}
\setlist[steps, 1]{leftmargin=45pt, label = \textbf{Step \arabic*}:}

\newcommand{\EndRemark}{$\hfill \blacktriangle$}

\begin{document}
	\ifbool{report}{
	\ititle{HyRRT-Connect:  Bidirectional Motion Planning for Hybrid Dynamical Systems}
	\iauthor{
		Nan Wang \\
		{\normalsize nanwang@ucsc.edu} \\
		Ricardo G. Sanfelice \\
		{\normalsize ricardo@ucsc.edu}}
	\idate{\today{}}
	\iyear{2023}
	\irefnr{04}
	\makeititle
	\begin{abstract}
		This paper proposes a bidirectional rapidly-exploring random trees (RRT) algorithm to solve the motion planning problem for hybrid systems. The proposed algorithm, called HyRRT-Connect, propagates in both forward and backward directions in hybrid time until an overlap between the forward and backward propagation results is detected. Then, HyRRT-Connect constructs a motion plan through the reversal and concatenation of functions defined on hybrid time domains, ensuring that the motion plan satisfies the given hybrid dynamics. To address the potential discontinuity along the flow caused by tolerating some distance between the forward and backward partial motion plans, we reconstruct the backward partial motion plan by a forward-in-hybrid-time simulation from the final state of the forward partial motion plan. effectively eliminating the discontinuity. The proposed algorithm is applied to an actuated bouncing ball system and a walking robot example to highlight its computational improvement.
	\end{abstract}
	}{
		\runninghead{Wang and Sanfelice}
		\title{HyRRT-Connect:  Bidirectional Motion Planning for Hybrid Dynamical Systems}
		\author{Nan Wang and Ricardo G. Sanfelice\affilnum{1}}
		
		\affiliation{\affilnum{1}Hybrid Systems Laboratory, Department of Electrical and Computer Engineering, University of California, Santa Cruz, USA}
		
		\corrauth{\affilnum{1}Nan Wang, Department of Electrical and Computer Engineering, University of California, Santa Cruz,
			USA.}
		\email{nanwang@ucsc.edu}
		
		\begin{abstract}
This paper proposes a bidirectional rapidly-exploring random trees (RRT) algorithm to solve the motion planning problem for hybrid systems. The proposed algorithm, called HyRRT-Connect, propagates in both forward and backward directions in hybrid time until an overlap between the forward and backward propagation results is detected. Then, HyRRT-Connect constructs a motion plan through the reversal and concatenation of functions defined on hybrid time domains, ensuring that the motion plan satisfies the given hybrid dynamics. To address the potential discontinuity along the flow caused by tolerating some distance between the forward and backward partial motion plans, we reconstruct the backward partial motion plan by a forward-in-hybrid-time simulation from the final state of the forward partial motion plan. effectively eliminating the discontinuity. The proposed algorithm is applied to an actuated bouncing ball system and a walking robot example to highlight its computational improvement.
\end{abstract}
		
		\maketitle
	}
	



%
%

\section{Introduction}
\subsection{Background}
\sj{In the context of dynamical systems,} motion planning consists of finding a state trajectory and corresponding inputs that connect initial and final state sets, satisfying the system dynamics and specific safety requirements.  Motion planning for purely continuous-time dynamical systems and purely discrete-time dynamical systems has been extensively explored in existing literature; see, e.g., \cite{lavalle2006planning}. \sj{Some examples} include graph search algorithms \cite{wilfong1988motion}, artificial potential field method \cite{khatib1986real} and sampling-based algorithms. The sampling-based algorithms have drawn much attention because of their fast exploration speed for high-dimensional problems \sj{as well as their} theoretical guarantees; \pn{especially}, probabilistic completeness, which means that the probability of failing to find a motion plan converges to zero, as the number of \sj{iterations} approaches infinity. Compared with other sampling-based algorithms, such as probabilistic roadmap algorithm, the rapidly-exploring random tree (RRT) algorithm \cite{lavalle2001randomized} is perhaps most successful to solve motion planning problems because it does not require a steering function to solve a two point boundary value problem. 

While motion planning algorithms have been extensively applied to purely continuous-time and purely discrete-time systems, comparatively less effort has been devoted into motion planning for systems with combined continuous and discrete behavior. 
For a special class of hybrid systems, like systems with explicit discrete states or logic variables, hybrid automata and switching systems, some RRT-type motion planning algorithms, such as hybrid RRT \cite{branicky2003rrts} and R3T \cite{wu2020r3t}, respectively, have been developed. These methodologies have also found application in falsification of hybrid automata, leading to the Monte-Carlo sampling algorithm in \cite{nghiem2010monte} and the Breach toolbox in \cite{donze2010breach}. However, this class of hybrid systems lacks the generality needed to cope with hybrid system arising in robotics applications, especially those with state variables capable of both continuous evolution and discrete behaviors. Moreover, the formal study of solution properties in such systems, such as completeness, remains less developed compared to hybrid equation models. 
In \pn{recent work,} \cite{wang2025motion} formulates a motion planning problem for hybrid systems using hybrid equations, as in \cite{sanfelice2021hybrid}. This formulation presents a general framework that encompasses a \pn{broad class} of hybrid systems. For example, hybrid automata are effective for modeling mode-switching behaviors but impose rigid transition rules that restrict their ability to accommodate Zeno behaviors, set-valued jump maps, and differential inclusions. The more general Hybrid Equation framework naturally handles these behaviors with well-posed mathematical foundations. In the same paper, a probabilistically complete RRT algorithm, referred to as HyRRT, specifically designed to address the motion planning problem for such broad class of hybrid systems \pn{is introduced}. Building on this work, \cite{wang2023hysst} \pn{introduces HySST}, an asymptotically near-optimal motion planning algorithm for hybrid systems that builds from the Stable Sparse RRT (SST) algorithm introduced in \cite{li2016asymptotically}. \sj{HySST stands as the first optimal motion planning algorithm for hybrid systems and its computational efficiency is notable, achieved through the sparsification of vertices during the search process.}

It is significantly challenging for \pn{the majority of} motion planning algorithms to maintain efficient computation performance, especially in solving high-dimensional problems. Although RRT-type algorithms have demonstrated notable efficiency in rapidly searching for solutions to high-dimensional problems compared to other algorithm types, there remains room for enhancing their computational performance.
To improve the computational performance, a modular motion planning system for purely continuous-time systems, named FaSTrack, is designed in \cite{herbert2017fastrack} to simultaneously plan and track a trajectory. FaSTrack accelerates the planning process by only considering a low-dimensional model of the system dynamics. The authors of~\cite{kuffner2000rrt} introduce RRT-Connect, an algorithm that propagates both in forward direction and backward direction, leading to a notable improvement in computational performance. 
\subsection{Contribution}
Motivated by the need of computationally efficient motion planning algorithms for hybrid systems, we design a bidirectional RRT-type algorithm for hybrid dynamical systems, called HyRRT-Connect. HyRRT-Connect incrementally constructs two search trees, one rooted in the initial state set and constructed forward in hybrid time, and the other rooted in the final state set and constructed backward in hybrid time. 
To facilitate backward propagation, we formally define a backward-in-time hybrid system that captures the dynamics of the given hybrid system when solutions evolve in backward hybrid time. In constructing each search tree, at first, HyRRT-Connect draws samples from the state space. Then, it selects the vertex such that the state associated with this vertex has minimal distance to the sample. Next, HyRRT-Connect simulates a state trajectory from the selected vertex by applying a randomly selected input. When HyRRT-Connect detects an overlap between a path in the forward search tree and a path in the backward search tree, it initiates the construction of a motion plan.  This construction involves initially reversing the trajectory associated with the path in the backward search tree, followed by concatenating this reversed trajectory with the trajectory associated with the path in the forward search tree. In this paper, we formally define the reversal and concatenation operations and thoroughly validate that the results of both operations satisfy the given hybrid dynamics.
In practice, HyRRT-Connect always tolerates some distance between states in the forward and backward search trees, due to the randomness in state and input selection. However, this tolerance can result in discontinuities along the flow of the constructed motion plan. To address this issue,  the trajectory associated with the backward path is reconstructed through simulation from the final state of the forward path while applying the reversed input from the backward path. By ensuring the same hybrid time domain as the reversal of the backward path, we guarantee that as the tolerance decreases, the reconstructed motion plan's endpoint converges to the final state set.
 
\subsection{Organization}
The remainder of the paper is structured as follows. Section~\ref{section:preliminaries} presents notation and preliminaries. Section~\ref{section:problemstatement} presents the problem statement and \pn{introduces} an example. Section \ref{section:hyrrtconnect} presents HyRRT-Connect algorithm. Section \ref{section:solutionchecking} presents the reconstruction process. Section \ref{section:simulation}  \pn{illustrates} HyRRT-Connect \pn{in an example}. Section \ref{section:parallelcomputation} discusses the parallelization and property of HyRRT-Connect Algorithm.
\section{Notation and Preliminaries}\label{section:preliminaries}
\subsection{Notation}
The real numbers are denoted as $\mathbb{R}$, its nonnegative subset is denoted as $\mathbb{R}_{\geq 0}$\pn{,} and its nonpositive subset is denoted as $\mathbb{R}_{\leq 0}$. The set of nonnegative integers is denoted as $\mathbb{N}$. The notation $\interior I$ denotes the interior of the interval $I$. The notation $\overline{S}$ denotes the closure of the set $S$. The notation $\partial S$ denotes the boundary of the set $S$. Given sets $P\subset\mathbb{R}^n$ and $Q\subset\mathbb{R}^n$, the Minkowski sum of $P$ and $Q$, denoted as $P + Q$, is the set $\{p + q: p\in P, q\in Q\}$. 
The notation $|\cdot|$ denotes the Euclidean norm.  The notation $\rge f$ denotes the range of the function $f$.
	Given a point $x\in \mathbb{R}^{n}$ and a subset $S\subset \mathbb{R}^{n}$, the distance between $x$ and $S$ is denoted $|x|_{S} := \inf_{s\in S} |x - s|$. The notation $\mathbb{B}$ denotes the closed unit ball of appropriate dimension in the Euclidean norm.


\subsection{Preliminaries}
A hybrid system $\mathcal{H}$ with inputs is modeled as  \cite{sanfelice2021hybrid}
\begin{equation}
\mathcal{H}: \left\{              
\begin{aligned}               
\dot{x} & = f(x, u)     &(x, u)\in C\\                
x^{+} & =  g(x, u)      &(x, u)\in D\\                
\end{aligned}   \right. 
\label{model:generalhybridsystem}
\end{equation}
where $x\in \mathbb{R}^n$ represents the state, $u\in \mathbb{R}^m$ represents the input, $C\subset \mathbb{R}^{n}\times\mathbb{R}^{m}$ represents the flow set, $f: \mathbb{R}^{n}\times\mathbb{R}^{m} \to \mathbb{R}^{n}$ represents the flow map, $D\subset \mathbb{R}^{n}\times\mathbb{R}^{m}$ represents the jump set, and $g:\mathbb{R}^{n}\times\mathbb{R}^{m} \to \mathbb{R}^{n}$ represents the jump map. The continuous evolution of $x$ is captured by the flow map $f$. The discrete evolution of $x$ is captured by the jump map $g$. The flow set $C$ collects the points where the state can evolve continuously. The jump set $D$ collects the points where jumps can occur.
Given a flow set $C$, the set $U_{C} := \{u\in \mathbb{R}^{m}: \exists x\in \mathbb{R}^{n}\text{ such that } (x, u)\in C\}$ includes all possible input values that can be applied during flows. Similarly, given a jump set $D$, the set $U_{D} := \{u\in \mathbb{R}^ {m}: \exists x\in \mathbb{R}^{n}\text{ such that } (x, u)\in D\}$ includes all possible input values that can be applied at jumps. These sets satisfy $C\subset \mathbb{R}^{n}\times U_{C}$ and $D\subset \mathbb{R}^{n}\times U_{D}$. Given a set $K\subset \mathbb{R}^{n}\times U_{\star}$, where $\star$ is either $C$ or $D$, we define
	$
	\Pi_{\star}(K) := \{x: \exists u\in U_{\star} \text{ such that } (x, u)\in K\}
	$
	as the projection of $K$ onto $\mathbb{R}^{n}$, and define 
	\begin{equation}
		\label{equation:Cprime}
		C' := \Pi_{C}(C)
		\end{equation} and 
		\begin{equation}
		\label{equation:Dprime}
		D' := \Pi_{D}(D).
		\end{equation}
In addition to ordinary time $t\in \mathbb{R}_{\geq 0}$, we employ $j\in \mathbb{N}$ to denote the number of jumps of the evolution of $x$ and $u$ for $\mathcal{H}$ in (\ref{model:generalhybridsystem}), leading to hybrid time $(t, j)$ for the parameterization of its solutions and inputs. 
The domain of a solution to $\mathcal{H}$, hybrid input/arc, solution pair to $\mathcal{H}$, and concatenation operation are defined as follows\pn{; see \cite{sanfelice2021hybrid} for more details}. 
\begin{definition}[(Hybrid time domain)]\label{definition:hybridtimedomain}
	A subset $E\subset \mypreals\times\nnumbers$ is a compact hybrid time domain if
	$$
		E = \bigcup_{j = 0}^{J - 1} ([t_{j}, t_{j + 1}]\times \{j\})
	$$
	for some finite sequence of times $0=t_{0}\leq t_{1}\leq t_{2}\leq \ldots \leq t_{J}$.
	
	A set $E \subset \mypreals\times\nnumbers$ is a hybrid time domain if it is the union of a nondecreasing sequence of compact hybrid time domain, namely, $E$ is the union of compact hybrid time domains $E_j$ with the property that $E_0 \subset E_1 \subset E_2 \subset \ldots \subset E_j \ldots$.
\end{definition}
The input to a hybrid system is defined as follows.
\begin{definition}[(Hybrid input)]\label{definition:hybridinput}
	A function $\myinputt: \dom \myinputt \to \myreals[m]$ is a hybrid input if $\dom \myinputt$ is a hybrid time domain and if, for each $j\in \nnumbers$, $t\mapsto \myinputt(t,j)$ is Lebesgue measurable and locally essentially bounded on the interval $I^{j}_{\myinputt}:=\{t:(t, j)\in \dom \myinputt\}$. 
\end{definition}
A hybrid arc  describes the state trajectory of the system as defined next.
\begin{definition}[(Hybrid arc)]\label{definition:hybridarc}
	A function $\mystatet: \dom \mystatet \to \myreals[n]$ is a hybrid arc if $\dom \mystatet$ is a hybrid time domain and if, for each $j\in \nnumbers$, $t\mapsto \mystatet(t,j)$ is locally absolutely continuous on the interval $I^{j}_{\mystatet}:=\{t:(t, j)\in \dom \mystatet\}$. 
\end{definition}
The definition of solution pair to a hybrid system is given as follows.
\begin{definition}[(Solution pair to a hybrid system)]
	\label{definition:solution}
	 A hybrid input $\myinputt$ and a hybrid arc $\mystatet$ define a solution pair $(\mystatet, \myinputt)$ to the hybrid system $\hs$ as in (\ref{model:generalhybridsystem}) if
	\begin{enumerate}[label=\arabic*)]
		\item $(\mystatet(0,0), \myinputt(0,0))\in \overline{C} \cup D$ and $\dom \mystatet = \dom \myinputt (= \dom (\mystatet, \myinputt))$.
		\item For each $j\in \mathbb{N}$ such that $I^{j}_{\mystatet}$ has nonempty interior $\interior(I^{j}_{\mystatet})$, $(\mystatet, \myinputt)$ satisfies
		$
		(\mystatet(t, j),\myinputt(t, j))\in C
		$ for all $t\in \interior I^{j}_{\mystatet}$, and 
		$
		\frac{\text{d} \mystatet}{\text{d} t}(t,j) = f(\mystatet(t,j), \myinputt(t,j))
		$ for almost all $t\in I^{j}_{\mystatet}$.
		\item For all $(t,j)\in \dom (\mystatet, \myinputt)$ such that $(t,j + 1)\in \dom (\mystatet, \myinputt)$, 
		\begin{equation}
			\begin{aligned}
			(\mystatet(t, j), \myinputt(t, j))&\in D\\
			\mystatet(t,j+ 1) &= g(\mystatet(t,j), \myinputt(t, j)).
			\end{aligned}
			\end{equation}
	\end{enumerate}
\end{definition}
The concatenation operation of solution pairs in \cite[Definition 5.1]{wang2025motion}  is given next.
\begin{definition}[(Concatenation operation)]
	\label{definition:concatenation}
	Given two functions $\phi_{1}: \dom \phi_{1} \to \mathbb{R}^{n}$ and $\phi_{2}:\dom \phi_{2} \to \mathbb{R}^{n}$, where $\dom \phi_{1}$ and $\dom \phi_{2}$ are hybrid time domains, $\phi_{2}$ can be concatenated to $\phi_{1}$ if $ \phi_{1}$ is compact and $\phi: \dom \phi \to \mathbb{R}^n$ is the concatenation of $\phi_{2}$ to $\phi_{1}$, denoted $\phi = \phi_{1}|\phi_{2}$, namely,
	\begin{enumerate}[label=\arabic*)]
		\item $\dom \phi = \dom \phi_{1} \cup (\dom \phi_{2} + \{(T, J)\}) $, where $(T, J) = \max \dom \phi_{1}$ and the plus sign denotes Minkowski addition;
		\item $\phi(t, j) = \phi_{1}(t, j)$ for all $(t, j)\in \dom \phi_{1}\backslash \{(T, J)\}$ and $\phi(t, j) = \phi_{2}(t - T, j - J)$ for all $(t, j)\in \dom \phi_{2} + \{(T, J)\}$.
	\end{enumerate}
\end{definition}
\section{Problem Statement and Applications}
The motion planning problem for hybrid systems studied in this paper is given in \cite[Problem 1]{wang2025motion} as follows, which, for convenience, is reproduced next.
\begin{problem}[Motion planning problem for hybrid systems]\label{problem:motionplanning}
	Given a hybrid system $\mathcal{H}$ as in (\ref{model:generalhybridsystem}) with input $u\in \myreals[m]$ and state $x\in \myreals[n]$, initial state set $X_{0}\subset\myreals[n]$,  final state set $X_{f}\subset\myreals[n]$, and unsafe set \endnote{The unsafe set $X_{u}$ can be used to constrain both the state and the input.} $X_{u}\subset\myreals[n]\times \myreals[m]$, find $(\mystatet, \myinputt): \dom (\mystatet, \myinputt)\to \mathbb{R}^{n}\times \mathbb{R}^{m}$, namely, \emph{a motion plan}, such that, for some $(T, J)\in \dom (\mystatet, \myinputt)$, the following hold:
	\begin{enumerate}[label=\arabic*)]
		\item $\mystatet(0,0) \in X_{0}$, namely, the initial state of the solution belongs to the given initial state set $X_{0}$;
		\item $(\mystatet, \myinputt)$ is a solution pair to $\mathcal{H}$ as defined in Definition~\ref{definition:solution};
		\item $(T,J)$ is such that $\mystatet(T,J)\in X_{f}$, namely, the solution belongs to the final state set at hybrid time $(T, J)$;
		\item $(\mystatet(t,j), \myinputt(t, j))\notin X_{u}$ for each $(t,j)\in \dom (\mystatet, \myinputt)$ such that $t + j \leq T+ J$, namely, the solution pair does not intersect with the unsafe set before its state trajectory reaches the final state set.
	\end{enumerate}
Therefore, given sets $X_{0}$, $X_{f}$\pn{,} and $X_{u}$, and a hybrid system $\mathcal{H}$ as in (\ref{model:generalhybridsystem}) with data $(C, f, D, g)$, a motion planning problem $\mathcal{P}$ is formulated as
$
	\mathcal{P} = (X_{0}, X_{f}, X_{u}, (C, f, D, g)).
$
\end{problem}
\begin{figure}[htbp]
		\centering
	\def\svgwidth{0.9\columnwidth}
\begingroup%
  \makeatletter%
  \providecommand\color[2][]{%
    \errmessage{(Inkscape) Color is used for the text in Inkscape, but the package 'color.sty' is not loaded}%
    \renewcommand\color[2][]{}%
  }%
  \providecommand\transparent[1]{%
    \errmessage{(Inkscape) Transparency is used (non-zero) for the text in Inkscape, but the package 'transparent.sty' is not loaded}%
    \renewcommand\transparent[1]{}%
  }%
  \providecommand\rotatebox[2]{#2}%
  \newcommand*\fsize{\dimexpr\f@size pt\relax}%
  \newcommand*\lineheight[1]{\fontsize{\fsize}{#1\fsize}\selectfont}%
  \ifx\svgwidth\undefined%
    \setlength{\unitlength}{460.76175145bp}%
    \ifx\svgscale\undefined%
      \relax%
    \else%
      \setlength{\unitlength}{\unitlength * \real{\svgscale}}%
    \fi%
  \else%
    \setlength{\unitlength}{\svgwidth}%
  \fi%
  \global\let\svgwidth\undefined%
  \global\let\svgscale\undefined%
  \makeatother%
  \begin{picture}(1,0.6529201)%
    \lineheight{1}%
    \setlength\tabcolsep{0pt}%
    \put(0,0){\includegraphics[width=\unitlength,page=1]{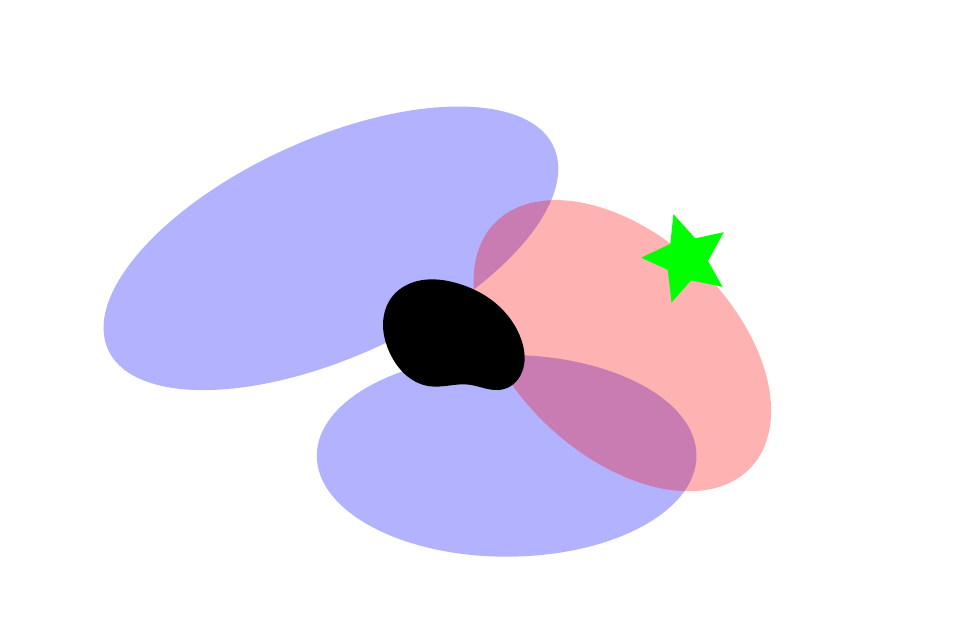}}%
    \put(0.05417655,0.20361801){\makebox(0,0)[lt]{\lineheight{1.25}\smash{\begin{tabular}[t]{l}initial state\end{tabular}}}}%
    \put(0.8141166,0.36989962){\makebox(0,0)[lt]{\lineheight{1.25}\smash{\begin{tabular}[t]{l}final state\end{tabular}}}}%
    \put(0.66638215,0.49017024){\makebox(0,0)[lt]{\lineheight{1.25}\smash{\begin{tabular}[t]{l}unsafe set\end{tabular}}}}%
    \put(0,0){\includegraphics[width=\unitlength,page=2]{hybridmotionplanning.pdf}}%
    \put(0.20653165,0.11709806){\makebox(0,0)[lt]{\lineheight{1.25}\smash{\begin{tabular}[t]{l}flow set\end{tabular}}}}%
    \put(0.82743469,0.26413127){\makebox(0,0)[lt]{\lineheight{1.25}\smash{\begin{tabular}[t]{l}jump set\end{tabular}}}}%
    \put(0,0){\includegraphics[width=\unitlength,page=3]{hybridmotionplanning.pdf}}%
  \end{picture}%
\endgroup%

		\caption{A motion plan to Problem \ref{problem:motionplanning}. The green square denote the initial state set. The green star denotes the final state set. The blue region denotes the flow set. The red region denotes the jump set. The solid blue lines denote flow and the dotted red lines denote jumps in the motion plan.}
\end{figure}
We illustrate this problem in two examples.
\begin{example}[(Actuated bouncing ball system)]\label{example:bouncingball}
	Consider a ball bouncing on a fixed horizontal surface as is shown in Figure \ref{fig:bouncingball}. The surface is located at the origin and, through control actions, is capable of affecting the velocity of the ball after the impact.  The dynamics of the ball while in the air are given by
	\begin{equation}
		\label{model:bouncingballflow}
		\dot{x} = \left[ \begin{matrix}
		x_{2} \\
		-\gamma
		\end{matrix}\right] =: f(x, u)\qquad (x, u)\in C
		\end{equation}
	where $x :=(x_{1}, x_{2})\in \mathbb{R}^2$, \sj{the height of the ball is denoted by $x_{1}$, and the velocity of the ball is denoted by $x_{2}$.} The gravity constant is denoted by $\gamma$. 
	\begin{figure}[htbp] 
			\centering
			\parbox[h]{0.49\columnwidth}{
				\centering
				\subfigure[The actuated bouncing ball system\label{fig:bouncingball} ]{%
	\def\svgwidth{0.49\columnwidth}
\begingroup%
  \makeatletter%
  \providecommand\color[2][]{%
    \errmessage{(Inkscape) Color is used for the text in Inkscape, but the package 'color.sty' is not loaded}%
    \renewcommand\color[2][]{}%
  }%
  \providecommand\transparent[1]{%
    \errmessage{(Inkscape) Transparency is used (non-zero) for the text in Inkscape, but the package 'transparent.sty' is not loaded}%
    \renewcommand\transparent[1]{}%
  }%
  \providecommand\rotatebox[2]{#2}%
  \newcommand*\fsize{\dimexpr\f@size pt\relax}%
  \newcommand*\lineheight[1]{\fontsize{\fsize}{#1\fsize}\selectfont}%
  \ifx\svgwidth\undefined%
    \setlength{\unitlength}{155.66442535bp}%
    \ifx\svgscale\undefined%
      \relax%
    \else%
      \setlength{\unitlength}{\unitlength * \real{\svgscale}}%
    \fi%
  \else%
    \setlength{\unitlength}{\svgwidth}%
  \fi%
  \global\let\svgwidth\undefined%
  \global\let\svgscale\undefined%
  \makeatother%
  \begin{picture}(1,0.86048663)%
    \lineheight{1}%
    \setlength\tabcolsep{0pt}%
    \put(0,0){\includegraphics[width=\unitlength,page=1]{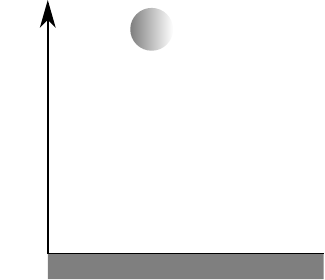}}%
    \put(0.03673765,0.65008529){\rotatebox{90}{\makebox(0,0)[rt]{\lineheight{1.25}\smash{\begin{tabular}[t]{r}height ($x_{1}$)\end{tabular}}}}}%
    \put(0,0){\includegraphics[width=\unitlength,page=2]{bouncingballfigure.pdf}}%
    \put(0.55137666,0.30449301){\makebox(0,0)[lt]{\lineheight{1.25}\smash{\begin{tabular}[t]{l}control\\input $(u)$\end{tabular}}}}%
  \end{picture}%
\endgroup%

}
			}
			\parbox[h]{0.49\columnwidth}{
				\centering
				\subfigure[A motion plan to the sample motion planning problem for actuated  bouncing ball system.\label{fig:samplesolutionbb}]{\includegraphics[width=0.46\columnwidth]{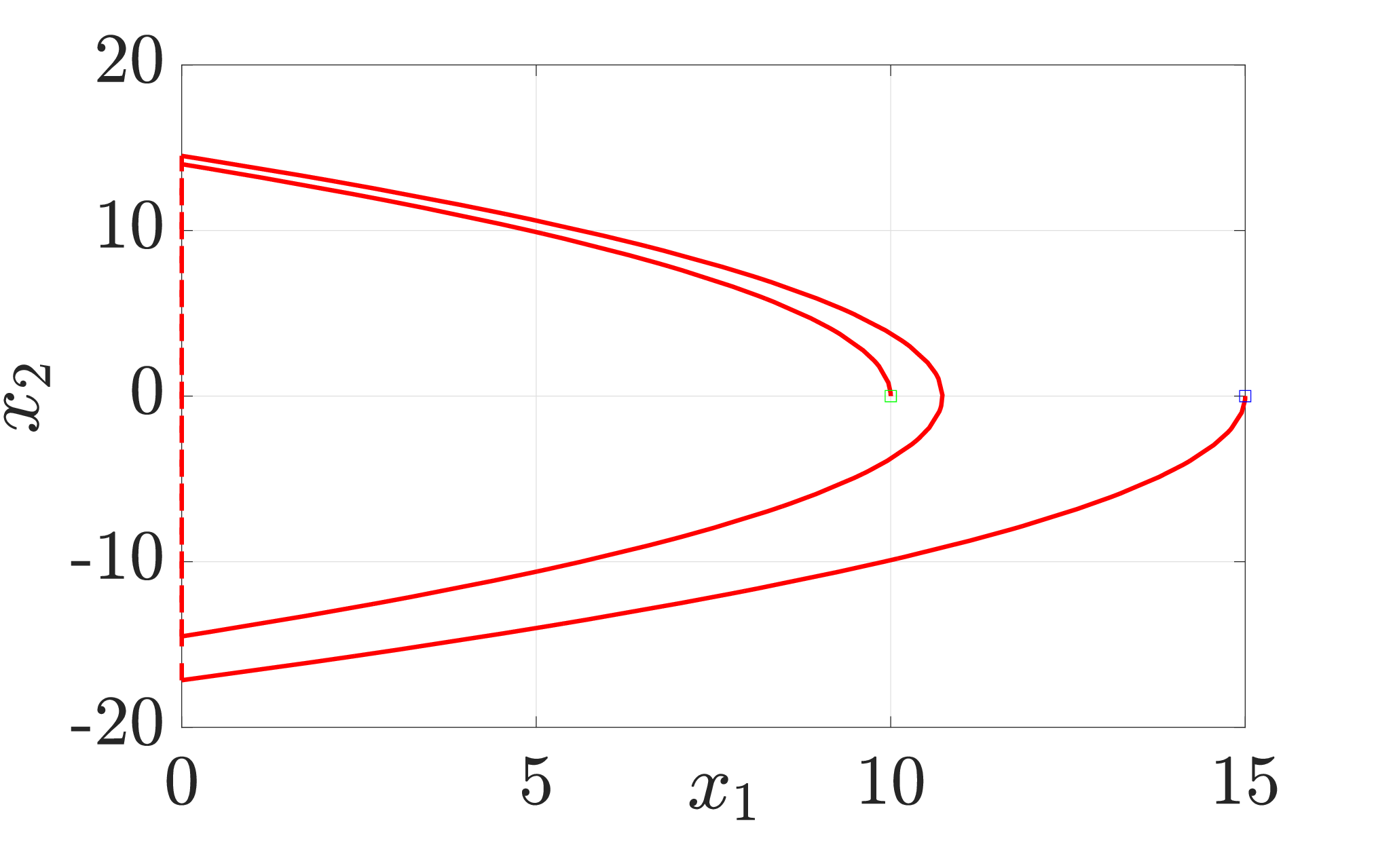}}
			}
			\caption{The actuated bouncing ball system in Example \ref{example:bouncingball}.}
	\end{figure}
	\sj{The ball is allowed to flow when the ball is \pn{on or} above the surface.} Hence, the flow set is
	\begin{equation}
		\label{model:bouncingballflowset}
		C := \{(x, u)\in \mathbb{R}^{2}\times \mathbb{R}: x_{1}\geq 0\}.
		\end{equation}
	At every impact, and with control input equal to zero, the velocity of the ball changes from negative to positive while the height remains the same. \sj{The dynamics at jumps are given by}
	\begin{equation}
		x^{+} = \left[ \begin{matrix}
		x_{1} \\
		-\lambda x_{2}+u
		\end{matrix}\right] =: g(x, u)\qquad (x, u)\in D
		\label{conservationofmomentum}
		\end{equation}
	where $u\geq 0$ is the input applied only at jumps and $\lambda \in (0,1)$ is the coefficient of restitution. 
	Jumps are allowed when the ball is on the surface with nonpositive velocity. Hence, the jump set is 
	\begin{equation}
		\label{model:bouncingballjumpset}
		D:= \{(x, u)\in \mathbb{R}^{2}\times \mathbb{R}: x_{1} = 0, x_{2} \leq 0, u\geq 0\}.
		\end{equation}


	An example of a motion planning problem for the actuated bouncing ball system is as follows: using a bounded input signal, find a solution pair to (\ref{model:generalhybridsystem}) when the bouncing ball is released from a certain height with zero velocity and such that it eventually reaches a given target height with zero velocity. The input is constrained to be positive and upper bounded by a given positive real number. To complete this task, not only the values of the input, but also the hybrid time domain of the input need to be planned properly such that the ball can reach the desired target. One such motion planning problem is given by defining the initial state set as $X_{0} = \{(14, 0)\}$, the final state set as $X_{f} = \{(10, 0)\}$, and the unsafe set as $X_{u} =\{(x, u)\in \mathbb{R}^{2}\times \mathbb{R}: u\in (-\infty, 0]\cup[5, \infty)\}$.  The motion planning problem $\mathcal{P}$ is given as $\mathcal{P} = (X_{0}, X_{f}, X_{u}, (C, f, D, g))$. We solve this motion planning problem later in this paper.
\end{example}

\begin{example}[(Walking robot)]\label{example:biped}
	The state $x$ of the compass model of a walking robot is composed of the angle vector $\theta$ and the velocity vector $\omega$ \cite{grizzle2001asymptotically}. 
	The angle vector $\theta$ contains the planted leg angle $\theta_{p}$, the swing leg angle $\theta_{s}$, and the torso angle $\theta_{t}$. The velocity vector $\omega$ contains the planted leg angular velocity $\omega_{p}$, the swing leg angular velocity $\omega_{s}$, and the torso angular velocity $\omega_{t}$. The input $u$ is the input torque, where $u_{p}$ is the torque applied on the planted leg from the ankle, $u_{s}$ is the torque applied on the swing leg from the hip, and $u_{t}$ is the torque applied on the torso from the hip.
	The continuous dynamics of $x = (\theta, \omega)$ are obtained from the Lagrangian method and are given by 
	$$\dot{\theta} = \omega,$$
	$$ \dot{\omega} = D_{f}(\theta)^{-1}( - C_{f}(\theta, \omega)\omega - G_{f}(\theta) + Bu) = : \alpha(x, u),$$
	where $D_{f}$ and $C_{f}$ are the inertial and Coriolis matrices, respectively, and $B$ is the actuator  relationship matrix.
	
	In \cite{short2018hybrid}, the input torques that produce an acceleration, denoted $a$, for a specific state $x$ are determined by a function $\mu$, defined as 
	$$\mu(x, a) := B^{-1}(D_{f}(\theta)a + C_{f}(\theta, \omega)\omega + G_{f}(\theta)).
	$$
	By applying $u = \mu(x, a)$ to $\dot{\omega} = \alpha(x, u)$, we obtain
	$$
	\label{equation:omegaflowmap}
	\dot{\omega} = a.
	$$
	Then, the flow map $f$ is defined as
		$$
		\label{model:bipedflowmap}
		f(x, a) := \left[\begin{matrix}
		\omega\\
		a
		\end{matrix}\right] \quad \forall  (x, a)\in C.
		$$
	Flow is allowed when only one leg is in contact with the ground. To determine if the biped has reached the end of a step, we define
	$$
	h(x) := \phi_{s} - \theta_{p}$$ for all $x\in \mathbb{R}^{6}
	$
	where $\phi_{s}$ denotes the step angle.
	The condition $h(x) \geq 0$ indicates that only one leg is in contact with the ground. Thus, the flow set is defined as 
	$$
	\label{model:bipedflowset}
	C:= \{(x, a)\in \mathbb{R}^{6}\times \mathbb{R}^{3}: h(x)\geq 0\}.
	$$
	Furthermore, a step occurs when the change of $h$ is such that $\theta_{p}$  is approaching $\phi_{s}$, and $h$ equals zero. Thus, the jump set $D$ is defined as
	$$
	\label{model:bipedjumpset}
	D := \{(x, a)\in \mathbb{R}^{6}\times \mathbb{R}^{3}: h(x) = 0, \omega_{p} \geq 0\}.
	$$
	Following \cite{grizzle2001asymptotically}, when a step occurs, the swing leg becomes the planted leg, and the planted leg becomes the swing leg. The function $\Gamma$ is defined to swap angles and velocity variables as
	$
	\label{equation:thetajumpmap}
	\theta^{+} = \Gamma(\theta).
	$
	The angular velocities after a step are determined by a contact model denoted as
	$\Omega(x) := (\Omega_{p}(x), \Omega_{s}(x), \Omega_{t}(x))$, where $\Omega_{p}$, $\Omega_{s}$, and $\Omega_{t}$ are the angular velocity of the planted leg, swing leg, and torso, respectively. Then, the jump map $g$ is defined as
	\begin{equation}
	\label{equation:bipedjumpmap}
	g(x, a) := \left[\begin{matrix}
	\Gamma(\theta)\\
	\Omega(x)
	\end{matrix}\right]\quad \forall (x, a)\in D.
	\end{equation} For more information about the contact model, see \cite[Appendix A]{grizzle2001asymptotically}.
	
	An example of a motion planning problem for the walking robot system is as follows: using a bounded input signal, find a solution pair to (\ref{model:generalhybridsystem}) associated to the walking robot system completing a step of a walking circle. One way to characterize a walking cycle is to define the final state and the initial state as the states before and after a jump occurs, respectively.  One such motion planning problem is given by defining the final state set as $X_{f} = \{(\phi_{s}, -\phi_{s}, 0, 0.1, 0.1, 0)\}$, the initial state set as $X_{0} = \{x_{0}\in \mathbb{R}^{6}: x_{0} = g(x_{f}, 0), x_{f}\in X_{f}\}$ where the input argument of $g$ can be set arbitrarily because input does not affect the value of $g$; see (\ref{equation:bipedjumpmap}), and the unsafe set as $X_{u} = \{(x, a)\in \mathbb{R}^{6}\times \mathbb{R}^{3}: a_{1} \notin [a_{1}^{\min}, a_{1}^{\max}]\text{ or }  a_{2} \notin [a_{2}^{\min}, a_{2}^{\max}]\text{ or } a_{3} \notin [a_{3}^{\min}, a_{3}^{\max}]\text{ or } (x, a)\in D  \}$, where $a_{1}^{\min}$, $a_{2}^{\min}$, and $a_{3}^{\min}$ are the lower bounds of $a_{1}$, $a_{2}$, and  $a_{3}$, respectively, and $a_{1}^{\max}$, $a_{2}^{\max}$, and $a_{3}^{\max}$ are the upper bounds of $a_{1}$, $a_{2}$, and  $a_{3}$, respectively.
We also solve this motion planning problem later in this paper.
	\begin{figure}[htbp] 
	\centering
	\def\svgwidth{0.5\columnwidth}
	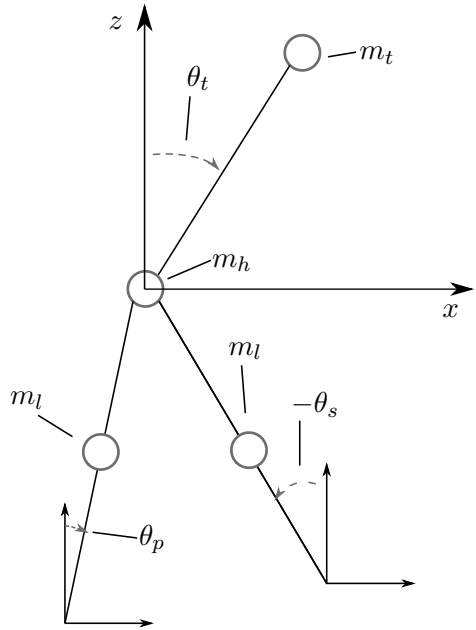

	\caption{The biped system in Example \ref{example:biped}. The angle vector $\theta$ contains the planted leg angle $\theta_{p}$, the swing leg angle $\theta_{s}$, and the torso angle $\theta_{t}$. The velocity vector $\omega$ contains the planted leg angular velocity $\omega_{p}$, the swing leg angular velocity $\omega_{s}$, and the torso angular velocity $\omega_{t}$. The input $u$ is the input torque, where $u_{p}$ is the torque applied on the planted leg from the ankle, $u_{s}$ is the torque applied on the swing leg from the hip, and $u_{t}$ is the torque applied on the torso from the hip.} \label{fig:bipedsystem} 
\end{figure} 
\end{example}\label{section:problemstatement}
\section{Algorithm Description}\label{section:hyrrtconnect}
\subsection{Overview}\label{section:algorithmoverview}
\ifbool{conf}{\vspace{-0.15cm}}{}
In this section, a bidirectional RRT-type motion planning algorithm for hybrid systems, called HyRRT-Connect, is proposed. HyRRT-Connect searches for a motion plan by incrementally constructing two search trees: one that starts from the initial state set and propagates forward in hybrid time, while the other tree starts from the final state set and propagates backward in hybrid time. Upon detecting overlaps between the two search trees, the trees are connected, subsequently yielding a motion plan, \pn{as described} in Section \ref{section:solutionchecking}.
Each search tree is modeled by a directed tree. A directed tree $\mathcal{T}$ is a pair $\mathcal{T} = (V, E)$, where $V$ is a set whose elements are called vertices and $E$ is a set of paired vertices whose elements are called edges. The edges in the directed tree are directed, which means the pairs of vertices that represent edges are ordered. The set of edges $E$ is defined as
	$
	E \subseteq \{(v_{1}, v_{2}): v_{1}\in V, v_{2}\in V, v_{1}\neq v_{2}\}.
	$
	The edge $e = (v_{1}, v_{2})\in E$ represents an edge from $v_{1}$ to $v_{2}$.
A path in $\mathcal{T} = (V, E)$ is a sequence of vertices
\begin{equation}\label{eq:path}
	p = (v_{1}, v_{2}, ..., v_{k})
\end{equation}
such that $(v_{i}, v_{i + 1})\in E$ for all $i= \{1, 2,..., k - 1\}$.

The search tree constructed forward in hybrid time is denoted as $\fw{\stree} = (\fw{V}, \fw{E})$ and the search tree constructed backward in hybrid time is denoted as $\bw{\stree} = (\bw{V}, \bw{E})$. \sj{For consistency}, we denote $\hs$ in (\ref{model:generalhybridsystem}) as $\fw{\hs} = (\fw{C}, \fw{f}, \fw{D}, \fw{g})$. Each vertex $v$ in $\fw{V}$(respectively, $\bw{V}$) is associated with a state of $\fw{\hs}$ (respectively, the hybrid system that represents inverse dynamics of $\fw{\hs}$, denoted $\bw{\hs}$), denoted $\overline{x}_{v}$. Each edge $e$ in $\fw{E}$ (respectively, $\bw{E}$) is associated with a solution pair to $\fw{\hs}$ (respectively, $\bw{\hs}$), denoted $\overline{\psi}_{e}$, that connects the states associated with their endpoint vertices. 
The solution pair associated with the path $p = (v_{1}, v_{2}, ..., v_{k})$ is the concatenation of all the solutions \sj{associated with} the edges therein, namely,
\begin{equation}
	\label{equation:concatenationpath}
	\tilde{\psi}_{p} := \overline{\psi}_{(v_{1}, v_{2})}|\overline{\psi}_{(v_{2}, v_{3})}|\ ...\  |\overline{\psi}_{(v_{k-1}, v_{k})}
	\end{equation}
where $\tilde{\psi}_{p}$ denotes the solution pair associated with $p$. For the concatenation notion, see Definition \ref{definition:concatenation}.

The proposed HyRRT-Connect algorithm requires a library of possible inputs to construct $\fw{\stree}$, denoted $\fw{\il} = (\fw{\il}_{C}, \fw{\il}_{D})$, and to construct $\bw{\stree}$, denoted $\bw{\il} = (\bw{\il}_{C}, \bw{\il}_{D})$. The input library $\fw{\il}$ (respectively, $\bw{\il}$) includes the input signals for the flows of $\fw{\hs}$ (respectively, $\bw{\hs}$), collected in $\fw{\il}_{C}$ (respectively, $\bw{\il}_{C}$), and the input values for the jumps of $\fw{\hs}$ (respectively, $\bw{\hs}$), collected in $\fw{\il}_{D}$ (respectively, $\bw{\il}_{D}$).

HyRRT-Connect addresses the motion planning problem $\mathcal{P} = (X_{0}, X_{f}, X_{u}, (\fw{C}, \fw{f}, \fw{D}, \fw{g}))$ using input libraries $\fw{\il}$ and $\bw{\il}$ through the following steps:
\begin{steps}
	\item Sample a finite number of points from $X_{0}$ (respectively, $X_{f}$) and initialize a search tree $\fw{\stree} = (\fw{V}, \fw{E})$ (respectively, $\bw{\stree} = (\bw{V}, \bw{E})$) by adding vertices associated with each sample.
	\item Incrementally construct $\fw{\stree}$ forward in hybrid time and $\bw{\stree}$ backward in hybrid time, executing both procedures in an interleaved manner\endnote{The information used in forward propagation remains unaffected by the backward propagation, and vice versa. This independence allows for the potential of executing both forward and backward propagation simultaneously, as opposed to an interleaved approach. Nonetheless, the nature of motion planning algorithm, which involves frequently joining the results computed in parallel to check overlaps between the forward and backward search trees, may prevent the computation improvement associated with parallel computation. Additional comments and experimental results can be found in the forthcoming Section ~\ref{section:parallelcomputation}.}.
	\item If an appropriate overlap between $\fw{\stree}$ and $\bw{\stree}$ is found, reverse the solution pair in $\bw{\stree}$, concatenate it to the solution pair in $\fw{\stree}$\pn{,} and return the concatenation result.
\end{steps}
An appropriate overlap between $\fw{\stree}$ and $\bw{\stree}$ implies the paths in both trees that can collectively be used to construct a motion plan to $\mathcal{P}$. Details on these overlaps are discussed in Section \ref{section:solutionchecking}.

\ifbool{conf}{\vspace{-0.20cm}}{}
\subsection{Backward-in-time Hybrid System}
\ifbool{conf}{\vspace{-0.12cm}}{}
In the HyRRT-Connect algorithm, a hybrid system that represents backward-in-time dynamics of $\fw{\hs} = (\fw{C}, \fw{f}, \fw{D}, \fw{g})$, denoted $\bw{\hs} = (\bw{C}, \bw{f}, \bw{D}, \bw{g})$, is required when propagating trajectories from $X_{f}$. The construction of $\mathcal{H}^{\text{bw}}$ is as follows.
\begin{definition}[(Backward-in-time hybrid system)]\label{definition:bwhs} Given a hybrid system $
	\fw{\hs} = (\fw{C}, \fw{f}, \fw{D}, \fw{g})$, the backward-in-time hybrid system of $\fw{\hs}$, denoted $\bw{\hs}$, \sj{is defined as}
	\begin{equation}
	\bw{\hs}: \left\{              
	\begin{aligned}               
	\dot{x} & =  \bw{f}(x, u)     &(x, u)\in \bw{C}\\                
	x^{+} & \in  \bw{g}(x, u)      &(x, u)\in \bw{D}\\                
	\end{aligned}   \right. 
	\label{model:generalhybridbackwardsystem}
	\end{equation}
	where
	\begin{enumerate}[label = \arabic*)]
		\item The backward-in-time flow set is constructed as
		\begin{equation}
			\label{equation:backwardflowset}
			\bw{C}:= \fw{C}.
			\end{equation}
		\item The backward-in-time flow map is constructed as
	\begin{equation}
			\label{equation:backwardflow}
			\begin{aligned}
			\bw{f}(x, u) :=  -\fw{f}(x, u) \quad\pnc{\forall} (x, u)\in \bw{C}.
			\end{aligned}
			\end{equation}
		
		\item The backward-in-time jump map is constructed as
		\begin{equation}
			\label{set:gbw}
			\begin{aligned}
			\bw{g}(x, u) &:= \{z\in \mathbb{R}^{n}: x = \fw{g}(z, u),\\
			&(z, u)\in \fw{D}\}\quad \forall (x, u)\in \mathbb{R}^{n}\times \mathbb{R}^{m}.
			\end{aligned}
			\end{equation}
		
		\item The backward-in-time jump set is constructed as
		\begin{equation}
			\label{set:dbw}
			\begin{aligned}
			\bw{D} &:= \{(x, u)\in \mathbb{R}^{n}\times \mathbb{R}^{m}:\\
			&\exists z\in\myreals[n]: x = \fw{g}(z, u), (z, u)\in \fw{D}\}.
			\end{aligned}
			\end{equation}
	\end{enumerate}
\end{definition}
The above shows a general method of constructing the backward-in-time system given the forward-in-time hybrid system.
While the jump map $\fw{g}$ of the forward-in-time system $\fw{\hs}$ is single-valued, the corresponding map $\bw{g}$ in $\bw{\hs}$ may not be, especially if $\bw{g}$ is not invertible. Therefore, a difference inclusion in (\ref{model:generalhybridbackwardsystem}) \sj{governs the discrete dynamics}. The backward-in-time jump map $\bw{g}$ is defined on the preimage of $\fw{g}$ for any available input.

We now illustrate the construction of the backward-in-time hybrid system for the actuated bouncing ball system \pn{in Example \ref{example:bouncingball}. 

\begin{example}[(Actuated bouncing ball system in Example \ref{example:bouncingball}, revisited)]\label{example:bouncingballbw} The backward-in-time hybrid system of the actuated bouncing ball system is constructed as follows.
	\begin{itemize}
		\item From (\ref{equation:backwardflowset}), the backward-in-time flow set $C^{\text{bw}}$ is given by
		\begin{equation}
			\label{equation:bouncingballbwflowset}
			\bw{C}:= \fw{C} =  \{(x, u)\in \mathbb{R}^{2}\times \mathbb{R}: x_{1} \geq 0\}.
			\end{equation}
		\item From (\ref{equation:backwardflow}), the backward-in-time flow map $f^{\text{bw}}$ is given by
		\begin{equation}
			\label{equation:bouncingballbwflowmap}
			\bw{f}(x, u) := -\fw{f}(x, u) = \left[ \begin{matrix}
			-x_{2} \\
			\gamma
			\end{matrix}\right] \qquad \forall(x, u)\in \bw{C}.
			\end{equation}
		\item From (\ref{set:dbw}), the backward-in-time jump set is given by
		\begin{equation}
			\label{equation:bouncingballbwjumpset}
			\bw{D} := \{(x, u)\in \mathbb{R}^{2}\times \mathbb{R}: x_{1} = 0, x_{2} \geq u, u\geq 0\}. 
			\end{equation}
		\item From (\ref{set:gbw}), since $\lambda\in (0, 1\pn{]}$, the backward-in-time jump map $g^{\text{bw}}$ is given by
		\begin{equation}
			\label{equation:bouncingballbwjumpmap}
			\bw{g}(x, u) :=  \left[ \begin{matrix}
			x_{1} \\
			-\frac{x_{2}}{\lambda}+\frac{u}{\lambda}
			\end{matrix}\right]\ \forall(x, u)\in \bw{D}.
			\end{equation}
	\end{itemize}
	Note \pn{that when $\lambda = 0$, $g$ is not invertible. In this case,} the backward-in-time jump map is given by
	$$
		\bw{g}(x, u) :=  \{x_{1}\}\times\myreals_{\leq 0} \qquad \forall(x, u)\in \bw{D}
		$$
	where, as a difference to (\ref{equation:bouncingballbwjumpset}), the backward-in-time jump set is given by
	$$
		\bw{D} := \{(x, u)\in \mathbb{R}^{2}\times \mathbb{R}: x_{1} = 0, x_{2} = u \geq 0\}.
		$$

\end{example}
\begin{remark}
	The backward-in-time hybrid system $\bw{\hs}$ of $\fw{\hs}$ characterizes the state evolution backward in hybrid time, relative to $\fw{\hs}$ at the equivalent state and input.  To justify that the backward-in-time hybrid system is properly defined, a reversal operation is introduced in the forthcoming Definition~\ref{definition:reversedhybrdarc}. Through this operation, we demonstrate in Proposition~\ref{theorem:reversedarcbwsystem} that the reversal of the solution pair to $\bw{\hs}$ is a solution pair to $\fw{\hs}$.
\end{remark}}
\ifbool{conf}{\vspace{-0.35cm}}{}
\subsection{Construction of Motion Plans}\label{section:constructionmotionplan}
\ifbool{conf}{\vspace{-0.30cm}}{}
To construct a motion plan, HyRRT-Connect reverses a solution pair associated with a path detected in $\mathcal{T}^{\text{bw}}$ and concatenates it with a solution pair associated with a path detected in $\mathcal{T}^{\text{fw}}$, \sj{according to Definition \ref{definition:concatenation}. The following result shows that the resulting} concatenation is a solution pair to $\fw{\hs}$ under mild conditions. 
\begin{proposition}
	\label{theorem:concatenationsolution}
	Given two solution pairs $\psi_{1} = (\phi_{1}, \myinputt_{1})$ and $\psi_{2} = (\phi_{2}, \myinputt_{2})$ to a hybrid system $\fw{\hs}$, their concatenation $\psi = (\phi, \myinputt) = (\phi_{1}|\phi_{2}, \myinputt_{1}|\myinputt_{2})$, denoted $\psi = \psi_{1}|\psi_{2}$,  is a solution pair to $\fw{\hs}$ if the following hold:
	\begin{enumerate}[label = \arabic*)]
		\item $\psi_{1} = (\phi_{1}, \myinputt_{1})$ is compact;
		\item $\phi_{1}(T, J) = \phi_{2}(0,0)$, where $(T, J) = \max \dom \psi_{1}$;
		\item If both $I_{\psi_{1}}^{J}$ and $I_{\psi_{2}}^{0}$ have nonempty interior, where $I_{\psi}^{j} = \{t: (t, j)\in \dom \psi\}$ and $(T, J) = \max \dom \psi_{1}$, then $\psi_{2}(0, 0)\in C$.
	\end{enumerate}
\end{proposition}
\begin{proof}
See Appendix \ref{appendix:proof_concatenation}.
\end{proof}

\begin{remark}
		Item 1 in Proposition \ref{theorem:concatenationsolution} guarantees that $\psi_{2}$ can be concatenated to $\psi_{1}$. The concatenation operation defined in Definition \ref{definition:concatenation} suggests that if $\psi_{2}$ can be concatenated to $\psi_{1}$, $\psi_{1}$ is required to be compact. 
		Item 2 in Proposition \ref{theorem:concatenationsolution} guarantees that the concatenation $\psi$ satisfies the requirement of being absolutely continuous in Definition~\ref{definition:hybridarc} at hybrid time $(T, J)$, where $(T, J) = \max \dom \psi_{1}$.
		Item 3 in Proposition \ref{theorem:concatenationsolution} guarantees that the concatenation result $\psi$ satisfies item 2 in Definition \ref{definition:solution} at hybrid time $(T, J)$, where $(T, J) = \max \dom \psi_{1}$. Note that item 2 therein does not require that $\psi_{1}(T, J)\in C$ and $\psi_{2}(0,0)\in C$ since $T\notin \interior I_{\psi_{1}}^{J}$ and $0\notin \interior I_{\psi_{2}}^{0}$. However, $T$ may belong to the interior of $I_{\psi}^{J}$ after concatenation. Hence, item 3 guarantees that if $T$ belongs to the interior of $I_{\psi}^{J}$ after concatenation, $\psi$ still satisfies item 2 in Definition \ref{definition:solution}.
\end{remark}
The reversal operation is defined next.
\begin{definition}[(Reversal of a solution pair)]
	\label{definition:reversedhybrdarc}
	Given a compact solution pair $(\phi, \myinputt)$ to $\fw{\hs} = (\fw{C}, \fw{f}, \fw{D}, \fw{g})$, where $\phi: \dom \phi \to \mathbb{R}^{n}$, $\myinputt: \dom \myinputt \to \mathbb{R}^{m}$, and $(T, J) = \max \dom (\phi, \myinputt)$, the pair $(\phi', \myinputt')$ is the reversal of $(\phi, \myinputt)$, where $\phi': \dom \phi' \to \mathbb{R}^{n}$ with $\dom \phi' \subset \mathbb{R}_{\geq 0}\times \mathbb{N}$ and $\myinputt': \dom \myinputt' \to \mathbb{R}^{m}$ with $\dom \myinputt' = \dom \phi'$,  if:
	\begin{enumerate}[label = \arabic*)]
		\item  The function $\phi'$ is defined as
		\begin{enumerate}[label = \alph*)]
			\item $\dom \phi' = \{(T, J)\} - \dom \phi$, where the minus sign denotes \sj{the} Minkowski difference; 
			\item $\phi'(t, j) = \phi(T-t, J - j)$ for all $(t, j)\in \dom \phi'$.
		\end{enumerate}
		\item The function $\myinputt'$ is defined as
		\begin{enumerate}[label = \alph*)]
			\item $\dom \myinputt' = \{(T, J)\} - \dom \myinputt$, where the minus sign denotes Minkowski difference; 
			\item For all $j\in \mathbb{N}$ such that $I^{j} = \{t: (t, j)\in \dom \myinputt'\}$ has nonempty interior, 
			\begin{enumerate}[label = \roman*)]
				\item For all $t\in \interior I^{j}$,
				$
				\myinputt'(t, j) = \myinputt(T-t, J - j);
				$
				\item If $I^{0}$ has nonempty interior, then $\myinputt'(0, 0)\in \mathbb{R}^{m}$ is such that 
				$
				(\phi'(0, 0), \myinputt'(0, 0))\in \overline{\fw{C}};
				$
				\item For all $t\in \partial I^{j}$ such that $(t, j + 1)\notin \dom \myinputt'$ and $(t, j) \neq (0, 0)$,
				$
				\myinputt'(t, j)\in \mathbb{R}^{m}.
				$
			\end{enumerate}
			\item For all $(t, j)\in \dom \myinputt'$ such that $(t, j + 1)\in \dom \myinputt'$,
			$
			\myinputt'(t, j) = \myinputt(T- t, J - j -1).
			$
		\end{enumerate}
	\end{enumerate}
\end{definition}
Proposition \ref{theorem:reversedarcbwsystem} shows that the reversal of the solution pair to a hybrid system is a solution pair to its backward-in-time hybrid system. 
\begin{proposition}
	\label{theorem:reversedarcbwsystem}
	Given a hybrid system $\fw{\hs}$ and its backward-in-time system $\bw{\hs}$, the reversal $\psi' = (\phi', \myinputt')$ of $\psi = (\phi, \myinputt)$ is a compact solution pair to $\hs^{\text{bw}}$ if and only if $\psi = (\phi, \myinputt)$ is a compact solution pair to $\fw{\hs}$.
\end{proposition}
\begin{proof}
See Appendix \ref{appendix:proof_reverse}.
\end{proof}

Proposition \ref{theorem:concatenationsolution} and Proposition \ref{theorem:reversedarcbwsystem} validate the results of the concatenation and reversal operations, respectively. The following assumption \sj{on the solution pairs that are used to construct motion plans} integrates the conditions in Proposition \ref{theorem:concatenationsolution} and Proposition \ref{theorem:reversedarcbwsystem}.
\begin{assumption}
	\label{assumption:theoreticalcondition}
	Given a solution pair $\psi_{1} = (\phi_{1}, \myinputt_{1})$ to a hybrid system $\fw{\hs} = (\fw{C}, \fw{f}, \fw{D}, \fw{g})$ and a solution pair $\psi_{2} = (\phi_{2}, \myinputt_{2})$ to the backward-in-time hybrid system $\bw{\hs}$ associated to $\fw{\hs}$, the following hold:
	\begin{enumerate}[label = \arabic*)]
		\item $\psi_{1} $ and $\psi_{2} $ are compact;
		\item $\phi_{1}(T_{1}, J_{1}) = \phi_{2}(T_{2}, J_{2})$, where $(T_{1}, J_{1}) = \max \dom \psi_{1}$ and $(T_{2}, J_{2}) = \max \dom \psi_{2}$;
		\item If both $I_{\psi_{1}}^{J_{1}}$ and $I_{\psi_{2}}^{J_{2}}$ have nonempty interior, where $I^{j}_{\psi} =\{t: (t, j)\in \dom \psi\}$, $(T_{1}, J_{1}) = \max \dom \psi_{1}$, and $(T_{2}, J_{2}) = \max$ $\dom \psi_{2}$, then $\psi_{2}(T_{2}, J_{2}) \in C$.
	\end{enumerate}
\end{assumption}
\begin{remark}
		Given a hybrid system $\fw{\hs}$, its backward-in-time system $\bw{\hs}$, a solution pair $\psi_{1}$ to $\fw{\hs}$, and a solution pair $\psi_{2}$ to $\bw{\hs}$, Assumption \ref{assumption:theoreticalcondition} is imposed on $\psi_{1}$ and $\psi_{2}$ to guarantee that the concatenation of the reversal of $\psi_{2}$ to $\psi_{1}$ is a solution pair to $\fw{\hs}$. 
		Assumption \ref{assumption:theoreticalcondition} guarantees that the conditions needed to apply Proposition \ref{theorem:reversedarcbwsystem} and Proposition  \ref{theorem:concatenationsolution} hold. Note that conditions that guarantee the existence of nontrivial solutions have been proposed  in \cite[Proposition 3.4]{chai2018forward}. If $\xi \in X_{0}$ is such that
		$\xi \in D'$, where $D'$ is defined in (\ref{equation:Dprime}), or there exist $\epsilon > 0$, an absolutely continuous function $z:[0, \epsilon] \to \mathbb{R}^{n}$ with $z(0) = \xi$, and a Lebesgue measurable and locally essentially bounded function
		$\tilde{u} : [0, \epsilon] \to U_{C}$ such that $(z(t),\tilde{u}(t))\in \fw{C}$ for all $t\in (0,\epsilon)$ and $\frac{\text{d} z}{\text{d} t}(t) = \fw{f}(z(t), \tilde{u}(t))$ for almost all $t \in [0, \epsilon]$, where\endnote{Given a flow set $\fw{C}\subset \mathbb{R}^{n}\times \mathbb{R}^{m}$, the set-valued maps $\Psi^{u}_{c}: \mathbb{R}^{n} \to U_{C}$ is defined for each $x\in \mathbb{R}^{n}$ as $\Psi^{u}_{c}(x):=\{ u \in U_{C}: (x, u)\in \fw{C}\}$.} $\tilde{u}(t)\in \Psi^{u}_{c}(z(t))$ for every $t\in [0, \epsilon]$,
		then the existence of nontrivial solution pairs is guaranteed from $\xi$. 
		Items 2 and 3 in Assumption \ref{assumption:theoreticalcondition} relate the final states and their ``last" interval of flow of the given solution pairs.\EndRemark
\end{remark}

The following result validates that the \pn{hybrid signal} constructed \pn{using} the solution pairs satisfying Assumption~\ref{assumption:theoreticalcondition} is a solution pair to $\fw{\hs}$. 
\begin{lemma}
	\label{lemma:reverseconcatenation}
	Given a hybrid system $\fw{\hs}$ and its backward-in-time hybrid system $\bw{\hs}$, if $\psi_{1}$ is a solution pair to $\fw{\hs}$ and $\psi_{2}$ is a solution pair to $\bw{\hs}$ such that $\psi_{1}$ and $\psi_{2}$ satisfy Assumption \ref{assumption:theoreticalcondition},  then the concatenation $\psi = \psi_{1}|\psi_{2}'$ is a solution pair to $\fw{\hs}$, where $\psi'_{2}$ is the reversal of $\psi_{2}$.
\end{lemma}
\begin{proof}
	According to Definition \ref{definition:reversedhybrdarc}, the reversal operation can be performed on $\psi_{2}$ when $\psi_{2}$ is compact. The first condition stipulated in Assumption \ref{assumption:theoreticalcondition} ensures that $\psi_{2}$ is compact. As per Proposition \ref{theorem:reversedarcbwsystem}, the reversed trajectory $\psi'_{2}$ of $\psi_{2}$ constitutes a solution pair to $\fw{\hs}$.
	
	Next, we show that each condition in Proposition \ref{theorem:concatenationsolution} are satisfied.
	\begin{enumerate}[label = \arabic*)]
			\item In Assumption \ref{assumption:theoreticalcondition}, the first condition demonstrates the compactness of $\psi_{1}$. Consequently, this ensures the fulfillment of the first condition in Proposition \ref{theorem:concatenationsolution}.
			\item It is noted that the second condition of Assumption \ref{assumption:theoreticalcondition} implies $\phi_{1}(T_{1}, J_{1}) = \phi_{2}(T_{2}, J_{2})$. Given that $\phi'_{2}(0, 0) =  \phi_{2}(T_{2}, J_{2})$,  it follows that the equality $\phi_{1}(T_{1}, J_{1}) = \phi'_{2}(0, 0)$ holds. Hence, this establishes that $\psi_{1}$ and $\psi'_{2}$ meet the second condition specified in Proposition  \ref{theorem:concatenationsolution}.
			\item Given that $I_{\psi_{2}}^{J_{2}} = I_{\psi'_{2}}^{0}$ and $\psi_{2}(T_{2}, J_{2}) = \psi'_{2}(0, 0)$, the third condition in Assumption \ref{assumption:theoreticalcondition}  indicates that $\psi_{1}$ and $\psi'_{2}$ meet the requirements of the third condition in Proposition \ref{theorem:concatenationsolution}.
	\end{enumerate}
	Given that the conditions in Proposition \ref{theorem:concatenationsolution} are met, Proposition \ref{theorem:concatenationsolution} ensures that the concatenation of $\psi'_{2}$ to $\psi_{1}$ is a solution pair to ~$\fw{\hs}$.
\end{proof}
Lemma \ref{lemma:reverseconcatenation} is exploited by our forthcoming HyRRT-Connect algorithm when detecting overlaps between $\fw{\stree}$ and $\bw{\stree}$.
\begin{example}[(Actuated bouncing ball system in Example~\ref{example:bouncingball}, revisited)]
	Consider an example solution pair $\psi_{1}$ to the actuated bouncing ball system $\fw{\hs}$ as defined in Example \ref{example:bouncingball} and depicted in Figure \ref{fig:fwsol}, along with an example solution pair $\psi_{2}$ to $\bw{\hs}$ as defined in Example \ref{example:bouncingballbw} and shown in Figure \ref{fig:bwsol}. The reversal of $\mystatet_{2}$, denoted $\mystatet'_{2}$, as defined in Definition~\ref{definition:reversedhybrdarc}, is computed and illustrated in Figure~\ref{fig:revsol}. The concatenation of $\mystatet'_{2}$ to $\mystatet_{1}$ is computed and illustrated in Figure~\ref{fig:recsol}.
	
	Next, we justify that $\mystatet_{1}$ and $\mystatet_{2}$ satisfy Assumption~\ref{assumption:theoreticalcondition}. The domains of $\mystatet_{1}$ and $\mystatet_{2}$ are both compact, as illustrated in Figure \ref{fig:fwsol} and Figure \ref{fig:bwsol}, confirming that the first item in Assumption \ref{assumption:theoreticalcondition} is met. As depicted in Figure \ref{fig:recsol}, we find that $\mystatet_{1}(T_{1}, J_{1}) = \mystatet_{2}(T_{2}, J_{2})$ where $(T_{1}, J_{1}) = \max \dom \mystatet_{1} \approx (0.33, 0)$ and $(T_{2}, J_{2}) = \max \dom \mystatet_{2} \approx (2.78, 1)$. This satisfies the second item in Assumption \ref{assumption:theoreticalcondition}. Additionally, since $I_{\psi_{1}}^{J_{1}} = I_{\psi_{1}}^{0} = [0, 0.33]$ and $I_{\psi_{2}}^{J_{2}} = I_{\psi_{2}}^{1} = [1.42, 2.78]$ both have nonempty interiors, and $\mystatet_{1}(T_{1}, J_{1}) = (13.45, -3.36) \approx \mystatet_{2}(T_{2}, J_{2})$, it follows that $\psi(T_{1}, J_{1})\in C$. Thus, the last item in Assumption \ref{assumption:theoreticalcondition} is satisfied.
	\begin{figure}[htbp]
		\centering
		\subfigure[The state trajectory of $\psi_{1}$.\label{fig:fwsol}]{\includegraphics[width = 0.48 \columnwidth]{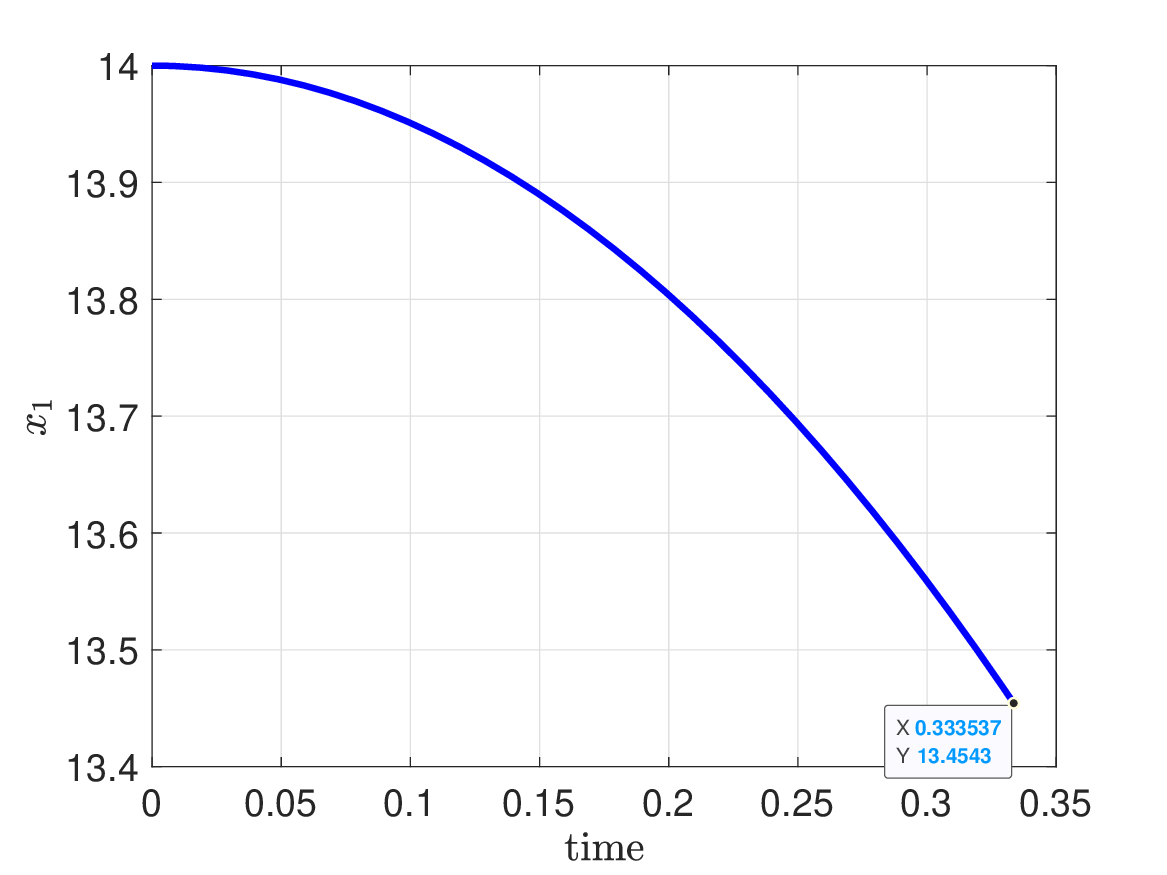}}
		\subfigure[The state trajectory of $\psi_{2}$.\label{fig:bwsol}]{\includegraphics[width = 0.48 \columnwidth]{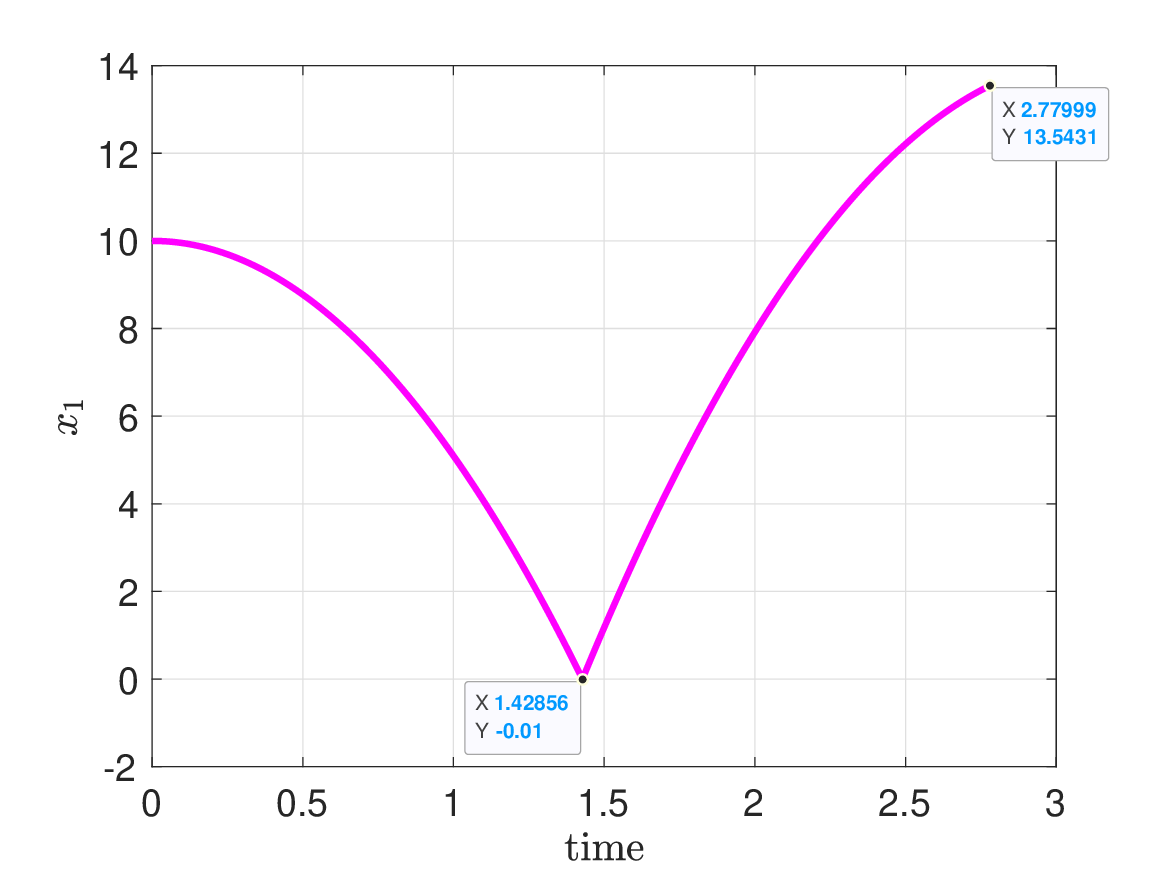}}
		\subfigure[The reversal of the state trajectory of $\psi_{2}$.\label{fig:revsol}]{\includegraphics[width = 0.48 \columnwidth]{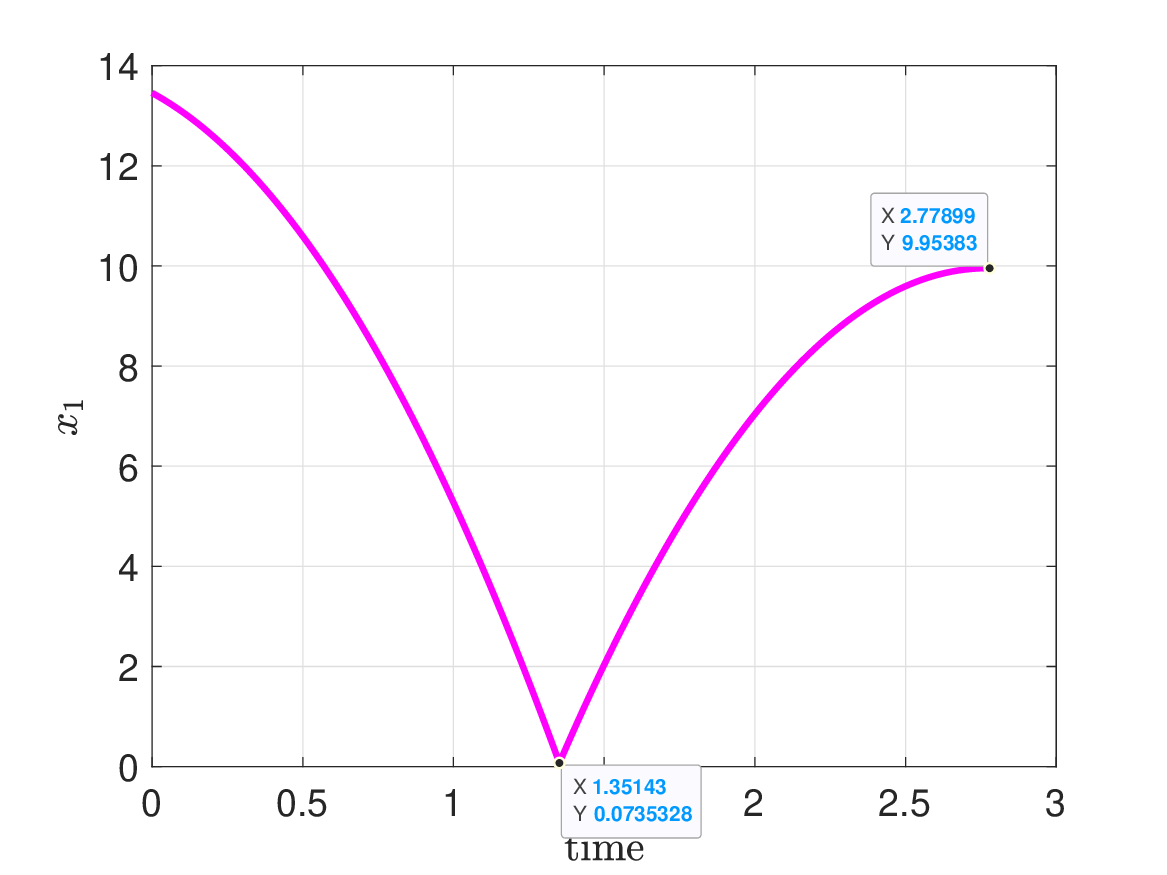}}
		\subfigure[The concatenation of the reversal of  $\psi_{2}$ to $\psi_{1}$.\label{fig:recsol}]{\includegraphics[width = 0.48 \columnwidth]{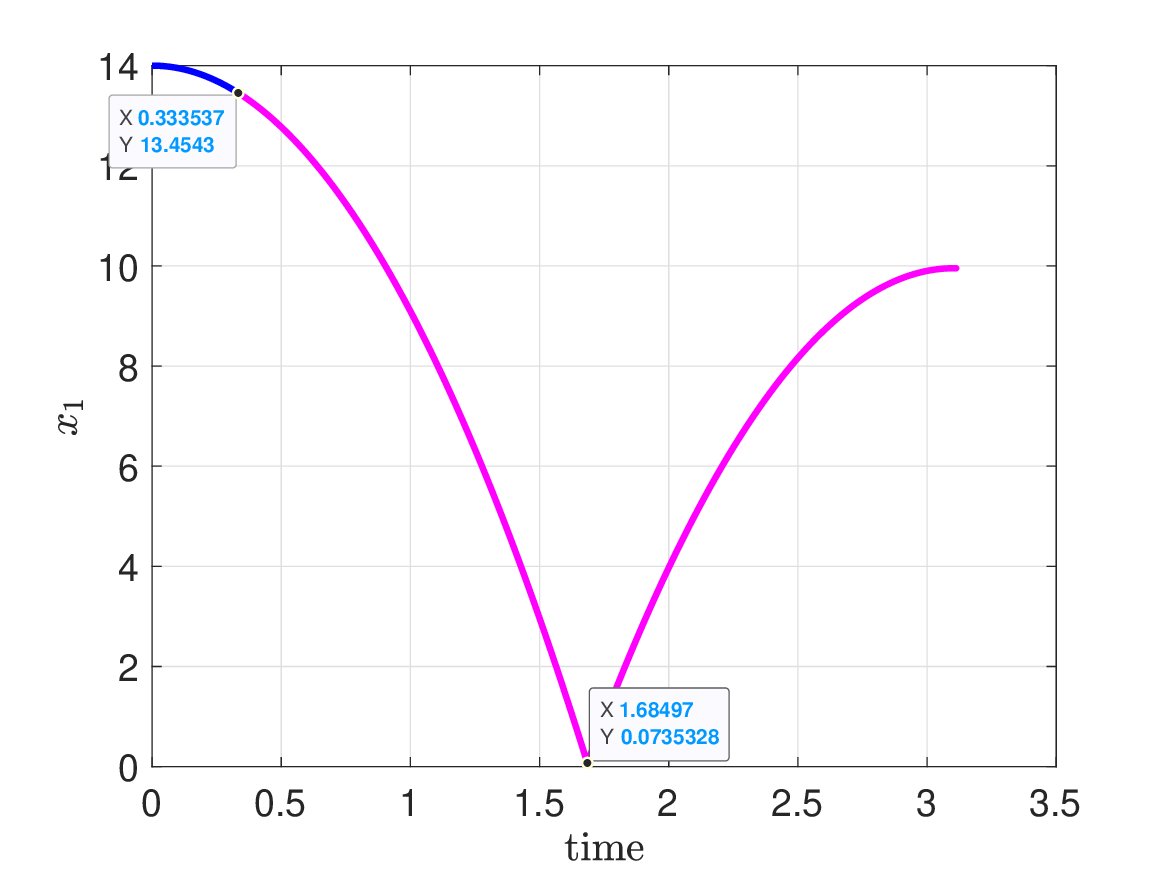}}
		\caption{Examples demonstrating the trajectories of reversal and concatenation operations.}
	\end{figure}
\end{example}
	\ifbool{conf}{\vspace{-0.4cm}}{}
\subsection{HyRRT-Connect Algorithm}
	\ifbool{conf}{\vspace{-0.3cm}}{}
The proposed algorithm is given in Algorithm~\ref{algo:bihyrrt}. The inputs of Algorithm~\ref{algo:bihyrrt} are the problem $\mathcal{P} = (X_{0}, X_{f}, X_{u}, (\fw{C}, \fw{f}, \fw{D}, \fw{g}))$, the backward-in-time hybrid system $\bw{\hs}$ obtained from (\ref{model:generalhybridbackwardsystem}), the input libraries $\fw{\il}$ and  $\bw{\il}$, two parameters $\fw{p}_{n}\in (0, 1)$ and $\bw{p}_{n}\in (0, 1)$, which tune the probability of proceeding with the flow regime or the jump regime during the forward and backward construction, respectively, an upper bound $K\in \mathbb{N}_{>0}$ for the number of iterations to execute, and four tunable sets acting as constraints in finding a closest vertex to $x_{rand}$: $\fw{X}_{c}\supset \overline{C^{\text{fw}'}}$, $\fw{X}_{d}\supset \overline{D^{\text{fw}'}}$, $\bw{X}_{c}\supset \overline{C^{\text{bw}'}}$ and $\bw{X}_{d}\supset \overline{D^{\text{bw}'}}$\pnc{,} where $C^{\text{fw}'}$, $C^{\text{bw}'}$, $D^{\text{fw}'}$ and $D^{\text{bw}'}$ are defined as in  (\ref{equation:Cprime}) and (\ref{equation:Dprime}). \textbf{Step} 1 in Section \ref{section:algorithmoverview} corresponds to the function calls $\fw{\stree}.\texttt{init}$ and~$\bw{\stree}.\texttt{init}$ in line 1 of Algorithm \ref{algo:bihyrrt}. The construction of $\fw{\stree}$ in \textbf{Step}~2 is implemented in lines 3 - 10. The construction of $\bw{\stree}$ in \textbf{Step}~2 is implemented in lines 11 - 18. The solution checking in \textbf{Step} 3 is executed depending on the return of the function call $\texttt{extend}$ and will be further discussed in\endnote{The solution checking process in \textbf{Step} 3 is not reflected in Algorithm \ref{algo:bihyrrt}. This omission is intentional, as practitioners might devise varied termination conditions for the HyRRT-Connect algorithm. } Section \ref{section:solutionchecking}. Each function in Algorithm \ref{algo:bihyrrt} is defined next.
\begin{algorithm}[htbp]
	\caption{HyRRT-Connect algorithm}
	\label{algo:bihyrrt}
	\hspace*{\algorithmicindent} \textbf{Input}: $X_{0}, X_{f}, X_{u}, \fw{\hs} = (\fw{C}, \fw{f}, \fw{D}, \fw{g}),\newline \bw{\hs} = (\bw{C}, \bw{f}, \bw{D}, \bw{g}), (\fw{\mathcal{U}}, \bw{\mathcal{U}}), \fw{p}_{n},  \bw{p}_{n} \in (0, 1)$, $K\in \mathbb{N}_{>0}$, $\fw{X}_{c}\supset \overline{C^{\text{fw}'}}$, $\fw{X}_{d}\supset \overline{D^{\text{fw}'}}$, $\bw{X}_{c}\supset \overline{C^{\text{bw}'}}$ and $\bw{X}_{d}\supset \overline{D^{\text{bw}'}}$.\\
	\hspace*{\algorithmicindent} \textbf{Output:} motion plan $\psi$ or \texttt{Failure} signal.
	\begin{algorithmic}[1]
		\State $\fw{\stree}.\texttt{init}(X_{0})$, $\bw{\stree}.\texttt{init}(X_{f})$
		\For{$k = 1$ to $K$}
		\State randomly select a real number $\fw{r}$ from $[0, 1]$.
		\If{$\fw{r}\leq \fw{p}_{n}$}
		\State $\fw{x}_{rand}\leftarrow \texttt{random\_state}(\overline{C^{\text{fw}'}})$.
		\State $\texttt{is\_extended\_fw}\leftarrow\texttt{extend}(\fw{\stree}, \fw{x}_{rand},\newline (\fw{\il}_{C}, \fw{\il}_{D}), \fw{\hs}, X_{u}, \fw{X}_{c})$.
		\Else
		\State $\fw{x}_{rand}\leftarrow \texttt{random\_state}(D^{\text{fw}'})$.
		\State $\texttt{is\_extended\_fw}\leftarrow\texttt{extend}(\fw{\stree}, \fw{x}_{rand},\newline (\fw{\il}_{C}, \fw{\il}_{D}), \fw{\hs}, X_{u}, \fw{X}_{d})$.
		\EndIf
		\State \hspace{-0.3cm}randomly select a real number $\bw{r}$ from $[0, 1]$.
		\If{$\bw{r}\leq \bw{p}_{n}$}
		\State $\bw{x}_{rand}\leftarrow \texttt{random\_state}(\overline{C^{\text{bw}'}})$.
		\State $\texttt{is\_extended\_bw}\leftarrow\texttt{extend}(\bw{\stree}, \bw{x}_{rand},\newline (\bw{\il}_{C}, \bw{\il}_{D}), \bw{\hs}, X_{u}, \bw{X}_{c})$.
		\Else
		\State $\bw{x}_{rand}\leftarrow \texttt{random\_state}(D^{\text{bw}'})$.
		\State $\texttt{is\_extended\_bw}\leftarrow\texttt{extend}(\bw{\stree}, \bw{x}_{rand},\newline (\bw{\il}_{C}, \bw{\il}_{D}), \bw{\hs}, X_{u}, \bw{X}_{d})$.
		\EndIf
		\If{$\texttt{is\_extended\_fw} = \texttt{Advanced}$ or $\texttt{is\_extended\_bw} = \texttt{Advanced}$}
			\State $(\texttt{is\_a\_motion\_found}, \psi)\leftarrow\texttt{\hspace{2cm}solution\_check}(\mathcal{T}, X_{0}, X_{f}, C)$.
			\If{\texttt{is\_a\_motion\_found} = \texttt{true}}
			\State \Return $\psi$.
			\EndIf
		\EndIf
		\EndFor
		\State \Return \texttt{Failure}.
	\end{algorithmic}
\end{algorithm}
\begin{algorithm}[htbp]
	\caption{Extend function}
	\label{algo:extend}
	\hspace*{\algorithmicindent} \textbf{Input: $\mathcal{T}, x, (\mathcal{U}_{C}, \mathcal{U}_{D}), \mathcal{H}, X_{u}, X_{*}$}\\
		\hspace*{\algorithmicindent} \textbf{Output:} $\texttt{Advanced}$ or $\texttt{Trapped}$ signal.
	\begin{algorithmic}[1]
		\Function{extend}{$(\mathcal{T}, x, (\mathcal{U}_{C}, \mathcal{U}_{D}), \mathcal{H}, X_{u}, X_{*})$}
		\State $v_{cur}\leftarrow \texttt{nearest\_neighbor}(x, \mathcal{T}, \mathcal{H}, X_{*})$;
		\State $(\texttt{is\_a\_new\_vertex\_generated}, x_{new}, \psi_{new} ) \leftarrow \texttt{new\_state}(v_{cur}, (\mathcal{U}_{C}, \mathcal{U}_{D}), \mathcal{H}, X_{u})$
		\If {$\texttt{is\_a\_new\_vertex\_generated} = \texttt{true}$}
		\State $v_{new} \leftarrow \mathcal{T}.\texttt{add\_vertex}(x_{new})$;
		\State $\mathcal{T}.\texttt{add\_edge}(v_{cur}, v_{new}, \psi_{new})$;
		\State \Return $\texttt{Advanced}$;
		\EndIf
		\State \Return $\texttt{Trapped}$;
		\EndFunction
	\end{algorithmic}
\end{algorithm}
\subsubsection{$\mathcal{T}.\texttt{init}(X)$:} 
	The $\mathcal{T}.\texttt{init}$ function initializes the search tree $\mathcal{T} = (V, E)$ by randomly selecting points from set $X$. The set $X$ can either be $X_{0}$ or $X_{f}$.  For each sampling point $x_{0}\in X$, a corresponding vertex $v_{0}$ is added to $V$, while no edges are added to $E$ at this stage.
\subsubsection{$x_{rand}$$\leftarrow$$\texttt{random\_state}(S)$:} 
	The function call $\texttt{random\_state}$ randomly selects a point from the set $S\subset \mathbb{R}^{n}$. Rather than selecting a point from $\overline{C'}\cup D'$, $\texttt{random\_state}$ is designed to select points from $\overline{C'}$ and $D'$ separately depending on the value of $r$ . The reason is that if $\overline{C'}$ (or $D'$) has zero measure while $D'$ (respectively, $\overline{C'}$) does not, the probability that the point selected from $\overline{C'}\cup D'$ lies in $\overline{C'}$ (respectively, $D'$) is zero, which would prevent establishing probabilistic completeness discussed in Section \ref{section:parallelcomputation}.
\subsubsection{$v_{cur}$$\leftarrow$$\texttt{nearest\_neighbor}$$(x_{rand}, $$\mathcal{T}, $$\mathcal{H}, X^{\Delta}_{*})$:} 
The function call $\texttt{nearest\_neighbor}$ searches for a vertex $v_{cur}$ in the search tree $\mathcal{T} = (V, E)$ such that its associated state value has minimal distance to $x_{rand}$. This function is implemented by solving the following optimization problem over $X^{\Delta}_{\star}$, where $\star$ is either $c$ or $d$ and $\Delta$ is either $\texttt{fw}$ or $\texttt{bw}$.
\begin{problem}
	\label{problem:nearestneighbor}
	Given a hybrid system $\mathcal{H} = (C, f, D, g)$, $x_{rand}\in \myreals[n]$, and a search tree $\mathcal{T} = (V, E)$, solve
	$$
	\begin{aligned}
	\argmin_{v\in V}& \quad |\overline{x}_{v} -  x_{rand}|\\
	\textrm{s.t.}& \quad\overline{x}_{v} \in X^{\Delta}_{\star}.
	\end{aligned}
	$$
\end{problem}
The data of Problem \ref{problem:nearestneighbor} comes from the arguments of the $\texttt{nearest\_neighbor}$ function call. This optimization problem can be solved by traversing all the vertices in $V$.
\subsubsection{$(\texttt{is\_a\_new\_vertex\_generated}, x_{new}, \psi_{new} ) \leftarrow \texttt{new\_state}(v_{cur}, (\mathcal{U}_{C}, \mathcal{U}_{D}), \mathcal{H} = (C, f, D, g) , X_{u})$:} 	\label{section:newstate}
If $\overline{x}_{v_{cur}}\in \overline{C'}\backslash D'$ ($\overline{x}_{v_{cur}}$$\in$$D'\backslash \overline{C'}$), the function call $\texttt{new\_state}$ generates a new solution pair $\psi_{new}$ to the hybrid system $\mathcal{H}$ starting from $\overline{x}_{v_{cur}}$ by applying an input signal $\tilde{u}$ (respectively, an input value $u_{D}$) randomly selected from $\mathcal{U}_{C}$ (respectively, $\mathcal{U}_{D}$). 
If $\overline{x}_{v_{cur}}$$\in\overline{C'}\cap D'$, then this function generates $\psi_{new}$ by randomly selecting flows or jump. The final state of $\psi_{new}$ is denoted as $x_{new}$.
Note that the choices of inputs are random. Some RRT variants choose the optimal input that drives $x_{new}$ closest to $x_{rand}$. However, \cite{kunz2015kinodynamic} proves that such a choice makes the RRT algorithm probabilistically incomplete. 

After $\psi_{new}$ and $x_{new}$ are generated, the function $\texttt{new\_state}$ checks if there exists $(t, j)\in \dom \psi_{new}$ such that $\psi_{new}(t, j)\in X_{u}$. In such a case, $\psi_{new}$ intersects with the unsafe set and $\texttt{new\_state}$ returns $\texttt{false}$. Otherwise, this function returns $\texttt{true}$.
After $\psi_{new}$ and $x_{new}$ are generated, the function $\texttt{new\_state}$ evaluates if $\psi_{new}$ is trivial. If $\psi_{new}$ is trivial, addition of $\psi_{new}$ to $\mathcal{T}$ will not generate additional solutions already included in $\mathcal{T}$ and is, hence, unnecessary, setting  $\texttt{is\_a\_new\_vertex\_generated}$ to $\texttt{false}$. The function call $\texttt{new\_state}$ then returns. If $\psi_{new}$ is not trivial, the function $\texttt{new\_state}$ finds $(t, j)\in \dom \psi_{new}$ such that $\psi_{new}(t, j)\in X_{u}$, implying intersection between $\psi_{new}$ and the unsafe set $X_{u}$, in which case it sets $\texttt{is\_a\_new\_vertex\_generated}$ to $\texttt{false}$. If neither condition is met, $\texttt{is\_a\_new\_vertex\_generated}$ is set to $\texttt{true}$.
\subsubsection{$v_{new}\leftarrow\mathcal{T}.\texttt{add\_vertex}(x_{new})$ and  $\mathcal{T}.\texttt{add\_edge}(v_{cur}, v_{new}, \psi_{new})$:} 
The function call $\mathcal{T}.\texttt{add\_vertex}(x_{new})$ adds a new vertex $v_{new}$ associated with $x_{new}$ to $\mathcal{T}$ and returns $v_{new}$. The function call $\mathcal{T}.\texttt{add\_edge}(v_{cur}, v_{new}, \psi_{new})$ adds a new edge $e_{new} = (v_{cur}, v_{new})$ associated with $\psi_{new}$  to $\mathcal{T}$. 
\section{Motion Plan Check and Reconstruction}\label{section:solutionchecking}
In Algorithm \ref{algo:bihyrrt}, when a call to the function $\texttt{extend}$ using $\fw{\hs}$ (Line 6 and Line 9) or $\bw{\hs}$ (Line 14 and Line 17) returns $\texttt{advanced}$, indicating that at least one new vertex has been added to the search tree, a solution-checking process (the function call $\texttt{solution\_check}$ in Line 20) is executed to determine whether this addition results in a valid motion plan. After introducing the function calls used to construct the search trees in Section~\ref{section:hyrrtconnect}, we discuss the function call $\texttt{solution\_check}$ in this section.

The following two scenarios are identified for which a motion plan can be constructed by utilizing one path from $\fw{\stree}$ and another one from $\bw{\stree}$:
\begin{enumerate}[label=S\arabic*)]
	\item A vertex in $\fw{\stree}$ is associated with the same state in the flow set as some vertex in $\bw{\stree}$.
	\item A vertex in $\fw{\stree}$ is associated with a state such that a forward-in-hybrid time jump from such state results in the state associated with some vertex in $\bw{\stree}$, or conversely, a vertex in $\bw{\stree}$ is associated with a state such that a backward-in-hybrid time jump from such state results in the state associated with some vertex in $\fw{\stree}$.
\end{enumerate}
These scenarios provide solution pairs that can be used to construct a motion plan by reversing a solution pair in $\bw{\stree}$ and concatenating its reversal to a solution pair in $\fw{\stree}$, as guaranteed by Lemma \ref{lemma:reverseconcatenation}. Each of these scenarios is evaluated by the function call $\texttt{solution\_check}$ whenever an $\texttt{Advanced}$ signal is returned by the $\texttt{extend}$ function. However, the random selection of the \pn{inputs} prevents the exact satisfaction of S1. This may lead to a discontinuity along the flow in the resulting motion plan. To address this issue, a reconstruction process is introduced below. This process propagates forward in hybrid time from the state associated with the vertex identified in $\fw{\stree}$. In addition, neglecting approximation errors due to numerical computation, it is typically possible to solve for an exact input at a jump from one state to an other, as required in S2. Next, we present an implementation for identifying two paths from $\fw{\stree}$ and $\bw{\stree}$, respectively, such that the solution pairs associated with these paths satisfy either S1 or S2, thereby leading to a motion plan for the given motion planning problem.
\subsection{Same State Associated with Vertices in $\fw{\stree}$ and $\bw{\stree}$ in~S1}
\label{section:checksolution_stateequal}
The function call $\texttt{solution\_check}$ first checks if there exists a path
\begin{equation}\label{equation:fwpath}
\begin{aligned}\fw{p} &:= ((\fw{v}_{0}, \fw{v}_{1}), (\fw{v}_{1}, \fw{v}_{2}), ..., (\fw{v}_{m - 1}, \fw{v}_{m})) \\
&=: (\fw{e}_{0}, \fw{e}_{1}, ..., \fw{e}_{m - 1})
\end{aligned}
\end{equation}
in $\fw{\stree}$ (see (\ref{eq:path}) for a definition), where $m\in\nnumbers$, and a path 
\begin{equation}\label{equation:bwpath}
\begin{aligned}
\bw{p} &:= ((\bw{v}_{0}, \bw{v}_{1}), (\bw{v}_{1}, \bw{v}_{2}), ..., (\bw{v}_{n - 1}, \bw{v}_{n}))\\
& = : (\bw{e}_{0}, \bw{e}_{1}, ..., \bw{e}_{n - 1})
\end{aligned}
\end{equation}
in $\bw{\stree}$, where $n\in\nnumbers$, satisfying the following conditions:
\begin{enumerate}[label=C\arabic*)]
	\item $\overline{x}_{\fw{v}_{0}} \in X_{0}$\pn{, where $\fw{v}_{0}$ is the first vertex in $\fw{p}$,}
	\item for \pn{each} $i \in \{0, 1, ..., m - 2\}$, if $\overline{\psi}_{\fw{e}_{i}}$ and $\overline{\psi}_{\fw{e}_{i + 1}}$ are both purely continuous, then $\overline{\psi}_{\fw{e}_{i + 1}}(0, 0)\in \fw{C}$,
	\item $\overline{x}_{\bw{v}_{0}} \in X_{f}$\pn{, where $\bw{v}_{0}$ is the first vertex in $\bw{p}$,}
	\item for \pn{each} $i \in \{0, 1, ..., n - 2\}$, if $\overline{\psi}_{\bw{e}_{i}}$ and $\overline{\psi}_{\bw{e}_{i + 1}}$ are both purely continuous, then $\overline{\psi}_{\bw{e}_{i + 1}}(0, 0)\in \bw{C}$,
	\item $\overline{x}_{\fw{v}_{m}} = \overline{x}_{\bw{v}_{n}}$, \pn{where $\fw{v}_{m}$ is the last vertex in $\fw{p}$ and $\bw{v}_{n}$ is the last vertex in $\bw{p}$,}
	\item if $\overline{\psi}_{\fw{e}_{m - 1}}$ and $\overline{\psi}_{\bw{e}_{n - 1}}$ are both purely continuous, 
	then $\overline{\psi}_{\bw{e}_{n - 1}}(\bw{T}, 0) \in \fw{C}$, where $(\bw{T}, 0) = \max\dom \overline{\psi}_{\bw{e}_{n - 1}}$.
\end{enumerate}
Conditions C1-C6 are derived directly from the problem setting in Problem~\ref{problem:motionplanning}: C1 ensures that the path starts from $X_{0}$ (item 1 in Problem \ref{problem:motionplanning}); C2, C4, C5, and C6 collectively guarantee that the solution pair associated with the path satisfies the given hybrid dynamics (item 2 in Problem~\ref{problem:motionplanning}); C3 ensures that the path ends at $X_{f}$ (item 3 in Problem~\ref{problem:motionplanning}). If $\texttt{solution\_check}$ is able to find a path $\fw{p}$ in $\fw{\stree}$ and a path $\bw{p}$ in $\bw{\stree}$ satisfying C1-C6, then a motion plan to $\mathcal{P}$ can be constructed as $\fw{\psi}|\bwr{\psi}$, where, for simplicity, $\fw{\psi} = (\fw{\mystatet}, \fw{\myinputt}) =: \tilde{\psi}_{\fw{p}}$ denotes the solution pair associated with the path $\fw{p}$ in (\ref{equation:fwpath}) and is referred to as a \emph{forward partial motion plan}. Similarly, $\bw{\psi} = (\bw{\mystatet}, \bw{\myinputt}) =: \tilde{\psi}_{\bw{p}}$ denotes the solution pair associated with the path $\bw{p}$ in (\ref{equation:bwpath})  and is referred to as a \emph{backward partial motion plan}.  With $\bwr{\psi}$ denoting the reversal of $\bw{\psi}$, the result $\fw{\psi}|\bwr{\psi}$ is guaranteed to satisfy each item in Problem \ref{problem:motionplanning}, as follows:
\begin{enumerate}[label=\arabic*)]
	\item By C1, it follows that $\fw{\psi}|\bwr{\psi}$ starts from $X_{0}$. Namely, item 1 in Problem \ref{problem:motionplanning} is satisfied.
	\item Due to C2 (respectively, C4), by iterative applying Proposition \ref{theorem:concatenationsolution} to each pair of $\overline{\psi}_{\fw{e}_{i}}$ and $\overline{\psi}_{\fw{e}_{i + 1}}$ (respectively, $\overline{\psi}_{\bw{e}_{i}}$ and $\overline{\psi}_{\bw{e}_{i + 1}}$) where $i\in \{0, 1, ..., m - 2\}$ (respectively, $i\in \{0, 1, ..., n - 2\}$), it follows that $\fw{\psi}$ (respectively, $\bw{\psi}$) is a solution pair to $\fw{\hs}$ (respectively, $\bw{\hs}$). Furthermore, given C5 and C6, Lemma~\ref{lemma:reverseconcatenation} establishes that $\fw{\psi}|\bwr{\psi}$ is a solution pair to $\fw{\hs}$.

	\item C3 ensures that $\fw{\psi}|\bwr{\psi}$ ends within $X_{f}$. This confirms the satisfaction of item 3 in Problem \ref{problem:motionplanning}.
	\item For any edge $e\in\fw{p}\cup\bw{p}$,  the trajectory $\overline{\psi}_{e}$ avoids intersecting the unsafe set as a result of the exclusion of solution pairs that intersect the unsafe set in the function call $\texttt{new\_state}$. Therefore, item 4 in Problem~\ref{problem:motionplanning} is satisfied.
\end{enumerate}
Since each requirement in Problem \ref{problem:motionplanning} is satisfied, it is established that $\fw{\psi}|\bwr{\psi}$ is a motion plan to $\mathcal{P}$. 

	In practice, guaranteeing C5 above is not possible due to numerical error. Thus, given $\delta > 0$ representing the tolerance associated with C5, the function call $\texttt{solution\_check}$ implements C5 as
	\begin{equation}
	\label{equation:tolerance}
	|\overline{x}_{\fw{v}_{m}} - \overline{x}_{\bw{v}_{n}}|\leq	\delta\sj{,}
	\end{equation} leading to a potential discontinuity during the flow.
	\subsection{Reconstruction Process}
	To smoothen the discontinuity associated with~(\ref{equation:tolerance}), we propose a reconstruction process that
		given the hybrid input $\bw{\myinputt}$ of $\bw{\psi}$ identified in S1, requires computing the \emph{maximal} solution pair to an auxiliary system, denoted $(\simu{\mystatet}, \simu{\myinputt})$, as follows. Using the initial condition and hybrid input
\begin{enumerate}[label= R\arabic*)]
				\item $\simu{\mystatet}(0, 0) = \fw{\mystatet}(\fw{T}, \fw{J})$, where $\fw{\mystatet}$ is the state trajectory of $\fw{\psi}$ identified in S1 and $(\fw{T}, \fw{J}) = \max\dom \fw{\mystatet}$,
				\item $\simu{\myinputt}(t, j) = \bwr{\myinputt}(t, j)$ for each $(t, j)\in \dom \simu{\myinputt}$; see item 2 in Definition~\ref{definition:reversedhybrdarc} for the reversal of a hybrid input.
			\end{enumerate}
We generate $\simu{\mystatet}$ via the following auxiliary system, denoted $\recon{\hs}$:
\begin{equation}
\recon{\hs}: \left\{              
\begin{aligned}               
\dot{x} & = \recon{f}(x, \bwr{\myinputt}(t, j))     &(t, j)\in \recon{C}\\                
x^{+} & =  \recon{g}(x, \bwr{\myinputt}(t, j))      &(t, j)\in \recon{D}\\                
\end{aligned}   \right. 
\label{model:reconhybridsystem}
\end{equation}where$x\in \myreals[n]$ is the state,
\begin{enumerate}
\item $\recon{D} := \{(t, j)\in \dom \bwr{\myinputt}: (t, j + 1)\in  \dom \bwr{\myinputt}\}$;
\item $\recon{C} := \overline{\dom \bwr{\myinputt}\backslash \recon{D}}$;
\item $\recon{g}(x, u) := g(x, u)$ for all\endnote{Note that the flow map $f$ and the jump map $g$ in (\ref{model:generalhybridsystem}) are defined on $\mathbb{R}^{n}\times\mathbb{R}^{m}$.} $(x, u)\in \mathbb{R}^{n}\times\mathbb{R}^{m}$;
\item $\recon{f}(x, u) := f(x, u)$ for all $(x, u)\in \mathbb{R}^{n}\times\mathbb{R}^{m}$.
\end{enumerate}
	\begin{remark}
			The definitions of $\recon{C}$ and $\recon{D}$ indicate that $\simu{\mystatet}$ follows the flow and jump of $\bwr{\myinputt}$. Condition R1 ensures that the reconstructed motion plan begins at the final state of the forward partial motion plan, effectively eliminating any discontinuity. Given that $\simu{\mystatet}$ is maximal, it follows that
			\begin{equation}\label{equation:equalhtd}
				\dom \simu{\mystatet} = \dom \bwr{\myinputt} = \dom \bwr{\mystatet}.
				\end{equation}
	\end{remark}

	\subsubsection{Convergence of $\simu{\mystatet}$ to $X_{f}$} 
	We first establish a dependency between $|\simu{\mystatet}(\simu{T}, \simu{J}) - \bw{\mystatet}(0, 0)|$ and the tolerance $\delta$ in (\ref{equation:tolerance}), where $(\simu{T}, \simu{J}) = \max\dom \simu{\mystatet}$.
		The following assumption is imposed on the flow map $f$ of the hybrid system $\mathcal{H}$ in (\ref{model:generalhybridsystem}). Assuming that the flow map $f$ is Lipschitz continuous ensures that compact solutions governed by the differential equation $\dot{x} = f(x, u)$, when starting close to each other, have their endpoint distances bounded exponentially in terms of their initial deviation, with a growth rate determined by the Lipschitz constant.
		\begin{assumption}
			\label{assumption:flowlipschitz}
			The flow map $f$ is Lipschitz continuous. In particular, there exist $K^{f}_{x}, K^{f}_{u}\in \mathbb{R}_{>0}$ such that, for each $(x_{0}, x_{1}, u_{0}, u_{1}),$ such that $(x_{0}, u_{0}) \in C$, $(x_{1}, u_{0}) \in C$, and $(x_{0}, u_{1}) \in C$,
				$$
					\begin{aligned}
					|f(x_{0}, u_{0}) - f(x_{1}, u_{0})|&\leq K^{f}_{x}|x_{0} - x_{1}|\\
					|f(x_{0}, u_{0}) - f(x_{0}, u_{1})|&\leq K^{f}_{u}|u_{0} - u_{1}|.
					\end{aligned}
				$$
		\end{assumption}
		The following lemma is a slight modification of \cite[Lemma 2]{kleinbort2018probabilistic}. It ensures that, for the continuous pieces of a motion plan, the reconstructed solution pair remains close to the backward partial motion plan.
			\begin{lemma}\label{lemma:flowbound}
				Given a hybrid system $\hs$ satisfying Assumption \ref{assumption:flowlipschitz}, if there exists a purely continuous solution pair $\bw{\psi} = (\bw{\mystatet}, \bw{\myinputt})$ to $\bw{\hs}$, then the reconstructed solution $\simu{\mystatet}$, which is a solution to $\recon{\hs}$ satisfying R1 and R2 with $\bwr{\myinputt}$ being the reversal of $\bw{\myinputt}$, satisfies the following properties: 
				\begin{enumerate}[label= P\arabic*)]
					\item $\dom \simu{\mystatet} = \dom \bwr{\myinputt}$; thus, $\simu{\mystatet}$ is purely continuous;
					\item if $|\simu{\mystatet}(0, 0) - \bwr{\mystatet}(0, 0)| \leq \delta$, then $|\simu{\mystatet}(T, 0) - \bwr{\mystatet}(T, 0)| \leq \exp{(K^{f}_{x}T)}\delta$, where $(T, 0) = \pn{\max \dom \simu{\mystatet} = \max \dom \bwr{\mystatet}}$, and $\max \dom \simu{\mystatet} = \max \dom \bwr{\mystatet}$ is valid because $\dom \simu{\mystatet} = \dom \bwr{\mystatet}$ in P1.
				\end{enumerate}
			\end{lemma}
			\begin{proof}
				Given that $\bw{\psi}$ is purely continuous, $\bwr{\psi}$ is purely continuous and, by (\ref{equation:equalhtd}), P1 is established immediately. By the Lipschitz continuity of the flow map $f$ in Assumption \ref{assumption:flowlipschitz}, since $\simu{\mystatet}$ is generated by (\ref{model:reconhybridsystem}), \cite[Lemma 2]{kleinbort2018probabilistic} (also presented in the appendix as Lemma \ref{lemma:2in6}), along with R1 and R2, establishes that $|\simu{\mystatet}(T, 0) - \mystatet(T, 0)| \leq \exp{(K^{f}_{x}T)}\delta$.
		\end{proof}
	
		The following assumption is imposed on the jump map $g$ of the hybrid system $\hs$ in (\ref{model:generalhybridsystem}). Assuming that the jump map $g$ is bounded in the linear manner ensures that any solutions governed by the difference equation $x^{+} = g(x, u)$, when starting close to each other, have their endpoint distances bounded exponentially in terms of their initial deviation, with a growth rate determined by the linear constant.
		\begin{assumption}
			\label{assumption:pcjumpmap}
			The jump map $g$ is such that there exist $K^{g}_{x}, K^{g}_{u}\in \mathbb{R}_{>0}$ such that, for each $(x_{0}, u_{0}) \in D$ and  each $(x_{1}, u_{1}) \in D$,
			$$
				|g(x_{0}, u_{0}) - g(x_{1}, u_{1})|\leq K^{g}_{x}|x_{0} - x_{1}| + K^{g}_{u}|u_{0} - u_{1}|.
			$$
		\end{assumption}
		The following lemma ensures that, for the discrete pieces of a motion plan, the reconstructed solution pair remains close to the backward partial motion plan.
		\begin{lemma}\label{lemma:jumpbound}
				Given a hybrid system $\hs$ satisfying Assumption \ref{assumption:pcjumpmap}, if there exists a purely discrete solution pair $\bw{\psi} = (\bw{\mystatet}, \bw{\myinputt})$ to $\bw{\hs}$, then the reconstructed solution $\simu{\mystatet}$, which is a solution to $\recon{\hs}$ satisfying R1 and R2 with $\bwr{\myinputt}$ being the reversal of $\bw{\myinputt}$, satisfies the following properties: 
				\begin{enumerate}[label= P\arabic*)]
					\item $\dom \simu{\mystatet} = \dom \bwr{\myinputt}$; thus, $\simu{\mystatet}$ is purely discrete;
					\item if $|\simu{\mystatet}(0, 0) - \bwr{\mystatet}(0, 0)| \leq \delta$, then $|\simu{\mystatet}(0, J) - \bwr{\mystatet}(0, J)| \leq \exp(J\ln(K_{x}^{g}))\delta$, where $(0, J) = \pn{\max \dom \simu{\mystatet} = \max \dom \bwr{\mystatet}}$, and $\max \dom \simu{\mystatet} = \max \dom \bwr{\mystatet}$ is valid because $\dom \simu{\mystatet} = \dom \bwr{\mystatet}$ in P1.
				\end{enumerate}
			\end{lemma}	
			\begin{proof}
				Given that $\bw{\psi}$ is purely discrete, $\bwr{\psi}$ is purely discrete and P1 is immediately established by (\ref{equation:equalhtd}). Furthermore, since $\simu{\mystatet}$ is generated by (\ref{model:reconhybridsystem}), R1 ensures that $|\simu{\mystatet}(0, 0) - \mystatet(0, 0)| \leq \delta$, and R2 ensures that $|\simu{\myinputt}(t, j) - \bwr{\myinputt}(t, j)| = 0$ for all $(t, j)\in \dom \simu{\myinputt} = \dom \bwr{\myinputt}$, in accordance with Assumption \ref{assumption:pcjumpmap}, we obtain the following inequality: 
				$$
				\begin{aligned}
				|\simu{\mystatet}(0, J) &- \mystatet(0, J)| \leq \exp(\ln K_{x}^{g})|\simu{\mystatet}(0, J - 1) - \mystatet(0, J - 1)| \\
				&+ K^{g}_{u} |\simu{\myinputt}(0, J - 1) - \bwr{\myinputt}(0, J - 1)|\\
				&= \exp(\ln K_{x}^{g})|\simu{\mystatet}(0, J - 1) - \mystatet(0, J - 1)|\\
				&\leq \exp(2\ln K_{x}^{g})|\simu{\mystatet}(0, J - 2) - \mystatet(0, J - 2)|\\
				&\leq \ \pnc{\vdots}\\
				&\leq \exp(J\ln K_{x}^{g})|\simu{\mystatet}(0, 0) - \mystatet(0, 0)| \leq \exp(J\ln K_{x}^{g})\delta,
				\end{aligned}
				$$
				thus establishing P2.
		\end{proof}
	
		Next, we show that the final state of the reconstructed motion plan $\simu{\mystatet}$ converges to $\bw{\mystatet}(0, 0)\in X_{f}$ as the tolerance $\delta$ in (\ref{equation:tolerance}) approaches zero. 
		\begin{theorem}\label{theorem:toleranceexist}
			Suppose Assumptions \ref{assumption:flowlipschitz} and \ref{assumption:pcjumpmap} are satisfied, and there exist a solution pair $\fw{\psi} = (\fw{\mystatet}, \fw{\myinputt})$ to $\fw{\hs}$ and a solution pair $\bw{\psi} = (\bw{\mystatet}, \bw{\myinputt})$ to $\bw{\hs}$ satisfying S1. For each $\epsilon > 0$, there exists a tolerance $\delta > 0$ in (\ref{equation:tolerance}) such that $
			|\fw{\mystatet}(\fw{T}, \fw{J}) - \bw{\mystatet}(\bw{T}, \bw{J})| \leq \delta$
			leads to
			$
			|\simu{\mystatet}(\simu{T}, \simu{J}) - \bw{\mystatet}(0, 0)| \leq \epsilon
			$
			where $(\fw{T}, \fw{J}) = \max \dom \fw{\mystatet}$, $(\bw{T}, \bw{J}) = \max \dom \bw{\mystatet}$, $\simu{\mystatet}$ is the reconstructed solution pair as a solution to $\recon{\hs}$ satisfying R1 and R2, and $(\simu{T}, \simu{J}) = \max \dom \simu{\mystatet}$.
		\end{theorem}
		\begin{proof}
				Given that $\bw{\psi} = (\bw{\mystatet}, \bw{\myinputt})$ is a solution pair to $\bw{\hs}$, Proposition ~\ref{theorem:reversedarcbwsystem} guarantees that $\bwr{\psi} = (\bwr{\mystatet}, \bwr{\myinputt})$ is a solution pair to $\fw{\hs}$. By (\ref{equation:equalhtd}), it follows that $\dom \simu{\mystatet} = \dom \bwr{\mystatet}$.
				Since $|\fw{\mystatet}(\fw{T}, \fw{J}) - \bw{\mystatet}(\bw{T}, \bw{J})|  = |\simu{\mystatet}(0, 0) - \bwr{\mystatet}(0, 0)| \leq \delta$, by iteratively applying Lemma \ref{lemma:flowbound} and Lemma \ref{lemma:jumpbound} throughout $\dom \simu{\mystatet} = \dom \bwr{\mystatet}$, it follows that 
				$$
				\begin{aligned}
				|\simu{\mystatet}(\simu{T}, \simu{J}) - \bw{\mystatet}(0, 0)| &= |\simu{\mystatet}(\simu{T}, \simu{J}) - \bwr{\mystatet}(\simu{T}, \simu{J})|\\
				&\leq \exp(K_{x}^{f}\simu{T} + \simu{J}\ln(K_{x}^{f}))\delta, 
				\end{aligned}$$
				which establishes the existence of $\delta > 0$. In particular, $\delta$ can be taken to be equal to $\frac{\epsilon}{\exp(K_{x}^{f}\simu{T} + \simu{J}\ln(K_{x}^{f}))}$.
		\end{proof}
		Furthermore, if $\bw{\mystatet}(0, 0)$ is not on the boundary of $X_{f}$, the following result shows there is a tolerance ensuring that $\simu{\mystatet}$ concludes within~$X_{f}$.
		\begin{corollary}\label{corollary:convergencexf}
			Suppose Assumptions \ref{assumption:flowlipschitz} and \ref{assumption:pcjumpmap} are satisfied, and there exist a  solution pair $\fw{\psi} = (\fw{\mystatet}, \fw{\myinputt})$ to $\fw{\hs}$, and a solution pair $\bw{\psi} = (\bw{\mystatet}, \bw{\myinputt})$ to $\bw{\hs}$ satisfying S1, and some $\epsilon' > 0$ such that $\bw{\mystatet}(0, 0) + \epsilon' \mathbb{B} \subset X_{f}$. Then, there exists a tolerance $\delta > 0$ in (\ref{equation:tolerance}) such that
	$
	|\fw{\mystatet}(\fw{T}, \fw{J}) - \bw{\mystatet}(\bw{T}, \bw{J})| \leq \delta
	$
	leads to 
	$
	\simu{\mystatet}(\simu{T}, \simu{J})\in X_{f}
	$
	where $(\fw{T}, \fw{J}) = \max \dom \fw{\mystatet}$, $(\bw{T}, \bw{J}) = \max \dom \bw{\mystatet}$, $\simu{\mystatet}$ is the reconstructed  as a solution to $\recon{\hs}$ satisfying R1 and R2, and $(\simu{T}, \simu{J}) = \max \dom \simu{\mystatet}$.
		\end{corollary}
		\begin{proof}
				By selecting $\epsilon = \epsilon'$, Theorem \ref{theorem:toleranceexist} ensures the existence of some $\delta> 0$ such that $|\simu{\mystatet}(\simu{T}, \simu{J}) - \bw{\mystatet}(0, 0)| \leq \epsilon'$. By $\bw{\mystatet}(0, 0) + \epsilon'\mathbb{B} \subset X_{f}$, it is established that $\simu{\mystatet}(\simu{T}, \simu{J})\in X_{f}$.
		\end{proof}
	
		Then, by replacing $\bw{\mystatet}$ with $\simu{\mystatet}$ and concatenating the reconstructed pair $\simu{\psi} := (\simu{\mystatet}, \bwr{\myinputt})$ to $\fw{\psi} = (\fw{\mystatet}, \fw{\myinputt})$, HyRRT-Connect generates the motion plan $\fw{\psi}|\simu{\psi}$, where the discontinuity associated with (\ref{equation:tolerance}) is removed. Note that the tolerance $\delta$ in (\ref{equation:tolerance}) is adjustable. Setting $\delta$ to a smaller value brings the endpoint of $\simu{\mystatet}$ closer to $X_{f}$, However, it also reduces the possibility of finding a motion plan, thereby increasing the time expected to find forward and backward partial motion plans.
		\begin{remark}
				Connecting two points via flow typically involves solving a two point boundary value problem constrained by the flow set. Solving such problems is difficult. This is the reason why we do not consider to actively connect two paths in $\fw{\stree}$ and $\bw{\stree}$ via flow.
	\end{remark}
\subsection{Connecting Two Search Trees via a Jump}\label{section:jumpconnection}
In S2, HyRRT-Connect checks the existence of $\fw{p}$ in~(\ref{equation:fwpath}) and $\bw{p}$ in (\ref{equation:bwpath})\pn{,} which, in addition to meeting C1-C4 in Section \ref{section:checksolution_stateequal}, results in a solution to the following constrained equation, denoted $u^{*}$, provided one exists\endnote{It is indeed possible that all the motion plans are purely continuous. In this case, no solution to (\ref{problem:jumpequation}) would be found since no jumps exist in every motion plan.}:
\begin{equation}\label{problem:jumpequation}
\overline{x}_{\bw{v}_{n}} = g(\overline{x}_{\fw{v}_{m}}, u^{*}),\quad(\overline{x}_{\fw{v}_{m}}, u^{*})\in \fw{D}.
\end{equation} 
\pn{At times, this} constrained equation can be solved analytically \pn{(e.g., for the system in Example~\ref{example:bouncingball}) or} numerically (e.g., using ideas in \cite{boyd2004convex})  using numerical techniques like root-finding methods \cite{burden2011numerical} or optimization methods \cite{boyd2004convex}.
A solution to (\ref{problem:jumpequation}) implies that $\overline{x}_{\fw{v}_{m}}$ and $\overline{x}_{\bw{v}_{n}}$ can be connected by applying $u^{*}$ at a jump from $\overline{x}_{\fw{v}_{m}}$ to $\overline{x}_{\bw{v}_{n}}$. Hence, a motion plan is constructed by concatenating $\fw{\psi}$, a single jump from $\overline{x}_{\fw{v}_{m}}$ to $\overline{x}_{\bw{v}_{n}}$, and $\bwr{\psi}$.
	\begin{remark}
		This approach enables the construction of a motion plan prior to the detection of any overlaps between $\fw{\stree}$ and $\bw{\stree}$, potentially leading to improved computational efficiency, \pn{as} illustrated in the forthcoming Section \ref{section:simulation}. Furthermore, as this connection is established through a jump, it prevents the discontinuity during the flow introduced in (\ref{equation:tolerance}).
\end{remark}

%


\section{Software Tool and Simulation Results}\label{section:simulation}
Algorithm \ref{algo:bihyrrt} leads to a software tool to solve Problem~\ref{problem:motionplanning}. This software only requires the inputs listed in Algorithm \ref{algo:bihyrrt}. We illustrate the HyRRT-Connect algorithm and this tool\endnote{Code at \hyperlink{https://github.com/HybridSystemsLab/HyRRTConnect}{https://github.com/HybridSystemsLab/HyRRTConnect.git}} in Example \ref{example:bouncingball}  and Example \ref{example:biped}.
\begin{example}[(Actuated bouncing ball system in Example ~\ref{example:bouncingball}, revisited)]
	We initially showcase the simulation results of the HyRRT-Connect algorithm without the functionality of connecting via jumps discussed in Section~\ref{section:jumpconnection}.
	We consider the case where HyRRT-Connect precisely connects the forward and backward partial motion plans. This is demonstrated by deliberately setting the initial state set as $X_{0} = \{(14, 0)\}$ and the final state set as $X_{f} = \{(0, -16.58)\}$. In this case, no tolerance is applied, and thus, no reconstruction process is employed. Instead, strict equality in C5 in Section \ref{section:checksolution_stateequal} is used to identify the motion plan. The motion plan detected under these settings is depicted in Figure~\ref{fig:examplebbexactflow}, where the forward and backward partial motion plans identified in S1 are depicted by the green and magenta lines, respectively.
However, for most scenarios, such as $X_{0} = \{(14, 0)\}$ and $X_{f} = \{(10, 0)\}$ in Example \ref{example:bouncingball}, if we require strict equality without allowing any tolerance \sj{as in (\ref{equation:tolerance})}, then HyRRT-Connect fails to return a motion plans in almost all the runs. This demonstrates the necessity of allowing a certain degree of tolerance in (\ref{equation:tolerance}). The simulation results, allowing a tolerance of $\delta = 0.2$, are shown in Figure \ref{fig:examplebbinexactflow}. A discontinuity during the flow between the partial motion plans is observed, as depicted in the red circle in Figure \ref{fig:examplebbinexactflow}. This discontinuity is addressed through the reconstruction process and a deviation between the endpoint of the reconstructed motion plan and the final state set is also observed in Figure \ref{fig:examplebbinexactrecon}, which, according to Theorem \ref{theorem:toleranceexist}, is bounded.
	\begin{figure}[htbp]
		\centering
		\includegraphics[width=0.7\columnwidth]{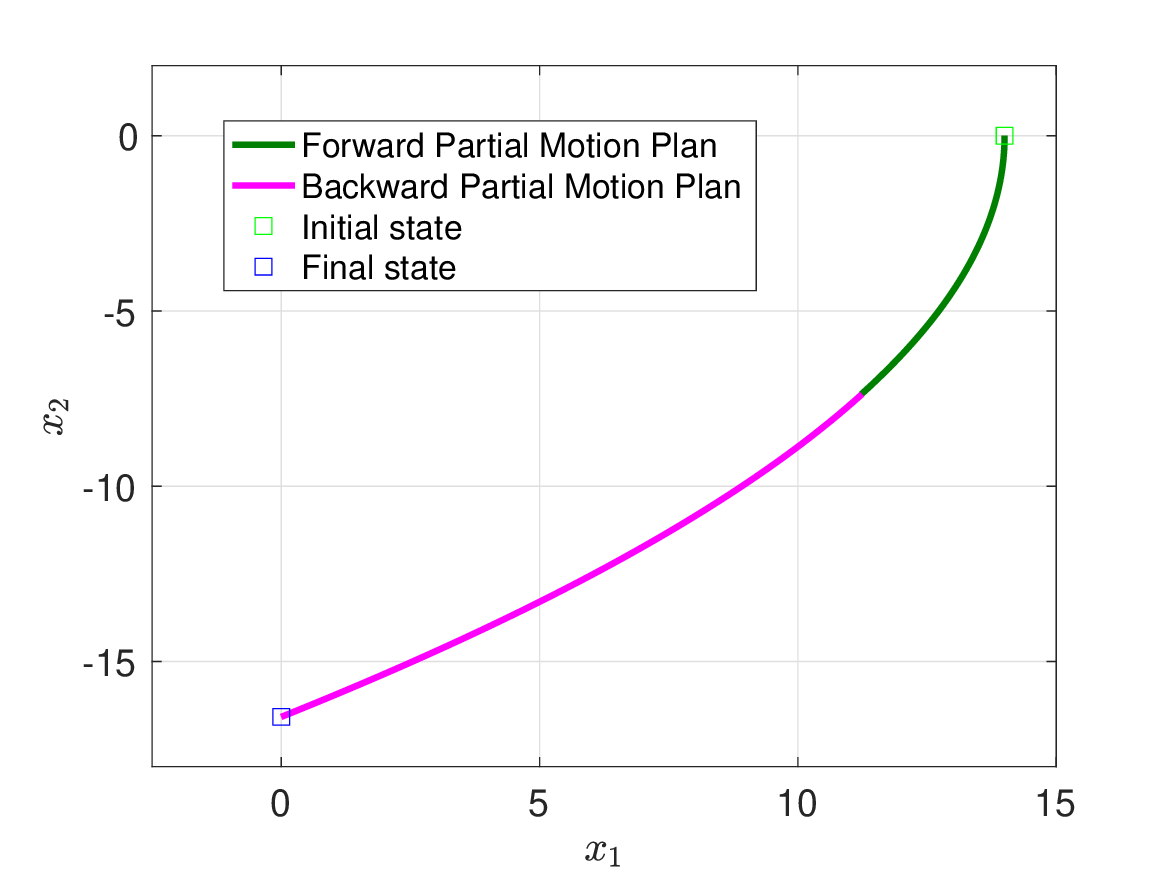}
		\caption{Motion plans for actuated bouncing ball example solved by HyRRT-Connect, where precise connection during the flow is achieved.}\label{fig:examplebbexactflow}
	\end{figure}
	\begin{figure}[htbp]
		\centering
		\includegraphics[width=0.7\columnwidth]{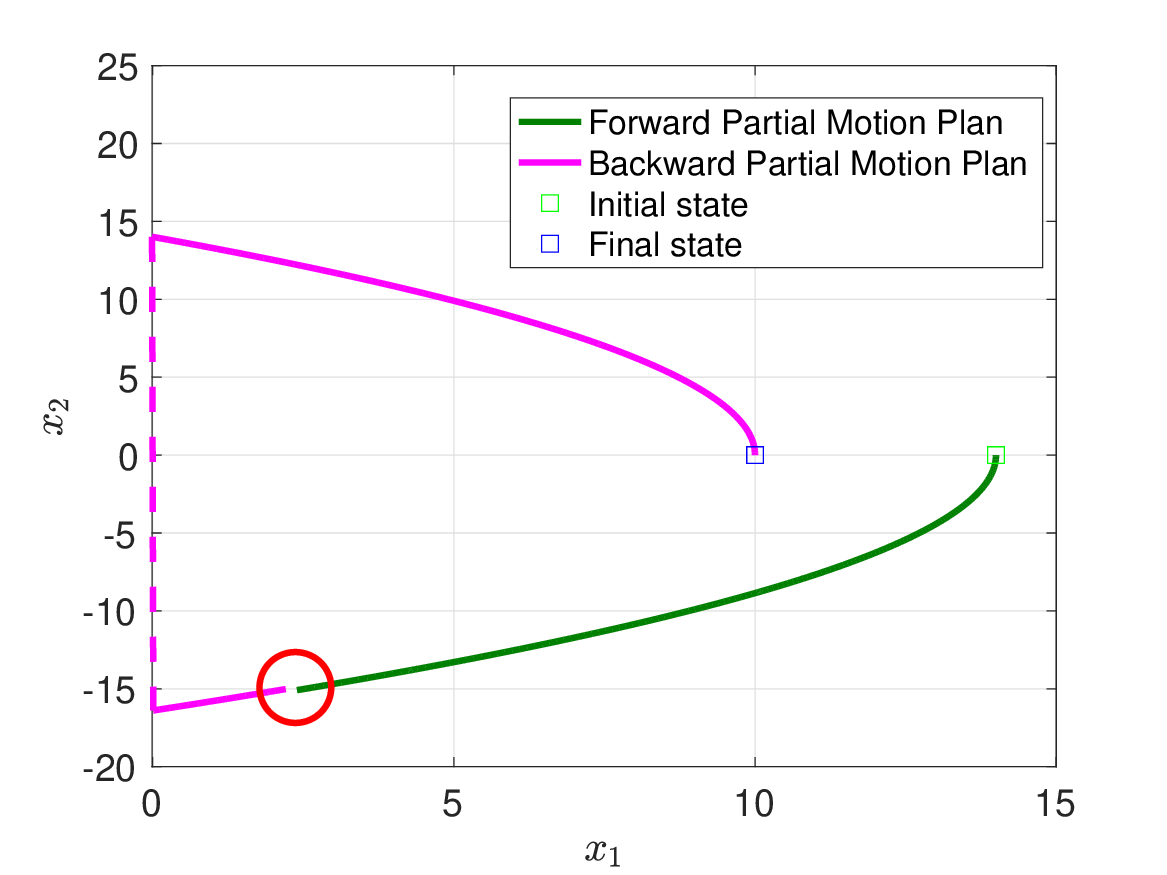}
		\caption{Motion plans for actuated bouncing ball example solved by HyRRT-Connect, where  a discontinuity during the flow, as is depicted in red circle, is observed.}\label{fig:examplebbinexactflow}
	\end{figure}
	\begin{figure}[htbp]
		\centering
		\includegraphics[width=0.7\columnwidth]{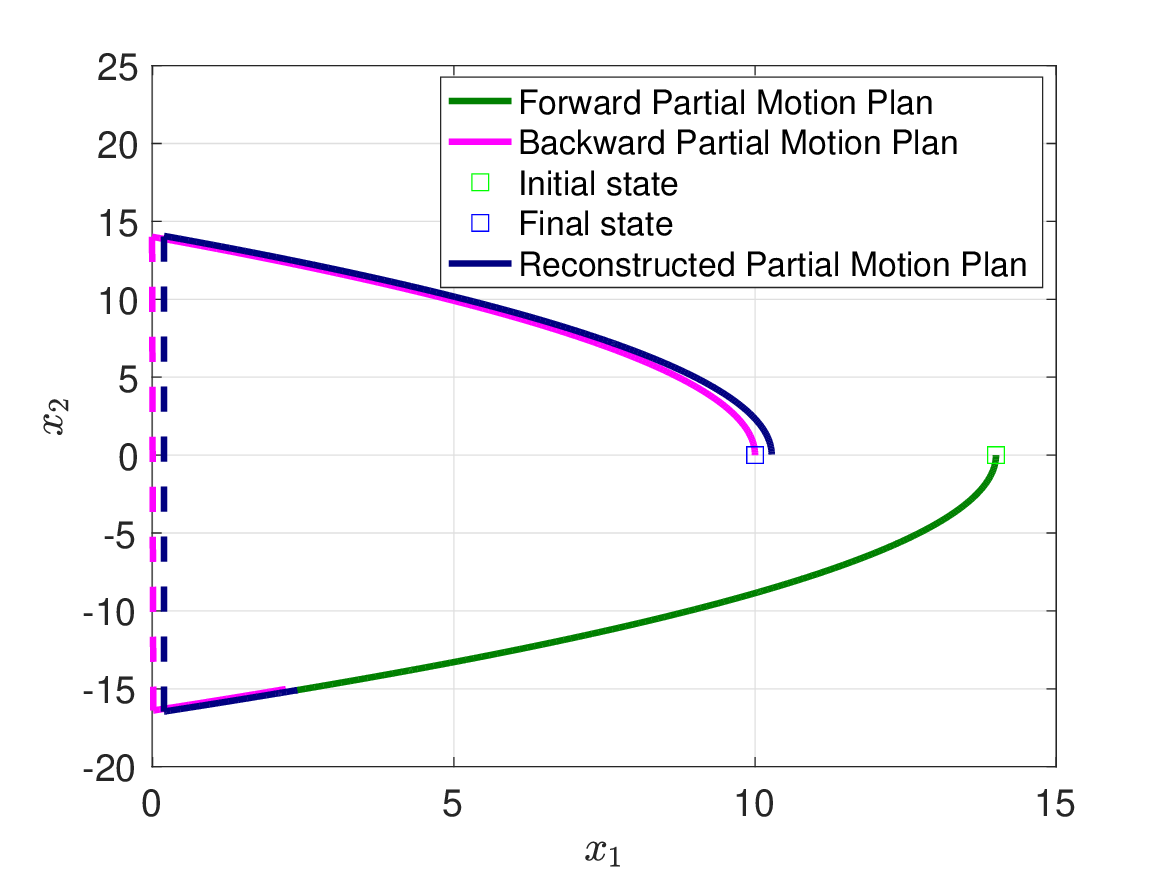}
		\caption{Motion plans for actuated bouncing ball example solved by HyRRT-Connect, where the backward partial motion plan is reconstructed.}\label{fig:examplebbinexactrecon}
	\end{figure}
	\begin{figure}[htbp]
		\centering
		\subfigure[HyRRT-Connect\label{fig:examplehySST}]{\includegraphics[width=0.7\columnwidth]{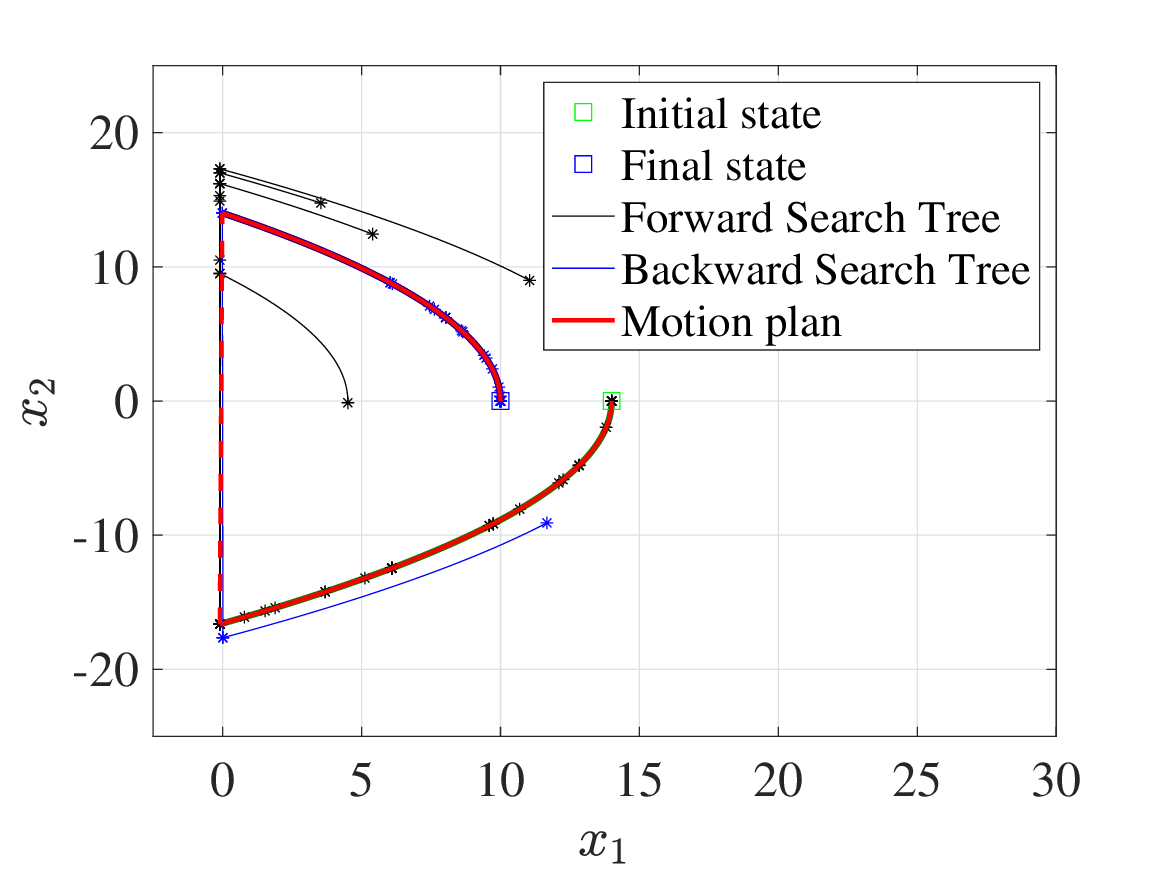}}
		\subfigure[HyRRT\label{fig:examplehyRRT}]{\includegraphics[width=0.7\columnwidth]{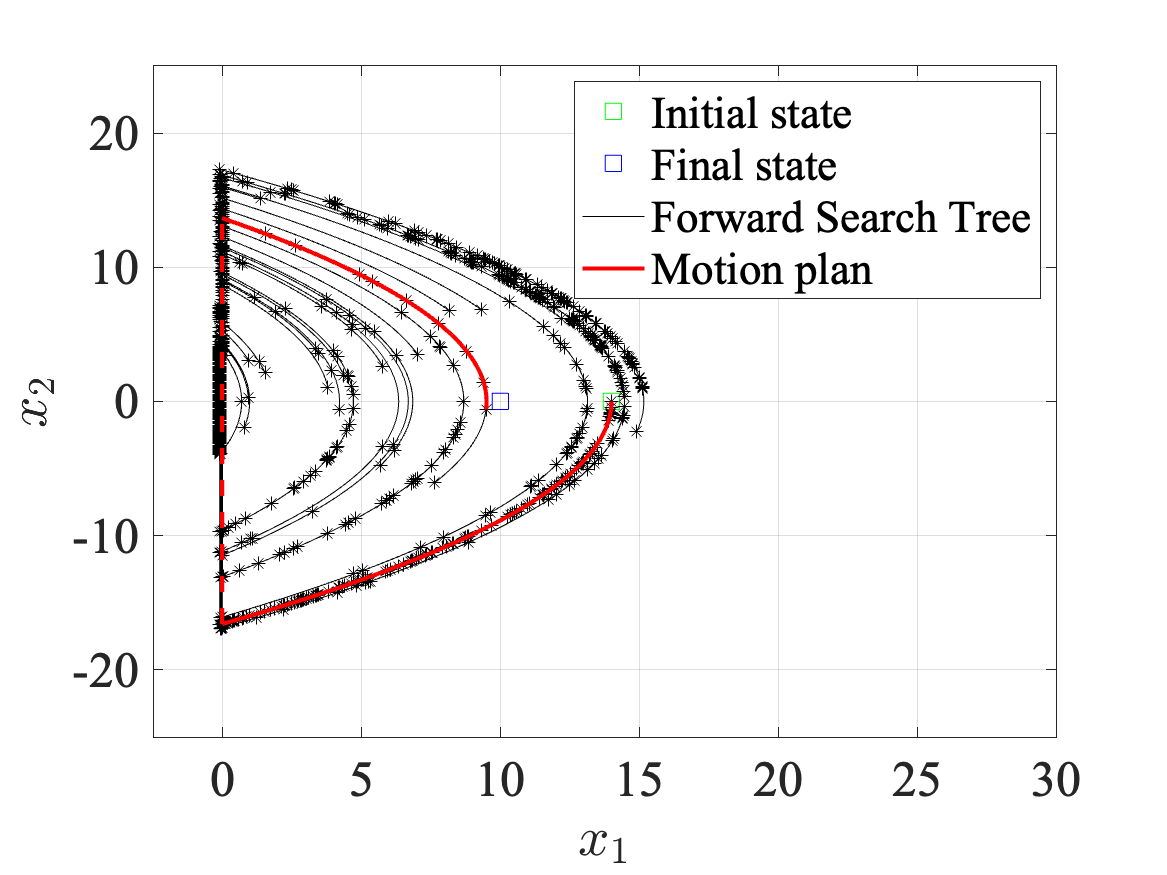}}
		\caption{Motion plans for the actuated bouncing ball example solved by HyRRT-Connect and HyRRT in \cite{wang2025motion}\label{fig:examplebb}.}
\end{figure}

We proceed to perform simulation results of HyRRT-Connect showcasing its full functionalities, including the ability to connect partial motion plans via jumps, and results are shown in Figure~\ref{fig:examplehySST}. This feature enables HyRRT-Connect to avoid discontinuities during the flow, as it computes exact solutions at jumps to connect forward and backward partial motion plans.
Furthermore, we compare the computational performance of the proposed HyRRT-Connect algorithm, its variant Bi-HyRRT (where the function to connect partial motion plans via jumps is disabled), and HyRRT.
 Conducted on a $3.5$GHz Intel Core i7 processor using MATLAB, each algorithm is run $20$ times on the same problem. HyRRT-Connect on average creates $78.8$ vertices in $0.27$ seconds, Bi-HyRRT $186.5$ vertices in $0.76$ seconds, and HyRRT $4\textsl{}57.4$ vertices in $3.93$ seconds. Compared to HyRRT, both HyRRT-Connect and Bi-HyRRT show considerable improvements in computational efficiency. Notably, HyRRT-Connect, with its jump-connecting capability, achieves a $64.5\%$ reduction in computation time and $57.7\%$ fewer vertices than Bi-HyRRT, demonstrating the benefits of jump connections.
\begin{table}
	\centering
	\begin{tabular}{c|ccc}
		\hline
		&HyRRT & Bi-HyRRT & Bi-HyRRT-Connect \\
		\hline
		Time & $3.93$ & $0.76$ & $0.27$\\
		Vertices & $457.4$ & $186.5$ & $78.8$\\
		\hline
	\end{tabular}
	\caption{Computation performance using different RRT-type algorithms.}
\end{table}
\end{example}
\begin{example}[(Walking robot system in Example ~\ref{example:biped}, revisited)]
		The simulation results demonstrate that HyRRT-Connect successfully finds a motion plan for the high-dimensional walking robot system with a tolerance $\delta$ of $0.3$. The forward search tree $\fw{\stree}$, with its partial motion plan shown in green, is displayed in Figure \ref{fig:examplebipedfw}. Similarly, the backward search tree $\bw{\stree}$, with its partial motion plan in magenta, is shown in Figure \ref{fig:examplebipedbw}. 
		These simulations were performed in MATLAB on a 3.5 GHz Intel Core i7 processor. Running HyRRT-Connect and HyRRT $20$ times each for the same problem, HyRRT-Connect generates $470.2$ vertices and takes $19.8$ seconds, while HyRRT generates $2357.1$ vertices and takes $71.5$ seconds. This indicates a significant $72.3\%$ improvement in computation time and $80.1\%$ in vertex creation for HyRRT-Connect compared to HyRRT, highlighting the efficiency of bidirectional exploration.
\end{example}

\section{Discussion on Parallel Implementation and Probabilistic Completeness}\label{section:parallelcomputation}
After each new vertex is added to the search tree, the HyRRT-Connect algorithm frequently halts and restarts parallel computations to check for overlaps between the forward and backward search trees. This process can prevent the potential computational performance improvement from parallelization. Our simulations using MATLAB's $\texttt{parpool}$ for parallel computation showed no improvement compared to an interleaved approach. In fact, the parallel implementation took significantly longer: about $2.47$ seconds for the actuated bouncing ball system, compared to $0.27$ second with interleaving, and $167.8$ seconds for the walking robot system versus only $19.8$ seconds. 
It is important to note that the time required for halting and restarting parallel computation is contingent upon factors like the specific parallel computation software toolbox used, the hardware platform, and other implementation details. Consequently, the conclusions regarding performance may differ with varying implementations.
 
The HyRRT-Connect algorithm renders probabilistic completeness, implying that as the number of samples approaches infinity, the probability of failure to find a motion plan converges to zero if one exists. This property extends from the probabilistic completeness of the HyRRT algorithm shown in \cite{wang2025motion}, which propagates in a forward direction. By truncating an existing motion plan into two segments, HyRRT maintains probabilistic completeness in finding each segment, thereby ensuring the probabilistic completeness of HyRRT-Connect.
%
%
%
\begin{figure}[htbp]
	\centering
	\subfigure[Forward search tree and partial motion plan.\label{fig:examplebipedfw}]{\includegraphics[width=0.7\columnwidth]{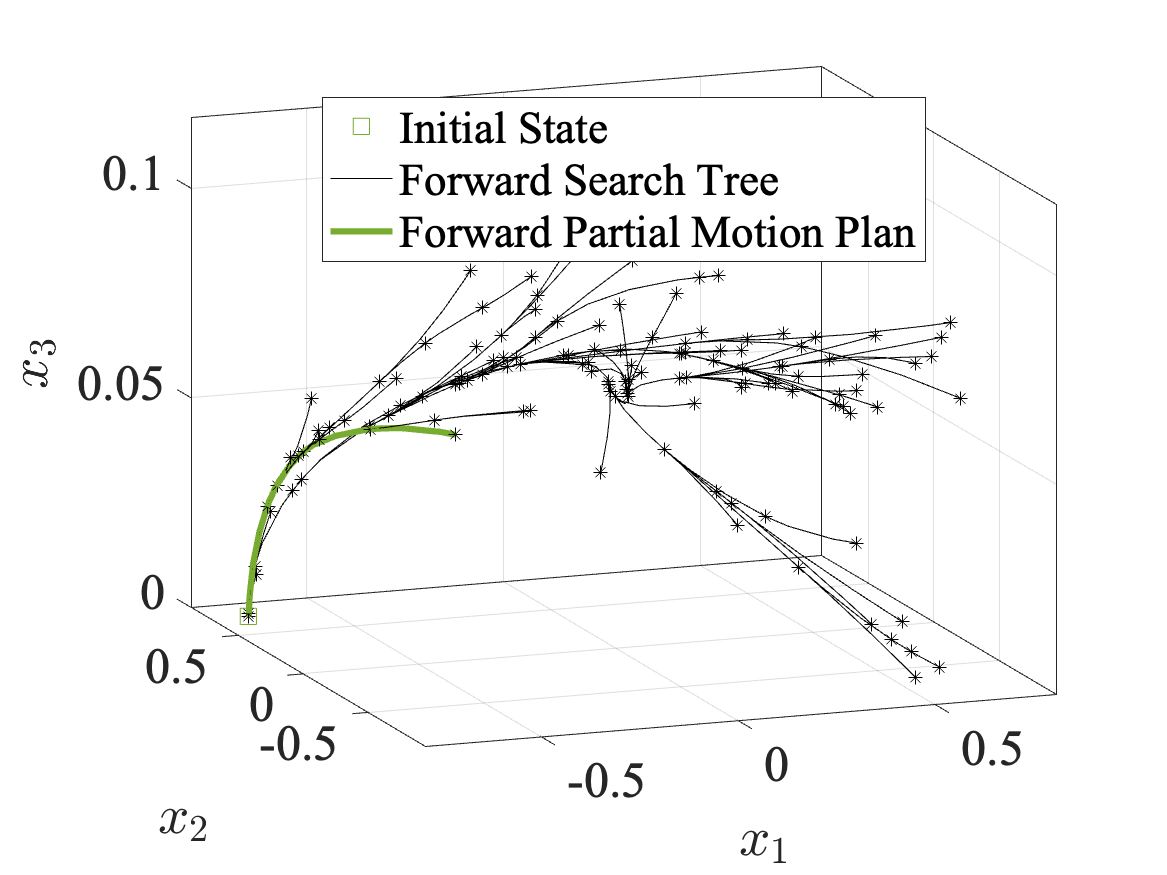}}
	\subfigure[Backward search tree and partial motion plan.\label{fig:examplebipedbw}]{\includegraphics[width=0.7\columnwidth]{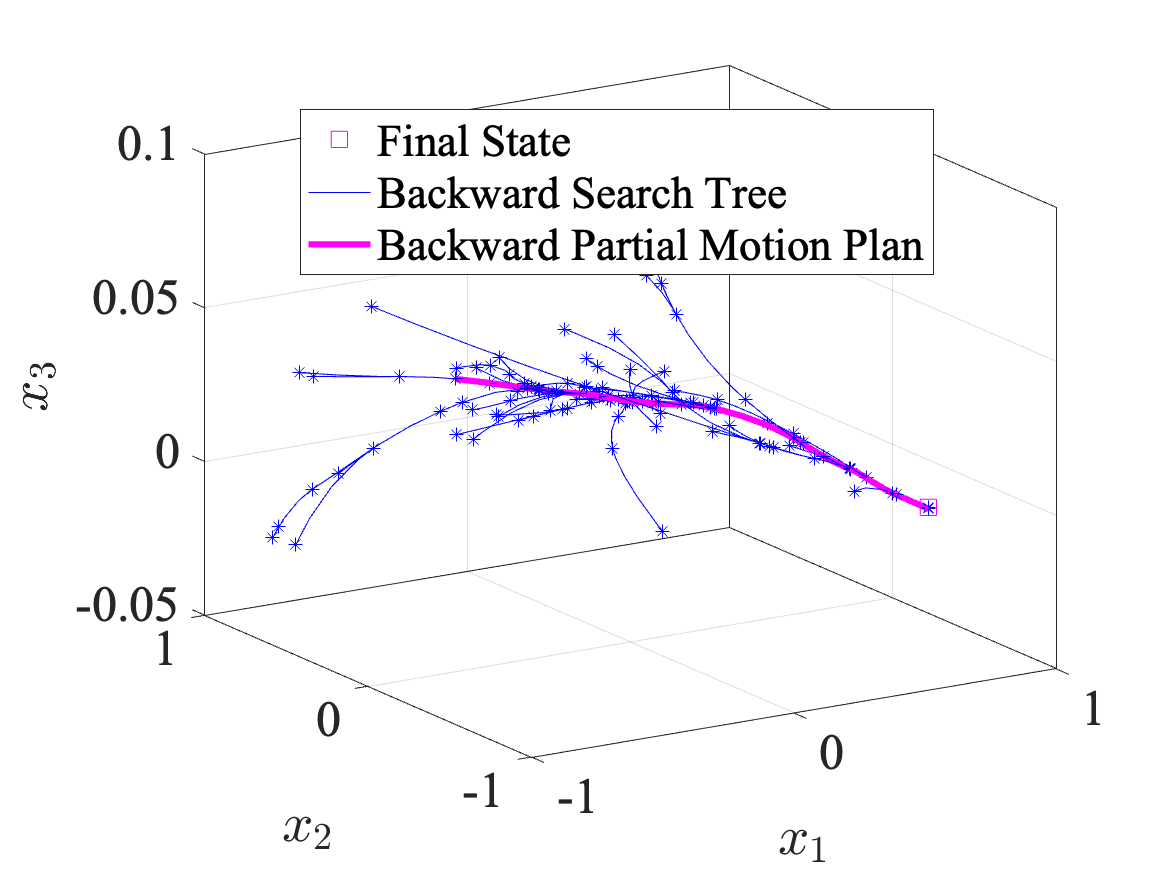}}
				\subfigure[Forward and reconstructed backward partial motion plans. \label{fig:examplebiped}]{\includegraphics[width=0.7\columnwidth]{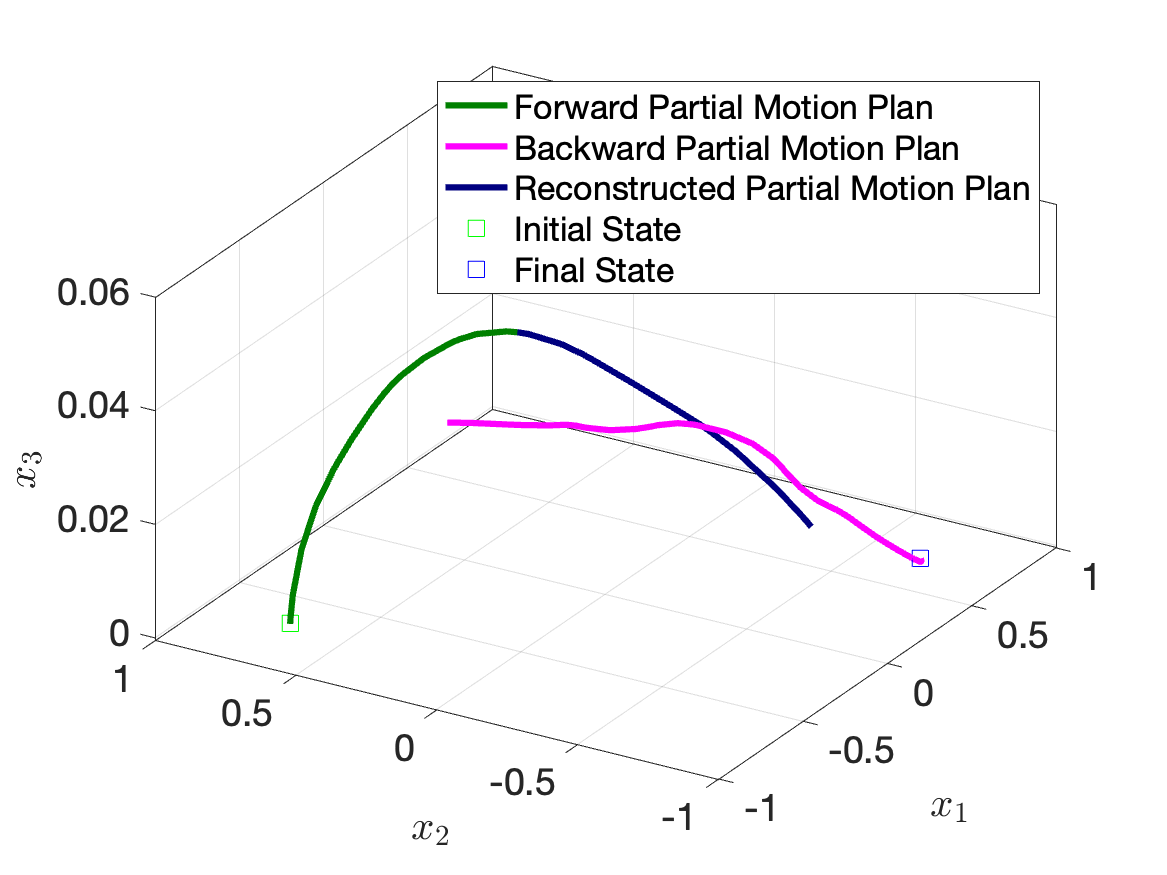}}
	\caption{Selected states of the partial motion plans for the walking robot system.}
\end{figure}

\section{Conclusion}
We present HyRRT-Connect, a bidirectional algorithm to solve  motion planning problems for hybrid systems. It includes a backward-in-time hybrid system formulation, validated by reversing and concatenating functions on the hybrid time domain. To tackle discontinuities during the flow, we introduce a reconstruction process.  In addition to smoothening the discontinuity, this process ensures convergence to the final state set as discontinuity tolerance decreases. The effectiveness and computational improvement of HyRRT-Connect are exemplified through applications to an actuated bouncing ball and a walking robot. Future work includes addressing motion planning under uncertaints and developing tracking control algorithms for hybrid systems to follow the planned motion.
\theendnotes
\bibliographystyle{SageH}
\bibliography{references}
\appendix
\section{Proof of Proposition \ref{theorem:concatenationsolution}}\label{appendix:proof_concatenation}
We show that $\psi = (\mystatet,  \myinputt) = (\mystatet_{1}|\mystatet_{2}, \myinputt_{1}|\myinputt_{2})$ satisfies the conditions in Definition \ref{definition:solution}, as follows.
\begin{lemma}\label{lemma:concatenationHTD}
	Given two functions, denoted $\psi_{1}$ and $\psi_{2}$ and defined on hybrid time domains, the domain of their concatenation $\psi = \psi_{1}|\psi_{2}$, is also a hybrid time domain.
\end{lemma}
\begin{proof}
	Because $\dom \psi_{1}$ and $\dom \psi_{2}$ are hybrid time domains, according to Definition \ref{definition:hybridtimedomain}, for each $(T_{1}, J_{1})\in \dom \psi_{1}$ and $(T_{2}, J_{2})\in \dom \psi_{2}$,
	$$
	\dom \psi_{1} \cap ([0, T_{1}]\times \{0,1,..., J_{1}\}) = \cup_{j = 0}^{J_{1} }([t^{1}_{j}, t^{1}_{j + 1}], j)
	$$
	for some finite sequence of times $0=t^{1}_{0}\leq t^{1}_{1}\leq t^{1}_{2}\leq ...\leq t^{1}_{J_{1} + 1}  = T_{1}$ and 
	$$
	\dom \psi_{2} \cap ([0, T_{2}]\times \{0,1,..., J_{2}\}) = \cup_{j = 0}^{J_{2} }([t^{2}_{j}, t^{2}_{j + 1}], j)
	$$
	for some finite sequence of times $0=t^{2}_{0}\leq t^{2}_{1}\leq t^{2}_{2}\leq ...\leq t^{2}_{J_{2} + 1} = T_{2}$.
	
	Since $\dom \psi = \dom \psi_{1} \cup (\dom \psi_{2} + \{(T, J)\})$ where $(T, J) = \max \dom \psi_{1}$, then for each $(T', J')\in \dom \psi$,
	\begin{equation}\label{equation:concatenation_proof1}
		\begin{aligned}
		&\dom \psi \cap ([0, T']\times \{0,1,..., J'\}) \\
		&= \dom \psi_{1} \cup (\dom \psi_{2} + \{(T, J)\})\\
		& \cap ([0, T']\times \{0,1,..., J'\}) \\
		&= ( \dom \psi_{1}  \cap ([0, T']\times \{0,1,..., J'\})) \cup \\
		&(\dom \psi_{2} + \{(T, J)\})\cap ([0, T']\times \{0,1,..., J'\}) \\
		\end{aligned}
	\end{equation}
	
	Since $(T', J')\in \dom \psi$ and $\dom \psi = \dom \psi_{1} \cup (\dom \psi_{2} + \{(T, J)\})$, either of the following hold
	\begin{enumerate}
		\item $(T', J') \in \dom \psi_{1}$ such that $T'\leq T$ and $J'\leq J$;
		\item $(T', J') \in  (\dom \psi_{2} + \{(T, J)\})$ such that $T'\geq T$ and $J'\geq J$;
	\end{enumerate}
	
	For the first item, (\ref{equation:concatenation_proof1}) can be written as
	$$
	\begin{aligned}
	&\dom \psi \cap ([0, T']\times \{0,1,..., J'\})\\
	&= ( \dom \psi_{1}  \cap ([0, T']\times \{0,1,..., J'\})) \cup \\
	&(\dom \psi_{2} + \{(T, J)\})\cap ([0, T']\times \{0,1,..., J'\}) \\
	& = \cup_{j = 0}^{J'}([t^{1}_{j}, t^{1}_{j + 1}], j) \cup \emptyset = \cup_{j = 0}^{J'}([t^{1}_{j}, t^{1}_{j + 1}], j)
	\end{aligned}
	$$
	for some finite sequence of times $0=t^{1}_{0}\leq t^{1}_{1}\leq t^{1}_{2}\leq ...\leq t^{1}_{J'+ 1} = T'$. 
	
	For the second item, (\ref{equation:concatenation_proof1}) can be written as
	\begin{equation}
	\begin{aligned}
	&\dom \psi \cap ([0, T']\times \{0,1,..., J'\})\\
	&= ( \dom \psi_{1}  \cap ([0, T']\times \{0,1,..., J'\})) \cup \\
	&(\dom \psi_{2} + \{(T, J)\})\cap ([0, T']\times \{0,1,..., J'\}) \\
	& = ( \dom \psi_{1}  \cap ([0, T]\times \{0,1,..., J\})) \cup \\
	&(\dom \psi_{2} + \{(T, J)\})\cap ([T, T']\times \{J,J + 1,..., J'\}) \\
	& = ( \dom \psi_{1}  \cap ([0, T]\times \{0,1,..., J\})) \cup \\
	&(\dom \psi_{2}\cap ([0, T' - T]\times \{0,1,..., J' - J\}) + \{(T, J)\})\\
	& = \cup_{j = 0}^{J}([t^{1}_{j}, t^{1}_{j + 1}], j) \cup (\cup_{j = 0}^{J' - J}([t^{2}_{j}, t^{2}_{j + 1}], j) + \{(T, J)\})\\
	& = \cup_{j = 0}^{J}([t^{1}_{j}, t^{1}_{j + 1}], j) \cup (\cup_{j = 0}^{J' - J}([t^{2}_{j} + T, t^{2}_{j + 1} + T], j + J))\\
	\end{aligned}
	\end{equation}
	for some finite sequence of times $0=t^{1}_{0}\leq t^{1}_{1}\leq t^{1}_{2}\leq ...\leq t^{1}_{J + 1} = T$ and $T=t^{2}_{0} + T\leq t^{2}_{1} + T\leq t^{2}_{2}+ T\leq ...\leq t^{2}_{J' - J + 1} + T= T' $. Hence, a finite sequence $\{t_{i}\}_{i = 0}^{J'}$ can be constructed as follows such that $0=t_{0}\leq t_{1}\leq t_{2}\leq ...\leq t_{J'} = T'$.
	$$
	t_{i} = \left\{\begin{aligned}
	t^{1}_{i}\qquad &i\leq J\\
	t^{2}_{i - J} + T\qquad &J + 1\leq i \leq J'\\
	\end{aligned}\right.
	$$
	
	Therefore, there always exists a finite sequence of $\{t_{i}\}_{i = 0}^{J'}$ such that for each $(T', J')\in \dom \psi$,
	$$
	\dom \psi \cap ([0, T']\times \{0,1,..., J'\}) =  \cup_{j = 0}^{J'}([t_{j}, t_{j + 1}], j). 
	$$ Hence, $\dom \psi$ is a hybrid time domain.  \ifbool{conf}{\qed}{}
\end{proof}
\begin{definition}(Lebesgue measurable set and Lebesgue measurable function)\label{definition:lebesguemeasurable}
	A set $S\subset \myreals$ is said to be \emph{Lebesgue measurable}  if it has positive Lebesgue measure $\mu(S)$, where $\mu(S) = \sum_{k} (b_{k} - a_{k})$ and $(a_{k}, b_{k})$'s are all of the disjoint open intervals in $S$. Then $s: S\to \myreals[n_{s}]$ is said to be  a \emph{Lebesgue measurable function} if for every open set $\mathcal{U}\subset \myreals[n_{s}]$ then set $\{r\in S: s(r) \in \mathcal{U}\}$ is Lebesgue measurable.
\end{definition}
\begin{lemma}\label{lemma:concatenation_input_Lebesguemeasurable}
	Given two hybrid inputs, denoted $\myinputt_{1}$ and $\myinputt_{2}$, their concatenation $\myinputt = \myinputt_{1}|\myinputt_{2}$, is such that for each $j\in \nnumbers$, $t\mapsto \myinputt(t,j)$ is Lebesgue measurable on the interval $I^{j}_{\myinputt}:=\{t:(t, j)\in \dom \myinputt\}$.
\end{lemma}
\begin{proof}
	We show that for all $j\in \mathbb{N}$, the function $t\mapsto \myinputt(t, j)$ satisfies conditions in Definition \ref{definition:lebesguemeasurable} on the interval $I^{j}_{\myinputt}:=\{t:(t, j)\in \dom \myinputt\}$. 
	
	According to Definition \ref{definition:concatenation}, $\dom \psi = \dom \psi_{1} \cup (\dom \psi_{2} + \{(T, J)\})$ where $(T, J) = \max \dom \psi_{1}$. For each $j\in \nnumbers$, there could be three cases:
	1) $j < J$;
	2) $j > J$;
	3) $j = J$.
	Then we show in each case,  $t\mapsto \myinputt(t,j)$ is Lebesgue measurable.
	\begin{enumerate}[label=\Roman*]
		\item Consider the case when $j < J$. Since $\myinputt_{1}$ is a hybrid input, the functions $t\mapsto \myinputt_{1}(t, j)$ is Lebesgue measurable on the interval $I^{j}_{\myinputt_{1}}:=\{t:(t, j)\in \dom \myinputt_{1}\}$.  Definition \ref{definition:concatenation} implies that $\myinputt(t, j) = \myinputt_{1}(t, j)$ holds for all $(t, j)\in \dom \myinputt_{1}$ such that $j < J$. Then $t\mapsto \myinputt(t, j) = \myinputt_{1}(t, j) $ is Lebesgue measurable on the interval $I^{j}_{\psi} = I^{j}_{\psi_{1}}$ if $j < J$. 
		\item Consider the case when $j > J$. Since $\myinputt_{2}$ is a hybrid input, then $t\mapsto \myinputt_{2}(t, j)$ is Lebesgue measurable on the interval $I^{j}_{\psi_{2}}$. 
		Definition \ref{definition:concatenation} implies that  $\myinputt(t, j) = \myinputt_{2}(t - T, j - J)$ for all $(t, j)\in \dom \myinputt_{2} + \{(T, J)\}$ such that $j > J$. Then $t\mapsto \myinputt(t, j) = \myinputt_{2}(t - T, j - J)$ is Lebesgue measurable on the interval $I^{j}_{\psi} = I^{j - J}_{\psi_{2}} + \{T\}$ if $j > J$.
		\item Consider the case when $j = J$. According to Definition \ref{definition:concatenation}, 
		$$
		\myinputt(t, J) = \left\{
		\begin{aligned}
		&\myinputt_{1}(t, J) &t\in I^{J}_{\psi_{1}}\backslash \{T\}\\
		&\myinputt_{2}(t - T,0) &t \in I^{0}_{\psi_{2}} + \{T\}.\\
		\end{aligned}
		\right.
		$$
		Note that $t\mapsto \myinputt_{1}(t, J)$ is Lebesgue measurable on the interval $I_{\myinputt_{1}}^{J}$ and $t\mapsto \myinputt_{2}(t, 0)$ is Lebesgue measurable on the interval $I_{\myinputt_{2}}^{0}$. According to Definition \ref{definition:lebesguemeasurable}, 
		\begin{enumerate}
			\item for every open set $\mathcal{U}_{1}\subset \{\myinputt_{1}(t, J)\in \myreals[m]: t \in I_{\myinputt_{1}}^{J}\}$, the set $\{t\in I_{\myinputt_{1}}^{J}: \myinputt_{1}(t, J) \in \mathcal{U}_{1}\}$ is Lebesgue measurable;
			\item for every open set $\mathcal{U}_{2}\subset \{\myinputt_{2}(t, 0)\in \myreals[m]: t \in I_{\myinputt_{2}}^{0}\}$, the set $\{t\in I_{\myinputt_{2}}^{0}: \myinputt_{2}(t, 0) \in \mathcal{U}_{2}\}$ is Lebesgue measurable. 
		\end{enumerate} 
		Then, for any open set $\mathcal{U}\subset \{\myinputt(t, J)\in \myreals[m]: t \in I_{\myinputt}^{J}\}$, define $\widetilde{\mathcal{U}}_{1} := \mathcal{U}\cap \{\myinputt_{1}(t, J)\in \myreals[m]: t \in I_{\myinputt_{1}}^{J}\}$ and $\widetilde{\mathcal{U}}_{2} := \mathcal{U}\cap \{\myinputt_{2}(t, 0)\in \myreals[m]: t \in I_{\myinputt_{2}}^{0}\}$. Note that 
		$$
		\begin{aligned}
		&\{\myinputt(t, J)\in \myreals[m]: t \in I_{\myinputt}^{J}\} = \\
		& \{\myinputt_{1}(t, J)\in \myreals[m]: t \in I_{\myinputt_{1}}^{J}\}\cup  \{\myinputt_{2}(t, 0)\in \myreals[m]: t \in I_{\myinputt_{2}}^{0}\}
		\end{aligned}$$
		Therefore, at least one of $\widetilde{\mathcal{U}}_{1} $ and $\widetilde{\mathcal{U}}_{2}$ is non-empty. Then the set $\{t\in I_{\myinputt}^{J}: \myinputt(t, J) \in \mathcal{U}\} \supset \{t\in I_{\myinputt_{1}}^{J}: \myinputt_{1}(t, J) \in \widetilde{\mathcal{U}}_{1}\} \cup \{t\in I_{\myinputt_{2}}^{0}: \myinputt_{2}(t, 0) \in \widetilde{\mathcal{U}}_{2}\} + \{T\}$. Note that at least one of the sets $\{t\in I_{\myinputt_{1}}^{J}: \myinputt_{1}(t, J) \in \widetilde{\mathcal{U}}_{1}\}$ and $\{t\in I_{\myinputt_{2}}^{0}: \myinputt_{2}(t, 0) \in \widetilde{\mathcal{U}}_{2}\}$ is Lebesgue measurable. Therefore, the set $\{t\in I_{\myinputt}^{J}: \myinputt(t, J) \in \mathcal{U}\} $  is Lebesgue measurable.
	\end{enumerate}
	Therefore, $t\mapsto \myinputt(t, j)$ is Lebesgue measurable on the interval $I_{\myinputt}^{j}$ for all $j\in \mathbb{N}$. \ifbool{conf}{\qed}{}
\end{proof}
\begin{definition} (Locally essentially bounded function)\label{definition:locallybounded}
	A function $s: S\to \myreals[n_{s}]$ is said to be \emph{locally essentially bounded} if for any $r\in S$ there exists an open neighborhood $\mathcal{U}$ of $r$ such that $s$ is bounded almost everywhere\endnote{A property is said to hold almost everywhere or for almost all points in a set $S$ if the set of elements of $S$ at which the property does not hold has zero Lebesgue measure.} on $\mathcal{U}$; i.e., there exists $c\geq 0$ such that $|s(r)|\leq c$ for almost all $r\in \mathcal{U}\cap S$.
\end{definition}
\begin{lemma}\label{lemma:concatenation_input_locallyessentiallybounded}
	Given two hybrid inputs, denoted $\myinputt_{1}$ and $\myinputt_{2}$, their concatenation $\myinputt = \myinputt_{1}|\myinputt_{2}$, is such that for each $j\in \nnumbers$, $t\mapsto \myinputt(t,j)$ is locally essentially bounded on the interval $I^{j}_{\myinputt}:=\{t:(t, j)\in \dom \myinputt\}$.
\end{lemma}
\begin{proof}
	We show that for all $j\in \mathbb{N}$, the function $t\mapsto \myinputt(t, j)$ satisfies conditions in Definition \ref{definition:locallybounded} on the interval $I^{j}_{\myinputt}:=\{t:(t, j)\in \dom \myinputt\}$. 
	
	According to Definition \ref{definition:concatenation}, $\dom \psi = \dom \myinputt_{1} \cup (\dom \myinputt_{2} + \{(T, J)\})$ where $(T, J) = \max \dom \myinputt_{1}$. For each $j\in \nnumbers$, there could be three cases:
		1) $j < J$;
		2) $j > J$;
		3) $j = J$.
	Then we show in each case,  $t\mapsto \myinputt(t,j)$ is locally essentially bounded.
	\begin{enumerate}[label=\Roman*]
		\item Consider the case when $j < J$. Since $\myinputt_{1}$ is a hybrid input, the functions $t\mapsto \myinputt_{1}(t, j)$ is locally essentially bounded on the interval $I^{j}_{\myinputt_{1}}:=\{t:(t, j)\in \dom \myinputt_{1}\}$.  Definition \ref{definition:concatenation} implies that $\myinputt(t, j) = \myinputt_{1}(t, j)$ holds for all $(t, j)\in \dom \myinputt_{1}$ such that $j < J$. Then $t\mapsto \myinputt(t, j) = \myinputt_{1}(t, j) $ is locally essentially bounded on the interval $I^{j}_{\psi} = I^{j}_{\psi_{1}}$ if $j < J$. 
		\item Consider the case when $j > J$. Since $\myinputt_{2}$ is a hybrid input, then $t\mapsto \myinputt_{2}(t, j)$ is locally essentially bounded on the interval $I^{j}_{\myinputt_{2}}$. 
		Definition \ref{definition:concatenation} implies that  $\myinputt(t, j) = \myinputt_{2}(t - T, j - J)$ for all $(t, j)\in \dom \myinputt_{2} + \{(T, J)\}$ such that $j > J$. Then $t\mapsto \myinputt(t, j) = \myinputt_{2}(t - T, j - J)$ is locally essentially bounded on the interval $I^{j}_{\myinputt} = I^{j - J}_{\myinputt_{2}} + \{T\}$ if $j > J$.
		\item Consider the case when $j = J$. According to Definition \ref{definition:concatenation}, 
		$$
		\myinputt(t, J) = \left\{
		\begin{aligned}
		&\myinputt_{1}(t, J) &t\in I^{J}_{\myinputt_{1}}\backslash \{T\}\\
		&\myinputt_{2}(t - T,0) &t \in I^{0}_{\myinputt_{2}} + \{T\}.\\
		\end{aligned}
		\right.
		$$
		
		Note that $t\mapsto \myinputt_{1}(t, j)$ is locally bounded on the interval $I^{J}_{\myinputt_{1}}$ and $t\mapsto \myinputt_{2}(t, 0)$ is Lebesgue measurable on the interval $I_{\myinputt_{2}}^{0}$. According to Definition \ref{definition:locallybounded},
		\begin{enumerate}
			\item for any $r\in I_{\myinputt_{1}}^{J}$, there exists a neighborhood $\mathcal{U}_{1}$ of $r$ such that there exists $c_{1}\geq 0$ such that $|\myinputt_{1}(r)| \leq c_{1}$ for almost all $r\in \mathcal{U}_{1}\cap I_{\myinputt_{1}}^{J}$;
			\item for any $r\in I_{\myinputt_{2}}^{0}$, there exists a neighborhood $\mathcal{U}_{2}$ of $r$ such that there exists $c_{2}\geq 0$ such that $|\myinputt_{2}(r)| \leq c_{2}$ for almost all $r\in \mathcal{U}_{1}\cap I_{\myinputt_{2}}^{0}$.
		\end{enumerate} 
		Then, for any $r\in I_{\myinputt}^{J} = I_{\myinputt_{1}}^{J} \cup I_{\myinputt_{2}}^{0} + \{T\}$, consider the following two cases: 1) $r \in I_{\myinputt_{1}}^{J}$ and 2) $r\in I_{\myinputt_{2}}^{0} + \{T\}$. 
		\begin{enumerate}
			\item If $r \in I_{\myinputt_{1}}^{J}$, then there exists a neighborhood $\mathcal{U}_{1}$ of $r$ such that there exists $c_{1}\geq 0$ such that $|\myinputt(r)| \leq c_{1}$ for almost all $r\in \mathcal{U}_{1}\cap I_{\myinputt}^{J}$.
			\item If $r\in I_{\myinputt_{2}}^{0} + \{T\}$, then there exists a neighborhood $\mathcal{U}_{2}$ of $r$ such that there exists $c_{2}\geq 0$ such that $|\myinputt_{2}(r)| \leq c_{2}$ for almost all $r\in \mathcal{U}_{1}\cap I_{\myinputt_{2}}^{0} + \{T\}$.
		\end{enumerate}
		Therefore, $t\mapsto \myinputt(t, J)$ is locally essentially bounded on the interval $I_{\myinputt}^{J}$.
	\end{enumerate}
	Therefore, $t\mapsto \myinputt(t, j)$ is locally bounded on the interval $I_{\myinputt}^{j}$ for all $j\in \mathbb{N}$. \ifbool{conf}{\qed}{}
\end{proof}
\begin{lemma}\label{lemma:concatenation_hybridinput}
	Given two hybrid inputs, denoted $\myinputt_{1}$ and $\myinputt_{2}$, then their concatenation $\myinputt = \myinputt_{1}|\myinputt_{2}$ is also a hybrid input.
\end{lemma}
\begin{proof}
	We show that $\myinputt = \myinputt_{1}|\myinputt_{2}$ satisfies the conditions in Definition \ref{definition:hybridinput} as follows.
	\begin{enumerate}
		\item We show that $\dom \myinputt$ is hybrid time domain;
		\item We show that for each $j\in \nnumbers$, $t\to \myinputt(t, j)$ is Lebesgue measurable on the interval $I^{j}_{\myinputt}:=\{t:(t, j)\in \dom \myinputt\}$.
		\item We show that for each $j\in \nnumbers$, $t\to \myinputt(t, j)$ is locally essentially bounded on the interval $I^{j}_{\myinputt}:=\{t:(t, j)\in \dom \myinputt\}$. 
	\end{enumerate}

	Note that $\myinputt_{1}$ and $\myinputt_{2}$ are hybrid inputs. Therefore, Lemma \ref{lemma:concatenationHTD} guarantees that $\dom \myinputt$ is hybrid time domain. Lemma \ref{lemma:concatenation_input_Lebesguemeasurable} and Lemma \ref{lemma:concatenation_input_locallyessentiallybounded} show that $\myinputt$ is Lebesgue measurable and locally essentially bounded on the interval $I^{j}_{\myinputt}:=\{t:(t, j)\in \dom \myinputt\}$, respectively. Therefore, $\myinputt$ is a hybrid input.\ifbool{conf}{\qed}{}
\end{proof}
\begin{definition} (Absolutely continuous function and locally absolutely continuous function \cite[Definition A.20]{sanfelice2021hybrid})\label{definition:locallycontinuous}
	A function $s: [a, b]\to \myreals[n]$ is said to be \emph{absolutely continuous} if for each $\varepsilon > 0$ there exists $\delta > 0$ such that for each countable collection of disjoint subintervals $[a_{k}, b_{k}]$ of $[a, b]$ such that $\sum_{k}(b_{k} - a_{k}) \leq \delta$, it follows that $\sum_{k}|s(b_{k}) - s(a_{k})| \leq \varepsilon$. A function is said to be locally absolutely continuous if $r\mapsto s(r)$ is absolutely continuous on each compact subinterval of $\mypreals$.
\end{definition}
\begin{lemma}\label{lemma:concatenation_state_locallyabsolutelycontinuous}
	Given two hybrid arcs, denoted $\mystatet_{1}$ and $\mystatet_{2}$, such that $\mystatet_{1}(T, J) = \mystatet_{2}(0,0)$, where $(T, J) = \max \dom \mystatet_{1}$, their concatenation $\mystatet = \mystatet_{1}|\mystatet_{2}$, is such that for each $j\in \nnumbers$, $t\mapsto \mystatet(t,j)$ is locally absolutely continuous on the interval $I^{j}_{\mystatet}:=\{t:(t, j)\in \dom \mystatet\}$.
\end{lemma}
\begin{proof}
	We show that for all $j\in \mathbb{N}$, the function $t\mapsto \mystatet(t, j)$ satisfies the conditions in Definition \ref{definition:locallycontinuous}.
	
	According to Definition \ref{definition:concatenation}, $\dom \mystatet = \dom \mystatet_{1} \cup (\dom \mystatet_{2} +\{(T, J)\})$ where $(T, J) = \max \dom \mystatet_{1}$.
	Therefore, 
	$$
		I^{j}_{\mystatet} = \left\{ \begin{aligned}
		&I^{j}_{\mystatet_{1}} &\text{ if }j< J\\
		&I^{J}_{\mystatet_{1}} \cup (I^{0}_{\mystatet_{2}} +  \{T\}) &\text{ if }j = J\\
		&I^{j - J}_{\mystatet_{2}} + T&\text{ if }j > J
		\end{aligned}\right.
	$$
	
	The following cases are considered to prove that $\mystatet$ is absolutely continuous on each interval $I^{j}_{\mystatet}$.
	\begin{enumerate}[label=\Roman*]
		\item Consider the case when $j < J$. Since $\mystatet_{1}$ is a hybrid arc, according to Definition \ref{definition:hybridarc}, $t\mapsto \mystatet_{1}(t, j)$ is locally absolutely continuous on the interval $I_{\mystatet_{1}}^{j} = I_{\mystatet}^{j}$ such that $j < J$. Note that according to Definition \ref{definition:concatenation}, $\mystatet(t, j) = \mystatet_{1}(t, j)$ for all $t\in I_{\mystatet_{1}}^{j} = I^{j}_{\mystatet}$ where $j < J$. Therefore, $t\mapsto \mystatet(t, j) =\mystatet_{1}(t, j)$ is locally absolutely continuous on the interval $I_{\mystatet}^{j}$ such that $j < J$.
		\item Consider the case when $j> J$. Since $\mystatet_{2}$ is a hybrid arc, according to Definition \ref{definition:hybridarc}, then $t\mapsto \mystatet_{2}(t - T, j - J)$ is locally absolutely continuous on the interval $I_{\mystatet_{2}}^{j - J} + \{T\}= I^{j}_{\mystatet}$ where $j > J$. Note that according to Definition \ref{definition:concatenation}, $\mystatet(t, j) = \mystatet_{2}(t - T, j - J)$ for all $t\in I_{\mystatet_{2}}^{j - J} + \{T\} = I^{j}_{\mystatet}$ where $j > J$. Therefore, $t\mapsto  \mystatet_{2}(t - T, j - J) = \mystatet(t, j)$ is locally absolutely continuous on the interval $I_{\mystatet}^{j} = I_{\mystatet_{2}}^{j - J} + \{T\}$ if $j > J$.
		\item Consider the case when $j = J$ and $I_{\mystatet}^{j} = I_{\mystatet_{1}}^{J}\cup (I_{\mystatet_{2}}^{0} + \{T\})$. Since $t\mapsto \mystatet(t, J)$ is locally absolutely continuous on the intervals $I_{\mystatet_{1}}^{J}$ and $(I_{\mystatet_{2}}^{0} + \{T\})$,  respectively, then
		\begin{enumerate}
			\item for each $\varepsilon_{1} > 0$ there exists $\delta_{1} > 0$ such that for each countable collection of disjoint subintervals $[a^{1}_{k_{1}}, b^{1}_{k_{1}}]$ of $I_{\mystatet_{1}}^{J}$ such that $\sum_{k_{1}}(b^{1}_{k_{1}} - a^{1}_{k_{1}}) \leq \delta_{1}$, it follows that $\sum_{k_{1}}|\mystatet(b^{1}_{k_{1}}) - \mystatet(a^{1}_{k_{1}})| \leq \varepsilon_{1}$;
			\item for each $\varepsilon_{2} > 0$ there exists $\delta_{2} > 0$ such that for each countable collection of disjoint subintervals $[a^{2}_{k_{2}}, b^{2}_{k_{2}}]$ of $I_{\mystatet_{2}}^{0} + \{T\}$ such that $\sum_{k_{2}}(b^{2}_{k_{2}} - a^{2}_{k_{2}}) \leq \delta_{2}$, it follows that $\sum_{k_{2}}|\mystatet(b^{2}_{k_{2}}) - \mystatet(a^{2}_{k_{2}})| \leq \varepsilon_{2}$.
		\end{enumerate}
		Then for each $\varepsilon > 0$, assume that $\delta > 0$ exists such that for each countable collection of disjoint subintervals $[a_{k}, b_{k}]$ of $I_{\mystatet_{1}}^{J}\cup (I_{\mystatet_{2}}^{0} + \{T\})$ such that $\sum_{k}(b_{k} - a_{k}) \leq \delta$, it follows that $\sum_{k}|\mystatet(b^{2}_{k}) - \mystatet(a^{2}_{k})| \leq \varepsilon$. 
		
		Partition the collection of $[a_{k}, b_{k}]$ as follows
		\begin{enumerate}
			\item $\{[a^{1}_{k_{1}}, b^{1}_{k_{1}}]\}_{k_{1}} \leftarrow \{[a_{k}, b_{k}]\}_{k}\cap I_{\mystatet_{1}}^{J}$;
			\item $\{[a^{2}_{k_{2}}, b^{2}_{k_{2}}]\}_{k_{2}} \leftarrow \{[a_{k}, b_{k}]\}_{k}\cap (I_{\mystatet_{2}}^{0} + \{T\})$.
		\end{enumerate}
	Then compute $\varepsilon_{1} = \sum_{k_{1}}|\mystatet(b^{1}_{k_{1}}) - \mystatet(a^{1}_{k_{1}})|$ and $\varepsilon_{2} = \sum_{k_{2}}|\mystatet(b^{2}_{k_{2}}) - \mystatet(a^{2}_{k_{2}})|$. Note that 
	$$\begin{aligned}
	&\sum_{k}|\mystatet(b^{2}_{k}) - \mystatet(a^{2}_{k})| \\
	&=  \sum_{k_{1}}|\mystatet(b^{1}_{k_{1}}) - \mystatet(a^{1}_{k_{1}})| 
	+ \sum_{k_{2}}|\mystatet(b^{2}_{k_{2}}) - \mystatet(a^{2}_{k_{2}})|\\
	& + |\phi_{2}(0, 0) - \phi_{1}(T, J)|, \\
	\end{aligned}$$
	and it is assumed that $\phi_{2}(0, 0) = \phi_{1}(T, J)$. Then
	$$
	\sum_{k}|\mystatet(b^{2}_{k}) - \mystatet(a^{2}_{k})| = \varepsilon_{1} + \varepsilon_{2} + 0.
	$$
	For $\varepsilon_{1}$ and $\varepsilon_{2}$, there exists $\delta_{1}$ and $\delta_{2}$ such that $\sum_{k_{1}}(b^{1}_{k_{1}} - a^{1}_{k_{1}}) \leq \delta_{1}$ and $\sum_{k_{2}}(b^{2}_{k_{2}} - a^{2}_{k_{2}}) \leq \delta_{2}$ hold, respectively. Then we can find $\delta = \delta_{1} + \delta_{2}$ because 
	
	$\sum_{k}(b_{k} - a_{k}) =  \sum_{k_{1}}(b^{1}_{k_{1}} - a^{1}_{k_{1}}) + \sum_{k_{2}}(b^{2}_{k_{2}} - a^{2}_{k_{2}}) \leq \delta_{1} + \delta_{2} = \delta.$
	\end{enumerate}
	Therefore, the function $t\mapsto \mystatet(t, j)$ is locally absolutely continuous on the interval $I_{\mystatet}^{j}$ for all $j\in \mathbb{N}$.\ifbool{conf}{\qed}{}
\end{proof}
\begin{lemma}\label{lemma:concatenation_hybridarc}
	Given two hybrid arcs, denoted $\mystatet_{1}$ and $\mystatet_{2}$, such that $\mystatet_{1}(T, J) = \mystatet_{2}(0,0)$, where $(T, J) = \max \dom \mystatet_{1}$, then their concatenation $\mystatet = \mystatet_{1}|\mystatet_{2}$ is also a hybrid arc.
\end{lemma}
\begin{proof}
	We show that $\mystatet = \mystatet_{1}|\mystatet_{2}$ satisfies the conditions in Definition \ref{definition:hybridarc} as follows.
	\begin{enumerate}
		\item We show that $\dom \mystatet$ is hybrid time domain;
		\item We show that, for each $j\in \nnumbers$, $t\mapsto \mystatet(t,j)$ is locally absolutely continuous on the interval $I^{j}_{\mystatet}:=\{t:(t, j)\in \dom \mystatet\}$. 
	\end{enumerate}
	
	Note that $\mystatet_{1}$ and $\mystatet_{2}$ are functions defined on hybrid time domains. Therefore, Lemma \ref{lemma:concatenationHTD} guarantees that $\dom \mystatet$ is hybrid time domain. Lemma \ref{lemma:concatenation_state_locallyabsolutelycontinuous} guarantees that for each $j\in \nnumbers$, $t\mapsto \mystatet(t,j)$ is locally absolutely continuous on the interval $I^{j}_{\mystatet}:=\{t:(t, j)\in \dom \mystatet\}$. Therefore, $\mystatet$ is a hybrid arc. \ifbool{conf}{\qed}{}
\end{proof}

\begin{lemma}\label{lemma:concatenation_solution_c}
	Given two solution pairs, denoted $\psi_{1}$ and $\psi_{2}$ such that $\psi_{2}(0, 0)\in C$ if both $I_{\psi_{1}}^{J}$ and $I_{\psi_{2}}^{0}$ have nonempty interior, where, for each $j\in \{0, J\}$, $I_{\psi}^{j} = \{t: (t, j)\in \dom \psi\}$ and $(T, J) = \max \dom \psi_{1}$, then their concatenation $\psi = \psi_{1}|\psi_{2}$ satisfies that for each $j\in \mathbb{N}$ such that $I^{j}_{\psi}$ has nonempty interior $\interior(I^{j}_{\psi})$, $\psi = (\mystatet, \myinputt)$ satisfies
	$$
	(\mystatet(t, j),\myinputt(t, j))\in C
	$$ for all $t\in \interior I^{j}_{\mystatet}$.
\end{lemma}
\begin{proof}
	We show that for all $j\in \mathbb{N}$ such that $I^{j}_{\psi}$ has  nonempty interior, $\psi(t, j)\in C$ for all $t\in \interior I^{j}_{\psi}$.
	
	According to Definition \ref{definition:concatenation}, $\dom \psi = \dom \psi_{1} \cup (\dom \psi_{2} +\{(T, J)\})$ where $(T, J) = \max \dom \psi_{1}$.
	Therefore, 
	$$
	I^{j}_{\psi} = \left\{ \begin{aligned}
	&I^{j}_{\psi_{1}} &\text{ if }j< J\\
	&I^{J}_{\psi_{1}} \cup (I^{0}_{\psi_{2}} +  \{T\}) &\text{ if }j = J\\
	&I^{j - J}_{\psi_{2}} + T&\text{ if }j > J
	\end{aligned}\right.
	$$
	
	\begin{enumerate}[label=\Roman*]
		\item Consider the case of all $j< J$ such that $I^{j}_{\psi_{1}}$ has nonempty interior. Since $\psi_{1}$ is a solution pair to $\mathcal{H}$, according to Definition \ref{definition:solution}, 
		$
		\psi_{1}(t, j)\in C
		$ for all 
		$t\in \interior I^{j}_{\psi_{1}}.
		$ Because $\psi(t, j) = \psi_{1}(t, j)$ for all $t \in \interior I^{j}_{\psi_{1}}$, then $\psi(t, j) = \psi_{1}(t, j)\in C$ for all $t\in \interior I^{j}_{\psi} = I^{j}_{\psi_{1}}$.
		\item Consider the case of all $j>J$ such that $I^{j - J}_{\psi_{2}} + \{T\}$ has nonempty interior. Since $\psi_{2}$ is a solution pair  to $\mathcal{H}$, according to Definition \ref{definition:solution}, then 
		$
		\psi_{2}(t, j)\in C
		$ for all 
		$
		t\in \interior I^{j}_{\psi_{2}}.
		$ 
		Because $\psi(t, j) = \psi_{2}(t - T, j - J)$ for all $t\in I_{\psi_{2}}^{j - J} + \{T\}$,  then 
		$
		\psi(t, j) = \psi_{2}(t - T, j - J)\in C
		$
		for all $t\in \interior I_{\psi}^{j} = I^{j - J}_{\psi_{2}} + \{T\}$.
		\item Consider the case when $j = J$ and $I_{\psi}^{J} = I_{\psi_{1}}^{J}\cup (I_{\psi_{2}}^{0} + \{T\})$ has nonempty interior. If $I_{\psi_{1}}^{J}$ or  $I_{\psi_{2}}^{0} + \{T\}$ has nonempty interior while the other has empty interior, then it is the same as above two items to show that $\psi(t, j) \in C$ holds for all $t\in \interior I_{\psi}^{J}$.
		
		If both $I_{\psi_{1}}^{J}$ and $(I_{\psi_{2}}^{0} + \{T\})$ has nonempty interior, since we have had 
		$$
		\psi(t, j) = \psi_{1}(t, j)\in C \quad \forall t\in \interior I_{\psi_{1}}^{J}
		$$ and 
		$$
		\begin{aligned}
		\psi(t, j) = \psi_{2}(t - T, j - J)\in C\\ \forall t\in \interior I_{\psi_{2}}^{0} + \{T\},
		\end{aligned}
		$$ then $\psi(t, j)\in C$ for all $t\in \interior I^{j}_{\psi}\cup (\interior I_{\psi_{2}}^{0} + \{T\})$.
		
		In addition,  since the point $\psi_{2}(0, 0)$ is assumed to belong to $C$ in this case, then $\psi(t, J) \in C$ holds  for all $t\in \interior I_{\psi}^{J}  = \interior (I_{\psi_{1}}^{J} \cup (I_{\psi_{2}}^{0} + \{T\}))$.
		
	\end{enumerate}
	Therefore, $\psi(t, j)\in C$ holds for all $t\in \interior I^{j}_{\psi}$ for all $j\in \mathbb{N}$ such that $I_{\psi}^{j}$ has nonempty interior.\ifbool{conf}{\qed}{}
\end{proof}
\begin{lemma}\label{lemma:concatenation_solution_f}
	Given two solution pairs, denoted $\psi_{1}$ and $\psi_{2}$, then their concatenation $\psi = \psi_{1}|\psi_{2}$ satisfies that for each $j\in \mathbb{N}$ such that $I^{j}_{\psi}$ has nonempty interior $\interior(I^{j}_{\psi})$, $\psi = (\mystatet, \myinputt)$ satisfies
	$$
	\frac{\text{d}}{\text{d} t} {\mystatet}(t,j) = f(\mystatet(t,j), \myinputt(t,j))
	$$ for almost all $t\in I^{j}_{\mystatet}$.
\end{lemma}
\begin{proof}
	We show that for all $j\in \mathbb{N}$ such that $I^{j}_{\psi}$ has  nonempty interior, for almost all $t\in I_{\psi}^{j}$,
	$
	\dot{\mystatet} = f(\mystatet(t, j), \myinputt(t, j)).
	$
	
	According to Definition \ref{definition:concatenation}, $\dom \psi = \dom \psi_{1} \cup (\dom \psi_{2} +\{(T, J)\})$ where $(T, J) = \max \dom \psi_{1}$.
	Therefore, 
	$$
	I^{j}_{\psi} = \left\{ \begin{aligned}
	&I^{j}_{\psi_{1}} &\text{ if }j< J\\
	&I^{J}_{\psi_{1}} \cup (I^{0}_{\psi_{2}} +  \{T\}) &\text{ if }j = J\\
	&I^{j - J}_{\psi_{2}} + T&\text{ if }j > J
	\end{aligned}\right.
	$$
	\begin{enumerate}[label=\Roman*]
		\item Consider the case of all $j < J$ such that $I^{j}_{\psi_{1}}$ has  nonempty interior. Since $\psi_{1}$ is a solution pair, then $\dot{\mystatet}_{1} (t, j)= f(\mystatet_{1}(t, j), \myinputt_{1}(t, j))$ for almost all $t\in \interior I^{j}_{\psi_{1}}$. Because 
		$$
		(\mystatet(t, j), \myinputt(t, j)) = (\mystatet_{1}(t, j), \myinputt_{1}(t, j))
		$$ for all $(t, j)\in \dom \psi_{1}\backslash (T, J)$, therefore, 
		$$
		\begin{aligned}
		\dot{\mystatet}(t, j) &= \dot{\mystatet}_{1}(t, j) \\
		&= f(\mystatet_{1}(t, j), \myinputt_{1}(t, j)) = f(\mystatet(t, j), \myinputt(t, j))
		\end{aligned}
		$$ for almost all $t\in \interior I^{j}_{\psi}$.
		\item Consider the case of all $j > J$ such that $I^{j- J}_{\psi_{2}}$ has  nonempty interior. Since $\psi_{2}$ is a solution pair and $I^{j- J}_{\psi_{2}}$ has nonempty interior, $\dot{\mystatet}_{2}(t, j) = f(\mystatet_{2}(t- T, j - J), \myinputt_{2}(t- T, j- J))$ for almost all $t\in \interior I^{j - J}_{\psi_{2}} +\{T\}$. Because 
		$$
		\begin{aligned}
		&(\mystatet(t, j), \myinputt(t, j)) = \\&(\mystatet_{2}(t - T, j - J), \myinputt_{2}(t - T, j - J))
		\end{aligned}
		$$
		for all $t\in I_{\psi_{2}}^{j - J} + \{T\}$,  then 
		$$
		\begin{aligned}
		\dot{\mystatet}(t, j) &=\dot{\mystatet}_{2}(t - T, j - J) \\
		&= f(\mystatet_{2}(t - T, j - J), \myinputt_{2}(t - T, j - J)) \\
		&= f(\mystatet(t, j), \myinputt(t, j))
		\end{aligned}
		$$
		for almost all $t\in \interior I^{j}_{\mystatet}$.
		\item Consider the case when $j = J$ and $I_{\psi}^{J} = I_{\psi_{1}}^{J}\cup (I_{\psi_{2}}^{0} + \{T\})$ has nonempty interior. If $I_{\psi_{1}}^{J}$ or  $I_{\psi_{2}}^{0} + \{T\}$ has nonempty interior while the other has empty interior, then it is the same as above two items to show that for almost all $t\in I_{\psi}^{j}$,
		$
		\dot{\mystatet} = f(\mystatet(t, j), \myinputt(t, j)).
		$
		If both $I_{\psi_{1}}^{J}$ and $(I_{\psi_{2}}^{0} + \{T\})$ has nonempty interior, then the interval $I_{\psi}^{J} =  I_{\psi_{1}}^{J} \cup (I_{\psi_{2}}^{0} + \{T\})$. Since we have had $\dot{\mystatet}(t, J) = f(\mystatet(t, J), \myinputt(t, J))$ for almost all $t\in I_{\psi_{1}}^{J}$ and $t\in I_{\psi_{2}}^{0} + \{T\}$, then $\dot{\mystatet}(t, j)= f(\mystatet(t, j), \myinputt(t, j))$ holds for $t\in I_{\psi}^{J} =  I_{\psi_{1}}^{J} \cup (I_{\psi_{2}}^{0} + \{T\})$.
	\end{enumerate}
	Therefore, $ \dot{\mystatet}(t, j) = f(\mystatet(t, j), \myinputt(t, j))$ holds for almost all $t\in I_{\psi}^{j}$.\ifbool{conf}{\qed}{}
\end{proof}
\begin{lemma}\label{lemma:concatenation_solution_dg}
	Given two solution pairs, denoted $\psi_{1}$ and $\psi_{2}$, such that $\mystatet_{1}(T, J) = \mystatet_{2}(0,0)$, where $(T, J) = \max \dom \mystatet_{1}$, then their concatenation $\psi = \psi_{1}|\psi_{2} = (\mystatet, \myinputt)$ satisfies that for all $(t,j)\in \dom (\mystatet, \myinputt)$ such that $(t,j + 1)\in \dom (\mystatet, \myinputt)$, 
$$
	\begin{aligned}
	(\mystatet(t, j), \myinputt(t, j))&\in D\\
	\mystatet(t,j+ 1) &= g(\mystatet(t,j), \myinputt(t, j)).
	\end{aligned}
$$
\end{lemma}
\begin{proof}
	We show that for all $(t, j)\in \dom \psi$ such that $(t, j + 1)\in \dom \psi  = \dom \psi_{1} \cup (\dom \psi_{2} +\{(T, J)\})$ where $(T, J) = \max \dom \psi_{1}$,
	$$
	\begin{aligned}
	\psi(t, j) &\in D\\
	\mystatet(t, j + 1) &= g(\mystatet(t, j), \myinputt(t, j)) 
	\end{aligned}
	$$
	\begin{enumerate}
		\item For all $(t, j)$ such that $(t, j) \in \dom \psi_{1}$ and $(t, j + 1)\in \dom \psi_{1}$, since $\psi_{1}$ is a solution pair, then 
		$$
		\begin{aligned}
		\psi_{1}(t, j) &\in D\\
		\mystatet_{1}(t, j + 1) &= g(\mystatet_{1}(t, j), \myinputt_{1}(t, j)).\\
		\end{aligned}
		$$ Due to $(t, j + 1)\in \dom \psi_{1}$, then $(t, j)\neq (T, J)$. Since $\psi(t, j) = \psi_{1}(t, j)$ for all $(t, j)\in \dom \psi_{1}\backslash (T, J)$, then			
		$$
		\begin{aligned}
		\psi(t, j)  &= 	\psi_{1}(t, j)\in D\\
		\mystatet(t, j + 1) &= \mystatet_{1}(t, j + 1)\\
		&= g(\mystatet_{1}(t, j), \myinputt_{1}(t, j))\\
		& = g(\mystatet(t, j), \myinputt(t, j)) .
		\end{aligned}
		$$
		\item For all $(t, j)$ such that $(t, j) \in \dom \psi_{2} + \{(T, J)\}$ and $(t, j + 1)\in \dom \psi_{2} + \{(T, J)\}$, we can have $(t - T, j - J)\in \dom \psi_{2}$ and $(t - T, j - J + 1)\in \dom \psi_{2}$. Since $\psi_{2}$ is a solution pair, then 
		$$
		\begin{aligned}
		&\psi_{2}(t - T, j - J) \in D\\
		&\mystatet_{2}(t - T, j - J + 1) \\
		&= g(\mystatet_{2}(t - T, j - J), \myinputt_{2}(t - T, j - J)).
		\end{aligned}
		$$ Since $\psi(t, j) = \psi_{2}(t - T, j - J)$ for all $(t - T, j - J)\in \dom \psi_{2}$, then
		$$
		\begin{aligned}
		&\psi(t, j)  = 	\psi_{2}(t - T, j - J) \in D\\
		&\mystatet(t, j + 1) = \mystatet_{2}(t - T, j - J + 1)\\
		&= g(\mystatet_{2}(t - T, j - J), \myinputt_{2}(t - T, j - J))\\
		& = g(\mystatet(t, j), \myinputt(t, j)). 
		\end{aligned}
		$$
					\item A special case is that $(t, j) \in \dom \psi_{1}$ and $(t, j + 1) \in \dom \psi_{2} + \{(T, J)\}$. The only case is that $(t, j) = (T, J - 1)\in \dom \psi_{1}$ and $(t, j + 1) = (T, J)\in \dom \psi_{2} + \{(T, J)\}$. Note that it is assume that $\mystatet_{1}(T, J) = \mystatet_{2}(0,0)$. Therefore, we have
					$$
					\begin{aligned}
						\mystatet_{1}(T, J ) = \mystatet_{2}(0,0) = \mystatet(T, J)\\
					\mystatet_{1}(T, J - 1) = \mystatet(T, J - 1)\\
					\myinputt_{1}(T, J - 1) = \myinputt(T, J - 1)
					\end{aligned}
					$$
					Since $\psi_{1}$  is a solution pair of $\mathcal{H}$ and 
					$$
						(T,J - 1)\in \dom \psi_{1} \quad (T, J)\in \dom \psi_{1},
					$$
					according to Definition \ref{definition:solution}, we have
					$$
					\begin{aligned}
					\mystatet_{1}(T, J) = g(\mystatet_{1}(T, J - 1), \myinputt_{1}(T, J - 1))\\
					\psi_{1}(T, J - 1) = (\mystatet_{1}(T, J - 1), \myinputt_{1}(T, J - 1)) \in D.
					\end{aligned}
				$$
					Therefore, we have
					$$
					\begin{aligned}
					 &\mystatet(T, J) = \mystatet_{2}(0,0) = \mystatet_{1}(T, J) \\
					 &= g(\mystatet_{1}(T, J - 1), \myinputt_{1}(T, J - 1)) \\
					 & = g(\mystatet(T, J - 1), \myinputt(T, J - 1)) \\
					 \end{aligned}
					 $$
					 and
					 $$
					 \begin{aligned}
					 &(\mystatet(T, J - 1), \myinputt(T, J - 1)) \\
					 &= (\mystatet_{1}(T, J - 1), \myinputt_{1}(T, J - 1)) \\
					 &\in D.
					\end{aligned}
					$$
	\end{enumerate}\ifbool{conf}{\qed}{}
\end{proof}

\section{Proof of Proposition \ref{theorem:reversedarcbwsystem}}\label{appendix:proof_reverse}
Given the solution pair $\psi = (\phi, u)$ to a hybrid system $\fw{\mathcal{H}}$, we need to show that the reversal $\psi' = (\phi', u')$ of $\psi$  is a compact solution pair to its backward-in-time hybrid system $\mathcal{H}^{\text{bw}}$. The following items are showing that $\psi'$ satisfies each condition in Definition \ref{definition:solution}.
	\begin{itemize}
		\item The first item is to prove that the domain $\dom \phi'$ equals $\dom u'$ and $\dom \psi' := \dom \phi' =  \dom u'$ is a compact hybrid time domain. 
		
		Since $\psi = (\phi, u)$ is a solution pair to $\fw{\mathcal{H}}$, then 
		$$\dom \phi = \dom u.$$ 
		Due to Definition \ref{definition:reversedhybrdarc}, $\dom \phi' = \{(T, J)\} - \dom \phi$ and $\dom u' = \{(T, J)\} - \dom u$ where $(T, J) = \max \dom \psi$. Therefore,
		 $$
		 \begin{aligned}
		 	\dom \phi' &= \{(T, J)\} - \dom \phi\\
		 	&= \{(T, J)\} - \dom u = \dom u'.
		 \end{aligned}$$
		Thus, $\dom \phi' = \dom u'$ is proved.
		
		Then we want to show that $\dom \psi' = \dom \phi' = \dom u'$ is a hybrid time domain. Since $\psi$ is a compact solution pair, then $\dom \psi$ is a compact hybrid time domain. Therefore, 
		\begin{equation}
		\dom \psi  = \cup_{j = 0}^{J}([t_{j}, t_{j+1}],j)
		\end{equation}
		holds for some finite sequence of times, 
		\begin{equation}
		\label{sequence:reverse}
			0=t_{0}\leq t_{1}\leq t_{2}\leq ... \leq t_{J + 1} = T
		\end{equation}
		where $(T, J) = \max \dom \psi$.
		
		Since $\dom \psi' = \{(T, J)\} - \dom \psi$, then
		\begin{equation}
		\label{equation:intersectionreverse}
		\begin{aligned}
		\dom \psi'  &= \{(T, J)\} - \dom \psi\\
		&= \{(T, J)\} - \cup_{j = 0}^{J}([t_{j}, t_{j+1}],j) \\
		&= \cup_{j = 0}^{J}([T - t_{j+1}, T - t_{j}],J -j).\\
%
%
		\end{aligned}
		\end{equation}
		Let $t'_{j} = T - t_{J + 1 - j}$, where $t_{j}$ is the sequence of time in (\ref{sequence:reverse}). Since (\ref{sequence:reverse}) holds, then 
		$$
			0=t'_{0}\leq t'_{1}\leq t'_{2}\leq ... \leq t'_{J + 1} = T.
		$$
		Therefore, (\ref{equation:intersectionreverse}) can be written as
		\begin{equation}
			\begin{aligned}
			\dom \psi'   &= \cup_{j = 0}^{J}([t'_{j}, t'_{j + 1}], j).
			\end{aligned}
		\end{equation}
		Hence, $\dom \psi'$ is a compact hybrid time domain. 
		
		\item The second item is to prove that $\psi'(0, 0) \in \overline{C^{\text{bw}}} \cup D^{\text{bw}}$. Since $\psi$ is compact, the hybrid time $(T, J)\in \dom \psi$, where $(T, J) = \max \dom \psi$. The state $\phi(T, J)$ is either reached during a flow or at a jump. 
		\begin{enumerate}
			\item Consider the case when $\phi(T, J)$ is reached during a flow. In this case, $I_{\psi}^{J} = I_{\psi'}^{0}$ has nonempty interior. 
			
			Definition  \ref{definition:reversedhybrdarc} suggests that $$
			\phi'(0, 0) = \phi(T, J)
			$$ and 
			$u'(0, 0)$ is such that 
			$$(\phi'(0, 0), u'(0, 0))\in \overline{\fw{C}} = \overline{C^{\text{bw}}}.$$ Therefore, 
			$$\psi'(0, 0) = (\phi'(0, 0), u'(0, 0)) \in \overline{C^{\text{bw}}} \subset \overline{C^{\text{bw}}} \cup D^{\text{bw}}.$$
			\item Consider the case when $\psi(T, J)$ is reached at a jump. In this case, 
			$$
			(T, J - 1)\in \dom \psi \quad (T, J)\in \dom \psi.
			$$
			According to Definition \ref{definition:solution}, the state input pair 
			$$
			\begin{aligned}
			\psi(T, J - 1) &= (\phi(T, J -1), u(T, J - 1))\in \fw{D}\quad \\ 
			\phi(T, J) &= g(\phi(T, J - 1), u(T, J - 1)).
			\end{aligned}
			$$ According to Definition \ref{definition:reversedhybrdarc},
			$$
				\phi'(0, 0) = \phi(T, J) \quad u'(0, 0) = u'(T, J - 1). 
			$$
			Therefore, 
			$$\begin{aligned}
			(\phi'(0, 1), u'(0, 0)) &\in \fw{D}\\
			 \phi'(0, 0) &= g(\phi'(0, 1), u'(0,0)). 
			\end{aligned}
			$$

			Since $D^{\text{bw}}$ is defined as 
			$$
			D^{\text{bw}} = \{(x, u): \exists z\in g^{\text{bw}} (x, u): (z, u)\in \fw{D}\},
			$$
			where $g^{\text{bw}}$ is defined as
			$$
			g^{\text{bw}}(x, u) = \{z: x = \fw{g}(z, u), (z, u)\in \fw{D}\},
			$$
			therefore, 
			$$
				\phi'(0, 1) \in g^{\text{bw}}(\phi'(0,0), u'(0,0))
			$$ and
			$$
				\psi'(0,0) = (\phi'(0,0), u'(0,0)) \in D^{\text{bw}}.
			$$
			Thus, $\psi'(0,0)= (\phi'(0, 0), u'(0, 0)) \in D^{\text{bw}}\subset \overline{C^{\text{bw}}} \cup D^{\text{bw}}$.
		\end{enumerate} 
		In conclusion, $\psi'(0, 0) = (\phi'(0, 0), u'(0, 0))\in \overline{C^{\text{bw}}} \cup D^{\text{bw}}$.
		\item This item is to prove that $\psi'$ satisfies the conditions in the first item in Definition \ref{definition:solution}. The following items are to prove each of these conditions is satisfied. 
		\begin{enumerate}[label=(\alph*)]		
			\item This item is to prove that for all $j\in \mathbb{N}$ such that $I^{j}_{\psi'}$ has nonempty interior, $\phi'$ is absolutely continuous on each $I_{\phi'}^{j} = \{t: (t, j)\in \dom \phi'\}$ with nonempty interior. 
			
			Due to $\dom \psi' = \{(T, J)\} - \dom \psi$, then $(t, j)\in \dom \psi'$ implies $(T - t, J - j)\in \dom \psi$. Hence, 
			\begin{equation}
			\label{equation:reverseinterval}
				I_{\psi'}^{j} = \{T\} - I_{\psi}^{J - j}.
			\end{equation}
			Therefore, $\interior I_{\psi'}^{j} \neq \emptyset$ implies that  $\interior I_{\psi}^{J - j} \neq \emptyset$. 
			
			Since $\psi = (\phi, u)$ is a solution pair to $\fw{\mathcal{H}}$, the function $t\mapsto \phi(t, j)$ is locally absolutely continuous for each $I_{\psi}^{j}$ with nonempty interior.  Therefore, $t\mapsto \phi(T - t, J - j)$ is locally absolutely continuous for each $\{T\} - I_{\psi}^{J - j}$ with nonempty interior. 
			
			According to Definition \ref{definition:reversedhybrdarc}, 
			$$
				\phi'(t, j) = \phi(T-t, J - j)
			$$
			for all $(t, j)\in \dom \phi'$.
			Therefore, the function $t\mapsto \phi'(t, j) = \phi(T-t, J - j)$ is locally absolutely continuous for each $I_{\psi'}^{j} = \{T\} - I_{\psi}^{J - j}$ with nonempty interior.
			\item This item is to prove that for all $j\in \mathbb{N}$ such that $I^{j}_{\psi'}$ has nonempty interior, $\psi'(t, j) = (\phi'(t, j), u'(t, j))\in C^{\text{bw}}$ for all $t\in \interior I_{\psi'}^{j}$. 
			
			Since $I^{j}_{\psi'} = \{T\} - I_{\psi}^{J - j}$ and $I^{j}_{\psi'}$ has nonempty interior, then $\{T\} - I_{\psi}^{J - j}$ has nonempty interior.
			Because $\psi = (\phi, u)$ is a solution pair to $\fw{\mathcal{H}}$, according to Definition \ref{definition:solution}, then $\psi(t, j)\in C$ for all $t\in \interior I_{\psi}^{j}$, where $ I_{\psi}^{j}$ has nonempty interior. 
			
			Since $\{T\} - I_{\psi}^{J - j}$ has nonempty interior, then 
			$$
			\begin{aligned}
			&\psi(T- t, J - j)\\
			 &= (\phi(T- t, J - j), u(T- t, J -j))\\
			 &\in \fw{C}
			\end{aligned}
			$$ for all $T - t\in I_{\psi}^{J - j}$.
			
			According to Definition \ref{definition:reversedhybrdarc}, since $I^{j}_{\psi'}$ has nonempty interior, then
			$$
			\begin{aligned}
			\phi'(t, j) &= \phi(T - t, J - j) \\
			 u'(t, j) &= u(T - t, J - j).
			\end{aligned}
			$$		
			Hence, state input pair 
			$$\psi'(t, j) = (\phi'(t, j), u'(t, j))\in \fw{C} = C^{\text{bw}}$$ for all $t\in \interior I_{\psi'}^{j}$, where $I_{\psi'}^{j}$ has nonempty interior.
			\item This item is to prove that for all $j\in \mathbb{N}$ such that $I^{j}_{\psi'}$ has nonempty interior, the function $t\mapsto u'(t, j)$ is Lebesgue measurable and locally bounded. 
			\begin{itemize}
				\item This item is to show that $t\mapsto u'(t, j)$ is Lebesgue measurable for all $j\in \mathbb{N}$ such that $I_{u'}^{j}$ has nonempty interior. 
				
				If $I_{u'}^{j}$ has nonempty interior, then $I_{u}^{J - j} = \{T\} - I_{u'}^{j}$ has nonempty interior. Therefore, $t\mapsto u(t, J - j)$ is Lebesgue measurable. Let $u_{i}$ denote the $i$th component of $u$. For all $a\in \mathbb{R}$ and $i\in \{1, 2, ..., m\}$, since $t\mapsto u(t, J - j)$ is Lebesgue measurable, $S_{i}:= \{t\in I_{u_{i}}^{J - j}: u_{i}(t, J - j)> a\}$ is Lebesgue measurable. 
				
				Note that 
				\begin{equation}
				\begin{aligned}
				S_{i} &= \{t\in I_{u_{i}}^{J - j}: u_{i}(t, J - j)> a\}\\
				& = \{t\in \{T\} - I_{u_{i}}^{J - j}: u_{i}(T - t, J - j)> a\}\\
				&= \{t\in I_{u'}^{j}: u_{i}(T - t, J - j)> a\}\\
				&= \{t\in \interior I_{u'}^{j}: u_{i}(T - t, J - j)> a\}\\
				&\cup \{t\in \partial I_{u'}^{j}: u_{i}(T - t, J - j)> a\}\\
				&= \{t\in \interior I_{u'}^{j}: u'_{i}(t, j)> a\}\\
				&\cup\{t\in \partial I_{u'}^{j}: u_{i}(T - t, J - j)> a\}\\
				&=: S_{i}^{\interior}\cup S_{i}^{\partial}.
				\end{aligned}
				\end{equation}
				Therefore, $S_{i}^{\interior} = S_{i}^{\interior}\cup S_{i}^{\partial}$ is Lebesgue measurable. Let $I_{u'}^{j} = [T_1, T_2]$, then $S_{i}^{\partial}$ can be one of the following:
				$$
				\emptyset, \{T_1\}, \{T_2\}, \{T_1, T_2\}.
				$$ Since all of the above are Lebesgue measurable, then $S_{i}^{\partial}$ is Lebesgue measurable. Since $S_{i}^{\partial}$ is measurable, then its complement $\mathbb{R}\backslash S_{i}^{\partial}$ is Lebesgue measurable. Hence, 
				$$
				S_{i}^{\interior} = S_{i}\cap (\mathbb{R}\backslash S_{i}^{\partial})
				$$
				is Lebesgue measurable. 
				
				Since we want to show that $t\mapsto u'(t, j)$ is Lebesgue measurable for all $j\in \mathbb{N}$ such that $I_{u'}^{j}$ has nonempty interior, it is equivalent to show that $S'_{i} := \{t\in I_{u'}^{j}: u'_{i}(t, j)> a\}$ is Lebesgue measurable for all $a\in \mathbb{R}$, $i\in \{1, 2,..., m\}$ and $j\in \mathbb{N}$ such that $I_{u'}^{j}$ has nonempty interior.
				
				Note that 
				\begin{equation}
				\begin{aligned}
					S'_{i} &= \{t\in \interior I_{u'}^{j}: u'_{i}(t, j)> a\}\\
					&\cup \{t\in \partial I_{u'}^{j}: u'_{i}(t, j)> a\}\\
					& = S_{i}^{\interior}\cup  \{t\in \partial I_{u'}^{j}: u'_{i}(t, j)> a\}\\
					&=: S_{i}^{\interior}\cup  S_{i}^{\partial'}
				\end{aligned}
				\end{equation} and $S_{i}^{\interior}$ is Lebesgue measurable. Hence, the Lebesgue measurability of $S'_{i}$ depends on that of set $S_{i}^{\partial'}$. Similarly, let $I_{u'}^{j} = [T'_1, T'_2]$, then $S_{i}^{\partial'}$ can be one of the following:
				$$
				\emptyset, \{T'_1\}, \{T'_2\}, \{T'_1, T'_2\}.
				$$
				Since all of the above are Lebesgue measurable, then $S_{i}^{\partial'}$ is Lebesgue measurable. Therefore, $S'_{i} = S_{i}^{\interior}\cup  S_{i}^{\partial'}$ is Lebesgue measurable for all $a\in \mathbb{R}$, $i\in \{1, 2,..., m\}$ and $j\in \mathbb{N}$ such that $I_{u'}^{j}$ has nonempty interior. Hence, $t\mapsto u'(t, j)$ is Lebesgue measurable for all $j\in \mathbb{N}$ such that $I_{u'}^{j}$ has nonempty interior. 
				\item This item is to show that $t\mapsto u'(t, j)$ is locally bounded for all $j\in \mathbb{N}$ such that $I_{u'}^{j}:= \{t\in \mathbb{R}_{\geq 0}: (t, j)\in \dom u'\}$. In other words, we want to show that for all $j\in \mathbb{N}$ such that $I^{j}_{u'}$ has nonempty interior, $i\in \{1, 2, ..., m\}$ and any $t_{0}\in I^{j}_{u'}$, there exists a neighborhood $A'$ of $t_{0}$ such that for some number $M' > 0$, one has 
				$$
				|u'_{i}(t, j)|\leq M'
				$$
				for all $t\in A'$.
				
				Note that for all the $j\in \mathbb{N}$ such that $I^{j}_{u}$ has nonempty interior, $t\mapsto u(t, j)$ is locally bounded. Therefore, for all $i\in \{1, 2, ..., m\}$ and any $t_{0}\in I^{j}_{u}$, there exists a neighborhood $A$ of $t_{0}$ such that for some number $M > 0$ one has 
				$$
					|u_{i}(t, j)|\leq M
				$$
				for all $t\in A$, where $u_{i}$ denotes the $i$th component of $u$.

				Note that $\dom u' = (T, J) - \dom u$. Hence, $I_{u'}^{j} = \{T\} - I_{u}^{J - j}$. If $I_{u'}^{j}$ has nonempty interior, then $\{T\} - I_{u}^{J - j}$ has nonempty interior. Since $\{T\} - I_{u}^{J - j}$ has nonempty interior, $t\mapsto u(T - t, J - j)$ is locally bounded. Therefore, for all $i\in \{1, 2, ..., m\}$ and any $t_{0}\in \{T\} - I_{u}^{J - j} = I_{u'}^{j}$, there exists a neighborhood $A$ of $t_{0}$ such that for some number $M > 0$ one has 
				$$
				|u_{i}(T - t, J- j)|\leq M
				$$
				for all $t\in A$.
				
				Note that $A = (A\cap \interior I_{u'}^{j})\cup (A\cap \partial I_{u'}^{j})$. For all the $t\in (A\cap \interior I_{u'}^{j})$, 
				$$
					u(T- t, J - j) = u'(t, j).
				$$ Therefore, $|u_{i}'(t, j)|\leq M$ for all $t\in A\cap \interior I_{u'}^{j}$. 
				
				Let  $I_{u'}^{j} = [T_{1}, T_{2}]$. Since $I_{u'}^{j}$ has nonempty interior, then $T_{1} < T_{2}$. Hence, $\partial I_{u'}^{j} = \{T_{1}, T_{2}\}$. Therefore, $A\cap \partial I_{u'}^{j} \subset \{T_{1}, T_{2}\}$. 
				
				Therefore, for all $t\in A$, 
				$$
				u'_{i}(t, j)\leq \max \{M, u'(T_{1}, j), u'(T_{2}, j)\}.
				$$ We can select $A' = A$ and $M' = \max \{M, u'(T_{1}, j), u'(T_{2}, j)\}$. Hence, then the local boundness of $t\mapsto u'(t, j)$ is proved. 
				
%
			\end{itemize}
		\color{black}
%
%
%
			\item This item is to prove that for all $j\in \mathbb{N}$ such that $I^{j}_{\psi'}$ has nonempty interior, for almost all $t\in I_{\psi'}^{j}$,
			\begin{equation}
			\dot{\phi}'(t, j) = f^{\text{bw}}(\phi'(t, j), u'(t, j)).
			\end{equation}
			For almost all $t\in I^{j}_{\psi'}$, the following holds.
			\begin{equation} 
			\begin{aligned}
			\dot{\phi'}(t, j) &= \dot{\phi}(T - t, J - j) \\
			&= \frac{\text{d} \phi(T - t, J - j)}{\text{d}(T-t)} \\
			&= -\fw{f}(\phi(T - t, J - j),\\
			& u(T - t, J - j))\\
			& = -\fw{f}(\phi'(t, j), u'(t, j))\\
			& = f^{\text{bw}}(\phi'(t, j), u'(t, j))
			\end{aligned}
			\end{equation}
			Therefore, 
			\begin{equation}
			\begin{aligned}
			&	\dot{\phi'}(t, j)  = f^{\text{bw}}(\phi'(t, j), u'(t, j))\\
			&	\text{for almost all $t\in I^{j}_{\psi'}$}
			\end{aligned}
			\end{equation}
		\end{enumerate}
		\item The last item is to  prove that for all $(t, j)\in \dom \psi'$ such that $(t, j + 1)\in \dom \psi'$,
		\begin{equation}
		\begin{aligned}
		(\phi'(t, j), u'(t, j))&\in D^{\text{bw}}\\
		\phi'(t, j + 1) &= g^{\text{bw}}(\phi'(t, j), u'(t, j))
		\end{aligned}
		\end{equation}
		Since the hybrid time domain $\dom \psi' = \{T, J\} - \dom \psi$, for all $(t, j)\in \dom \psi'$ such that $(t, j + 1)\in \dom  \psi'$, the hybrid times $(T - t, J - j - 1)\in \dom \psi$ and $(T - t, J -  j)\in \dom  \psi$. Since $(T - t, J - j - 1)\in \dom  \psi$ and $(T - t, J -  j)\in \dom  \psi$, then
		\begin{equation}
		\begin{aligned}
		&\psi(T - t, J - j - 1) \in \fw{D}\\
		&\phi(T - t, J - j) \\
		&= \fw{g}(\phi(T - t, J - j - 1 ), u(T - t, J - j - 1)).\\
		\end{aligned}
		\end{equation}
		According to Definition \ref{definition:reversedhybrdarc},
		\begin{equation}
		\begin{aligned}
		\phi'(t, j) &= \phi(T - t, J - j)\\
		\phi'(t, j + 1) &= \phi(T - t, J - j - 1 ) \\
		u'(t, j) &= u(T - t, J - j - 1).
		\end{aligned}
		\end{equation}
		Since the state $\phi(T - t, J - j) = \fw{g}(\phi(T - t, J - j - 1 ), u(T - t, J - j - 1))$, then
		\begin{equation}
		\begin{aligned}
		\phi'(t, j + 1) &= \phi(T - t, J - j - 1 ) \\
		&\in g^{\text{bw}}(\phi(T - t, J - j), u(T - t,\\
		& J - j - 1) ) \\
		&\in g^{\text{bw}}(\phi'(t, j), u'(t, j) )\\
		\end{aligned}
		\end{equation}
		Since the state input pair $\psi(T - t, J - j - 1) = (\phi(T - t, J - j - 1 ), u(T - t, J - j - 1 )) \in \fw{D}$, then 
		\begin{equation}
		(\phi'(t, j+ 1 ), u'(t,  j)) \in \fw{D}
		\end{equation}
		Since $\phi'(t, j + 1) \in g^{\text{bw}}(\phi'(t, j), u'(t, j) )$ and the jump set of the backward-in-time hybrid system is defined as (\ref{set:dbw}), the state input pair $(\phi'(t,j),u'(t,j))\in D^{\text{bw}}$.
	\end{itemize}
	Because all the conditions in Definition \ref{definition:solution} are satisfied, $\psi'=(\phi', u')$ is a solution pair to $\mathcal{H^{\text{bw}}}$. 
	
	Note that the proof in the inverse direction can be established by substituting $\psi'=(\phi', u')$ and $\bw{\hs} = (\bw{C}, \bw{f}, \bw{D}, \bw{g})$  with $\psi=(\phi, u)$ and $\fw{\hs} = (\fw{C}, \fw{f}, \fw{D}, \fw{g})$ in each item above, and vice versa.

\section{Lemma 2 in \cite{kleinbort2018probabilistic}}
\begin{lemma}[Lemma 2 in \cite{kleinbort2018probabilistic}]\label{lemma:2in6}
	Let $\pi$, $\pi'$ be two trajectories, with the corresponding control functions $\Gamma(t)$, $\Gamma'(t)$, for some constant $\delta > 0$,. Let $T > 0$ be a time duration such that for all $t\in [0, T]$, it holds that $\Gamma(t) = u$ and $\Gamma'(t) = u'$. That is, $\Gamma$, $\Gamma'$ remain fixed throughout $[0, T]$. Then
	$$
		|\pi(T), \pi'(T)| \leq e^{K_{x}T}\delta + K_{u}Te^{K_{x}T}\Delta u
	$$
	where $\Delta u = |u - u'|$.
\end{lemma}
\theendnotes
\end{document}